\documentclass[a4paper,11pt]{article} 

\usepackage{fullpage}         
\usepackage{lipsum}         
\usepackage[english]{babel}
\usepackage[utf8]{inputenc}
\usepackage[style=alphabetic, backend=bibtex8, natbib]{biblatex}
\usepackage{authblk}          
\usepackage{amssymb}          
\usepackage{amsthm}	          
\usepackage{amsmath}
\usepackage{upgreek}          
\usepackage{mathrsfs}
\usepackage{subcaption}       
\usepackage{graphicx,caption}         
\usepackage{float}            
\usepackage{enumitem}
\usepackage{xcolor}           
\usepackage{hyperref}
\usepackage{algpseudocode}
\usepackage{algorithm}
\usepackage[title]{appendix}

\hypersetup{
    linktocpage = true,
    colorlinks  = true,                        
    linkcolor   = blue,
    citecolor   = blue,
    urlcolor    = blue
}

\numberwithin{equation}{section}

\DeclareMathOperator*{\argmin}{arg\!\min}

\DeclareMathOperator*{\vspan}{span}
\DeclareMathOperator*{\tr}{tr}
\DeclareMathOperator*{\diag}{diag}
\DeclareMathOperator*{\ran}{ran}

\DeclareMathOperator*{\kl}{KL}

\theoremstyle{definition}
\newtheorem{definition}{Definition}[section]

\theoremstyle{remark}

\newtheorem{remark}[definition]{Remark}

\theoremstyle{plain}
\newtheorem{lemma}[definition]{Lemma}

\newtheorem{corollary}[definition]{Corollary}
\newtheorem{theorem}[definition]{Theorem}
\newtheorem{proposition}[definition]{Proposition}



\begin{document}
\title{Contraction rates for conjugate gradient and Lanczos approximate
       posteriors in Gaussian process regression
      }
\renewcommand\Affilfont{\itshape\small}
\author[1]{Bernhard Stankewitz}
\affil[1]{Department of Decision Sciences, Bocconi University,
          Bocconi Institute for Data Science and Analytics (BIDSA), 
          \normalfont \texttt{bernhard.stankewitz@unibocconi.it}
         }
\author[2]{Botond Szab{\'o}}
\affil[2]{Department of Decision Sciences, Bocconi University,
          Bocconi Institute for Data Science and Analytics (BIDSA), 
          \normalfont \texttt{botond.szabo@unibocconi.it}
         }
\date{}
\maketitle
\abstract{ 
Due to their flexibility and theoretical tractability Gaussian process (GP)
regression models have become a central topic in modern statistics and machine
learning.
While the true posterior in these models is given explicitly, numerical
evaluations depend on the inversion of the augmented kernel matrix 
\( K + \sigma^2 I \), which requires up to \( O(n^3) \) operations.
For large sample sizes n, which are typically given in modern applications,
this is computationally infeasible and necessitates the use of an approximate
version of the posterior.
Although such methods are widely used in practice, they typically have very
limtied theoretical underpinning.

In this context, we analyze a class of recently proposed approximation
algorithms from the field of Probabilistic numerics.
They can be interpreted in terms of Lanczos approximate eigenvectors of the
kernel matrix or a conjugate gradient approximation of the posterior mean,
which are particularly advantageous in truly large scale applications, as they
are fundamentally only based on matrix vector multiplications amenable to the
GPU acceleration of modern software frameworks.
We combine result from the numerical analysis literature with state of the art
concentration results for spectra of kernel matrices to obtain minimax
contraction rates.
Our theoretical findings are illustrated by numerical experiments.
}

\section{Introduction}
\label{sec_Introduction}

Due to their flexibility in capturing complex patterns in data without assuming
a specific functional form and their capacity to directly model the uncertainty
of estimates, Gaussian process (GP) have become a mainstay of modern Statistics
and Machine Learning, see e.g. \cite{RasmussenWilliams2006GPForML}.  
Formally, in the classical GP regression model one observes the data 
$ (X_1,Y_1),...(X_n,Y_n) $ satisfying that
\begin{align}
  \label{eq_I_GPRegressionModel}
  Y_{i} = f( X_{i} ) + \varepsilon_{i}, \qquad i = 1, \dots, n,
\end{align}
where \( X: = ( X_{i} )_{ i \le n } \) are i.i.d. random design points from a
subset \( \mathfrak{X} \subset \mathbb{R}^{d} \) forming a Polish space with
distribution \( G \), the noise vector \( \varepsilon: = ( \varepsilon_{i} )_{ i
\le n } \) is an \( n \)-dimensional standard Gaussian independent from the
design, with
variance \( \sigma^{2} \) and the prior knowledge about the unknown function 
\( f: \mathfrak{X} \to \mathbb{R} \) is modeled by a centered Gaussian Process
\( F( x )_{ x \in \mathfrak{ X } } \) defined by a covariance kernel 
\( k: \mathfrak{X}^{2} \to \mathbb{R} \). 
We assume that \( k \) is sufficiently regular such that \( F \) is also Gaussian 
random variable in \( F \in L^{2}(G) \), see, e.g., Section 11.1 in
\cite{GhosalvdVaart2017FundamentalsOfBayes}. 
The corresponding posterior \( \Pi_{n}( \cdot | X, Y ) \) is then also a GP and
can be computed explicitly with mean and covariance functions given by
\begin{align}
  \label{eq_I_TruePosterior}
  x & \mapsto k( X, x )^{ \top } ( K + \sigma^{2} I_{n} )^{-1} Y,
  \\
  ( x, x' ) & \mapsto k(x, x') 
                    - k(X, x )^{ \top } ( K + \sigma^{2} I_{n})^{-1}
                      k(X, x'), 
  \notag
\end{align}
respectively, where 
\( K: = k( X_{i}, X_{j} )_{i, j \le n} \in \mathbb{R}^{n \times n} \) 
is the empirical kernel matrix and
\( k( X, x ): = ( k( X_{i}, x ) )_{i \le n} \in \mathbb{R}^{n} \) 
is the vector valued evaluation of the function \( k( \cdot, x ) \) at the
design points, see e.g. \cite{RasmussenWilliams2006GPForML}.  
Note that the analytic form of the posterior given in \eqref{eq_I_TruePosterior}
involves the inversion of the matrix $K+\sigma^2 I_n\in\mathbb{R}^{n\times n}$,
which has computational complexity $O(n^3)$ and memory requirement $O(n^2)$.
Therefore, in large scale applications, exact Bayesian inference based on GPs
become computationally intractable.

From the side of computational theory, a large body of research has been
dedicated to overcoming this computational bottleneck by developing
approximations of the posterior in Equation \eqref{eq_I_TruePosterior} that
reduce its computational complexity.
Popular methods based on probabilistic considerations include 
variational Bayes posteriors 
\cite{Titsias2009InducingVariablesVB,
      Titsias2009InducingVariablesVBTechnicalReport}, 
Vecchia approximations 
\cite{DattaEtal2016Hierarchical,FinleyEtal2019Efficient,
      KatzfussEtal2020VecchiaOfGPPredictions},
distributed GPs
\cite{Tresp2000BayesianCommittee,RasmussenGhahramani2002GPExperts,
      ParkHuang2016LargeSpatialGPs,KimEtal2005PiecewiseGPs,
      GuhaniyogiEtal2022Distributed,DeisenrothNG2015DistributedGPs}
and banding of the kernel or precision matrix
\cite{DurrandeEtal2019BandedMatrixOperators,YuEtal2021GPDiscriminantAnalysis}.

Recently, there has been a particular emphasis on iterative methods from
classical numerical analysis such as conjugate gradient (CG) and Lanczos
aprooximations of the posterior in Equation \eqref{eq_I_TruePosterior}, see
\cite{PleissEtal2018ConstantTimePredictiveDistributions,WangEtal2019ExactGPs}.
Compared to some of the other methods, prima facie, these algorithms do not have
probabilistic interpretations.
They are, however, particulary advantageous in truly large scale applications.
Specifically, they are solely based on matrix vector multiplications.
Beyond reducing the computational complexity of the algorithm,
this additionally makes them amenable to the GPU acceleration in modern software
frameworks, see, e.g., \texttt{GPyTorch} \cite{GardnerEtal2018GPyTorch}.

Although some form of approximation is indispensable from a computational
perspective, from the perspective of statistical inference, most of the
procedures above  have only limited theoretical underpinning.
Only recently, approximation techniques have been started to be investigated
from a frequentist perspective and contraction rate guarantees were derived for
the approximate posteriors, see for instance
\cite{SzabovZanten2019Distributed,GuhaniyogiEtal2019Divide,
      SzaboEtal2023SpatialAdaptation} 
for distributed GPs and \cite{NiemanEtal2022ContractionRates} for variational
GP methods.
In particular, for the aforementioned procedures from classical numerics, no
frequentist results are known so far, which largely has been due to a lack of
probabilistic interpretations of the resulting approximate posteriors in
these methods and the difficulty of rigorously understanding the spectral
properties of the posterior covariance kernel, which in turn depends on the
random design points.

In this work, we derive frequentist contraction rates for a class of algorithms
from probabilistic numerics called computation aware GPs proposed in
\cite{WengerEtal2022ComputationalUncertainty}, which are based on Bayesian
updates conditional on the data vector \( Y \) projected into different
directions.
Specifically, we cover the instances of their algorithm in which the projection
directions are given by the empirical eigenvectors of the kernel matrix, their
approximations via the Lanczos algorithm and the directions provided by solving
\( K_{ \sigma } w = Y \) via CG.
The last two versions of the algorithm represent probabilistic interpretations
for CG and Lanczos approximations of the posterior.
We then combine bounds from classical numerics with only recently developed
precise spectral concentration results for kernel matrices from
\cite{JirakWahl2023RelativePerturbationBounds} to derive contraction rates.
The results provide a guide for the required number of iterations in the
algorithms to achieve optimal inference for the functional parameter of
interest.
The theoretical thresholds were confirmed in our numerical analysis on synthetic
data sets for different choices of GPs.
Our theory directly applies to the implementation of the CG approximate
posterior implemented in \texttt{GPyTorch} \cite{GardnerEtal2018GPyTorch}, where
the covariance estimation via the Lanczos variance estimation from
\cite{PleissEtal2018ConstantTimePredictiveDistributions} is replaced by a
computationally lighter method based on the conjugate gradient directions
directly, proposed in \cite{WengerEtal2022ComputationalUncertainty}.

Beyond deriving frequentist contraction rate guarantees, we also establish
previously unexplored connections between various numerical algorithms.  
In particular, we provide new probabilistic context for the conjugate gradient
approximation of the posterior and formally derive its connection to the Lanczos
algorithm and a  particular variational Bayes approach proposed in
\cite{Titsias2009InducingVariablesVB}.
These connections can be used to further speed up the computation of the
approximate posteriors and put the algorithms in a wider context.
\\

The remainder of the paper is organized as follows.
In Section, \ref{sec_ApproximatePosteriorsFromNrobabilisticNumerics} we recall
the recently proposed, general Bayesian updating algorithm from probabilistic
numerics and consider specific examples including Lanczos iteration and
conjugate gradient descent.
In Section \ref{sec_ApproximatePosteriorContraction}, we discuss the frequentist
analysis of (approximate) posterior distributions, which provides the context
for our main contraction rate result for numerical algorithms in Section
\ref{sec_MainResults}.
We apply this general theorem for several standard examples in Section
\ref{sec_Examples} and compare the numerical performance of the considered
algorithms over synthetic data sets in Section \ref{sec_NumericalSimulations}.
The fundamental ideas of the proofs are highlighted in Section
\ref{sec_TechnicalAnalysis} while the rigorous proof of our main contraction
result is given in Appendix \ref{sec_ProofsOfMainResults}.
Auxiliary lemmas, the proofs covering the examples and additional technical
lemmas are deferred to Appendix \ref{sec_AuxiliaryResults}, Appendix
\ref{sec_ProofsForTheExamples} and Appendix \ref{sec_ComplementaryResults},
respectively.\\

\noindent \textbf{Notation.}
We collect some notational choices, we use in the following.
For two positive sequences \( a_n, b_n \) let us denote by 
\( a_n \lesssim b_n \) if there exists a constant \( C > 0 \) independent of \(
n \) such that \( a_n / b_n \le C \) for all \( n \). 
Furthermore, we denote by \( a_n \asymp b_n \), if \( a_n \lesssim b_n \) and
\( b_n \lesssim a_n \) hold simultaneously.
When needed explicitly, we will also refer to \( c, C > 0 \) for constants
independent of \( n \), which are allowed to change from line to line.

The evaluation \( ( f( X_{1} ), \dots, f( X_{n} ) ) \in \mathbb{R}^{ n } \) of a
function \( f \in L^{2}( G ) \) at the design points will be often be denoted by
\( \mathbf{ f } \), \( \land \) and \( \lor \) denote minimum and maximum
between two numbers respectively and ``\( (\cdot) \)'' will denote an open
function argument.

Finally, for a probability measure \( \mathbb{P} \), in some proofs, the
distribution of a random variable \( X \) under \( \mathbb{P} \) will be denoted 
\( \mathbb{P}^{X} \).
Similarly, when \( Y \) is another random variable, the conditional
distributions of \( Y | X \) and \( Y | X = x \) will be denoted as 
\( \mathbb{P}^{ Y | X } \) and \( \mathbb{P}^{ Y | X = x } \) respectively.


\section{Approximate posteriors from probabilistic numerics}
\label{sec_ApproximatePosteriorsFromNrobabilisticNumerics}

The true posterior mean function
\begin{align}
  x \mapsto k( X, x )^{ \top } K_{ \sigma }^{-1} Y,
  \qquad x \in \mathfrak{X},
\end{align}
in \eqref{eq_I_TruePosterior} is a linear combination of the functions
\(  k( X_{i}, \cdot ), i \le n \).
The weights \( K_{ \sigma }^{-1} Y \) in the linear combination above,
however, are computationally inaccessible for large sample sizes.

In this work, we focus on a class of approximation algorithms for the posterior
from probabilistic numerics, called computation aware GPs, proposed in
\cite{WengerEtal2022ComputationalUncertainty}.
They can be interpreted as iteratively learning the \emph{representer weights}
\( K_{ \sigma }^{-1} Y \) of the posterior mean function by iteratively solving
the equation
\begin{align}
  Y = K_{ \sigma } w, \qquad w \in \mathbb{R}^{n},
\end{align}
in a computationally efficient manner.

\begin{remark}[Representer weights]
  \label{rem_RepresenterWeights}
  Recall, that the posterior mean function is the solution of the kernel
  regression problem
  \begin{align}
    \label{eq_APPN_KernelRidgeRegression}
    \min_{ f \in \mathbb{H} } 
    \| Y - f \|_{n}^{2} + \sigma^{2} \| f \|_{ \mathbb{H} }^{2}, 
  \end{align}
  where \( \mathbb{H} \) is the reproducing kernel Hilbert (RKHS) space induced
  by \( k \), see e.g. Section 6.2 in
  \cite{WilliamsRasmussen1995GPsForRegression}.
  Hence, the name representer weights originates from the Representer Theorem,
  see \cite{KimeldorfWahba1971SplineFunctions} and
  \cite{WilliamsRasmussen1995GPsForRegression},
  stating that in minimization problems, such as in Equation
  \eqref{eq_APPN_KernelRidgeRegression}, the minimizer over the whole space 
  \( \mathbb{H} \) can be represented as a weighted linear combination of the
  features \( k( X_{i}, \cdot ) \).
\end{remark}

An important aspect of the general algorithm proposed in
\cite{WengerEtal2022ComputationalUncertainty} is its probabilistic
interpretation in terms of a Bayesian updating scheme.
We shortly recall this interpretation, which will be a corner stone of our
theoretical analysis.

\subsection{A general class of algorithms}
\label{ssec_AGeneralClassOfAlgorithms}

In \cite{WengerEtal2022ComputationalUncertainty}, the authors regard the
distribution \( N( 0, K_{ \sigma }^{-1} ) \) of the true representer weights 
\( W^{*}: = K_{ \sigma }^{-1} Y \) as a model of our prior knowledge about 
\( W^{*} \).
The mean \( w_{0}: = 0 \) is an initial best guess about \( W^{*} \) and 
\( \Gamma_{0}: = K_{ \sigma }^{-1} \) represents our excess
uncertainty owed to the existing computational constraints.
Then, our guess and the corresponding uncertainty are iteratively updated via
the following algorithm. 

For given believes \( W^{*} \sim N( w_{ m - 1 }, \Gamma_{ m - 1 }) \),
\( m \in \mathbb{N} \), the Bayesian update at the \( m \)-th step is computed
based on the following computationally accessible information.
Let us compute first the \textit{observation via information operator}
\( \alpha_{m} \), which is the inner product of the \( m \)-th policy or search
direction \( s_{m} \in \mathbb{R}^{n} \) (provided by the user) and the
predictive residual \( Y - K_{ \sigma } w_{ m - 1 } \in \mathbb{R}^n \), i.e.,
\begin{align}
      \alpha_{m}:
  =
      s_{m}^{ \top } ( Y - K_{ \sigma } w_{ m - 1 } )
  = 
      s_{m}^{ \top } K_{ \sigma } ( K_{ \sigma }^{-1} Y - w_{ m - 1 } ).
\end{align}
Then, Lemma \ref{lem_MultivariateConditionalGaussians} below yields inductively
that \( W^{*} | \alpha_{m} \sim N( w_{m}, \Gamma_{m} ) \) with 
\begin{align}
  \label{eq_APPN_BayesianUpdateMean}
      w_{m} 
  & = 
      w_{ m - 1 } 
    + 
      \underbrace{ \Gamma_{ m - 1 } K_{ \sigma } s_{m} }_{ =: d_{m} }
      (
        \underbrace{ s_{m}^{ \top } K_{ \sigma } \Gamma_{ m - 1 } K_{ \sigma } 
                     s_{m}
                   }
                  _{ =: \eta_{m} }
      )^{-1}  
      \alpha_{m}
  = 
      w_{ m - 1 } + \eta_{m}^{-1} d_{m} d_{m}^{ \top } Y
  = 
      C_{m} Y,
  \\
  \label{eq_APPN_BayesianUpdateCov}
      \Gamma_{m} 
  & =
      \Gamma_{ m - 1 } - 
      \underbrace{ \Gamma_{ m - 1 } K_{ \sigma } s_{m} }_{ =: d_{m} }
      (
        \underbrace{ s_{m}^{ \top } K_{ \sigma } \Gamma_{ m - 1 } K_{ \sigma }
                     s_{m}
                   }
                  _{ =: \eta_{m} }
      )^{-1}  
      ( \underbrace{ \Gamma_{ m - 1 } K_{ \sigma } s_{m} }_{ =: d_{m} }
      )^{ \top }
  = 
      K_{ \sigma }^{-1} - C_{m},
\end{align}
and \( C_{m}: = \sum_{ j = 1 }^{m} \eta_{j}^{-1} d_{j} d_{j}^{ \top } \).
Here \( C_m \) denotes the \textit{precision matrix approximation}, \( d_m \)
the \textit{search direction}, and \( \eta_m \) is the \textit{normalization
constant} after \( m \) iterations.

Note that \( C_{m} K_{ \sigma } \) is the \( K_{ \sigma } \)-orthogonal
projection onto the \( \vspan \{ s_{j}: j \le m \} \), see Lemma S1 in
\cite{WengerEtal2022ComputationalUncertainty}.
This guarantees that the updating formulas in 
Equations \eqref{eq_APPN_BayesianUpdateMean} 
      and \eqref{eq_APPN_BayesianUpdateCov}
are well defined as long as the policies \( ( s_{j} )_{ j \le m } \) are
linearly independent.
Indeed, then, it follows inductively that
\begin{align}
      \eta_{j} 
  =
      ( ( I - C_{ j - 1 } K_{ \sigma } ) s_{j} )^{ \top }
      K_{ \sigma }
      ( I - C_{ j - 1 } K_{ \sigma } ) s_{j}
  >   0
\end{align}
for all \( j \le m \).
Further, for \( m \to n \), \( C_{m} \) monotonously approaches 
\( K_{ \sigma }^{-1} \).
This implies that updating iteratively improves the estimate \( w_{m} \) of the
representer weights and reduces the excess uncertainty \( \Gamma_{m} \).

Based on the updated believes \( N( w_{m}, \Gamma_{m} ) \) about the representer
weights, a full approximate posterior can be constructed in the following
manner.
Conditional on the true representer weights the prior process \( F \) is
distributed according to
\begin{align}
        F | W^{*} = w^{*}
  \sim
        \text{GP}\big( 
          k( X, \cdot )^{ \top } w^{*}, 
          k( \cdot, \cdot ) 
        -
          k( X, \cdot )^{ \top } K_{ \sigma }^{-1} k( X, \cdot ) 
        \big). 
\end{align}
At a fixed iteration \( m \), integrating out \( F | W^{*} = w^{*} \) against
our current belief \( N( W_{m}, \Gamma_{m} ) \) yields the process 
\( 
  \Psi_{m}( \cdot ): = \mathbb{P}^{ F | W^{*} = w^{*} }( \cdot ) 
                       N( W_{m}, \Gamma_{m} )( d w^{*}) 
\). 
By another application of Lemma \ref{lem_MultivariateConditionalGaussians}, this
is the GP with mean and covariance functions
\begin{align}
  \label{eq_APPN_MathematicalAndComputationalUncertainty}
              x 
  & \mapsto 
              k( X, x )^{ \top } w_{m}
    = 
              k( X, x )^{ \top } C_{m} Y,
  \\
              ( x, x' ) 
  & \mapsto 
              \underbrace{ k( x, x' ) - k( X, x )^{ \top } C_{m} k( X, x' ) }
                        _{ \text{Combinded uncertainty} }
  \notag
  \\
  & = 
              \underbrace{ k( x, x' ) 
                         - k( X, x )^{ \top } K_{ \sigma }^{-1} k( X, x' )
                         }
                        _{\text{Mathematical uncertainty}}
              + 
              \underbrace{ k( X, x )^{ \top } \Gamma_{m} k( X, x' )}
                        _{\text{Computational uncertainty}},
  \notag
\end{align}
respectively.
The process \( \Psi_{m} \) is then an approximation of the true posterior 
\( \Pi( \cdot | X, Y )  \) from Equation \eqref{eq_I_TruePosterior}, in which we
have replaced \( K_{ \sigma }^{-1} \) with \( C_{m} \).
Since \( C_{m} = K_{ \sigma }^{-1} - \Gamma_{m} \), this approximation can be
understood as combining the \emph{mathematical/statistical} uncertainty
of the true posterior with the \emph{computational} uncertainty introduced by
the approximation of \( K_{ \sigma }^{-1} \).
Note that the Bayesian updating scheme guarantees that both \( C_{m} \) and 
\( \Gamma_{m} \) are well defined positive semi-definite covariance matrices.
Furthermore, while individually the terms \( K_{ \sigma }^{-1} \) and 
\( \Gamma_{m} \) are computationally prohibitive, the combined uncertainty
represented by \( C_{m} \) can be evaluated.
We obtain the following iterative algorithm:
\begin{center}
\begin{minipage}{.8\linewidth}
  \begin{algorithm}[H]
  \caption{GP approximation scheme}
  \label{alg_GPApproximation}
  \begin{algorithmic}[1]
    \Procedure{ITERGP\(( k, X, Y ) \) }{}
    \vspace{2px}
    \State  \( C_{0} \gets 0 \in \mathbb{R}^{ n \times n } \),
            \( w_{0} \gets 0 \in \mathbb{R}^{n} \)
    \vspace{2px}
    \For{ \( j = 1, 2, \dots, m \) }
      \State \( s_{j} \gets \text{POLICY()} \) 
      \State \( d_{j} \gets ( I - C_{ j - 1 } K_{ \sigma } ) s_{j} \) 
      \vspace{2px}
      \State \( \eta_{j} \gets s_{j}^{ \top } K_{ \sigma } d_{j} \) 
      \State \( C_{j} \gets C_{ j - 1 } + \eta_{j}^{-1} d_{j} d_{j}^{ \top }   \) 
      \State \( w_{j} \gets w_{ j - 1 } + \eta_{j}^{-1} d_{j} d_{j}^{ \top } Y \) 
    \EndFor
    \State \( \mu_{m}( \cdot ) \gets k( X, \cdot )^{ \top } w_{m} \) 
    \State \( 
             k_{m}( \cdot, \cdot ) \gets k( \cdot, \cdot ) 
           - 
             k( X, \cdot )^{ \top } C_{m} k( X, \cdot ) 
           \) 
    \vspace{2px}
    \EndProcedure
    \State \Return \( \text{GP}( \mu_{m}, k_{m} ) \) 
\end{algorithmic}
\end{algorithm}
\end{minipage}
\end{center}
\vspace{10px}
Since any individual iteration of Algorithm \ref{alg_GPApproximation} only
involves evaluating a fixed number of matrix vector multiplications and
dot-products in \( \mathbb{R}^{n} \), a full run of \( m \) iterations has a
computational complexity of \( O( m n^{2} ) \).


\subsection{Eigenvector, Lanczos and conjugate gradient posteriors}
\label{ssec_EigenvectorLanczosAndConjugateGradientPosteriors}

Several standard approximation methods for Gaussian processes can be cast in
terms of Algorithm \ref{alg_GPApproximation} via an appropriate choice of the
policies \( ( s_{j})_{ j \le m } \).
However, we cannot hope to obtain a good approximation for arbitrary choices of
the policies.
In fact, we provide a toy example in Remark \ref{rem_InconsistencyExample}
below, showing that for certain bad choices even after $m=n-1$ iterations the
approximation can be inconsistent.  
Under mild assumptions, it was shown in
\cite{WengerEtal2022ComputationalUncertainty} that the approximate posterior
converges to the true posterior as  $m$ converges to $n$, see Theorem 1 and
Corollary 1 in the aforementioned paper.
However, taking $m\asymp n$ would not reduce the computational complexity of
the algorithm compared to the original posterior.

In order to move beyond these type of results, to cover the computationally more
interesting $m=o(n)$ case, the policies have to incorporate relevant information
about the kernel matrix \( K \).
In the following, we will focus on three specific instances of Algorithm
\ref{alg_GPApproximation}. 

\subsubsection{Empirical eigenvector posterior}
\label{sssec_EmpiricalEigenvectorPosterior}

Assuming that our policy is informed by the kernel matrix \( K \), a natural
choice of actions \( ( s_{j} )_{ j \le m } \) is based on the singular value
decomposition
\begin{align}
  K = \sum_{ j = 1 }^{n} 
      \widehat{ \mu }_{j} \widehat{u}_{j} \widehat{u}_{j}^{ \top },
\end{align}
where \( \widehat{ \mu }_{1} \ge \dots \widehat{ \mu }_{n} \ge 0 \) are the
ordered eigenvalues of \( K \) and the eigenvectors 
\( ( \widehat{u}_{j})_{ j \le n } \)
form an orthonormal basis of \( \mathbb{R}^{n} \). 
By setting \( s_{j}: = \widehat{u}_{j} \), \( j \le m \), in each step of
Algorithm \ref{alg_GPApproximation}, we project the residuals onto the search
directions that carries the most information about \( K \).
Since \( K_{ \sigma }^{-1} \) can be expressed in terms of the eigenpairs
\( ( \widehat{ \mu }_{j}, \widehat{u}_j )_{ j \le n } \) as 
\( 
  \sum_{ j = 1 }^{n} 
  ( \widehat{ \mu }_{j} + \sigma^{2} )^{-1}
  \widehat{u}_{j} \widehat{u}_{j}^{ \top }
\),
the following lemma provides a matching form of the approximate precision matrix
\( C_{m} \) under this policy.
The resulting empirical eigenvector posterior will be referred to as EVGP in the
following.

\begin{lemma}[EVGP]
  \label{lem_EVGP}
  Given actions \( s_{j}: = \widehat{u}_{j} \), \( j \le m \), in 
  Algorithm \ref{alg_GPApproximation}, the approximate precision matrix
  \( C_{m} = C_{m}^{ \normalfont \text{EV} } \) 
  of \( K_{ \sigma }^{-1} \) is given by
  \begin{align}
        C_{m}^{ \normalfont \text{EV} }
    & = 
        \sum_{ j = 1 }^{m} 
        \frac{1}{ \widehat{ \mu }_{j} + \sigma^{2} } 
        \widehat{u}_{j} \widehat{u}_{j}^{ \top }.
  \end{align}
\end{lemma}

\noindent A short, inductive derivation can be found in Appendix \ref{prf_EVGP}.

Although this version of Algorithm \ref{alg_GPApproximation} only depends on
empirical quantities, which can be computed from the data, it remains a
substantial idealization, as for large values of \( n \), the eigenpairs \( (
\widehat{ \mu }_{j}, \widehat{u}_{j} )_{ j \le m } \) of \( K \) themselves can
only be accessed via numerical approximation.
Since \( K \) is a random matrix depending on the design, in each instance, the
approximation of the eigenpairs has to be computed alongside Algorithm
\ref{alg_GPApproximation}.
Consequently, this issue cannot be circumvented by approximating the eigenpairs
up to an arbitrary tolerance in advance, as it could be for a fixed
deterministic matrix.
In order to proceed to a fully numerical algorithm, we consider search
directions \( ( s_{j} )_{ j \le m } \) based on approximate eigenvectors of 
\( K \) obtained either from the Lanczos algorithm or conjugate gradient
descent.


\subsubsection{Lanczos eigenvector posterior (LGP)}
\label{sssec_LanczosEigenvectorPosterior}

First, we consider search directions \( ( s_{j} )_{ j \le m } \) based on
approximate eigenvectors of \( K \) obtained from the Lanczos algorithm.
This is one of the standard numerical tools to obtain the singular value
decomposition of a matrix, see, e.g., \texttt{sparse.linalg.svds} from
\texttt{SciPy} \cite{VirtanenEtal2020SciPy}.
The Lanczos algorithm is an orthogonal projection method, see Chapter 4 of
\cite{Saad2011LargeEigenvalueProblems}, based on the Krylov spaces
\begin{align}
  \mathcal{K}_{m}: = \vspan \{ v_{0}, K v_{0}, \dots, K^{ m - 1 } v_{0} \}, 
  \qquad  m = 1, 2, \dots, n,
\end{align}
where we assume that \( v_{0} \in \{ Y / \| Y \|, Z / \| Z \| \} \) is an
initial vector vith length \( \| v_{0} \| = 1 \) either based on the data 
\( Y \) or a \( n \)-dimensional standard Gaussian vector \( Z \).
Given that \( \dim \mathcal{K}_{m} = m \), the algorithm computes an
orthonormal basis \( ( v_{j} )_{ j \le m } \) of \( \mathcal{K}_{m} \).
Approximations of the eigenpairs 
\( ( \widehat{ \mu }_{j}, \widehat{u}_{j} )_{ j \le m } \) are then given
by the first eigenpairs \( ( \tilde \mu_{j}, \tilde u_{j} )_{ j \le m } \) of
\( V V^{ \top } K V V^{ \top } \), where \( V \in \mathbb{R}^{ n \times m } \)
denotes the matrix whose columns contain the \( ( v_{j} )_{ j \le m } \). 
From a purely numerical perspective, we obtain the approximation 
\begin{align}
          \sum_{ j = 1 }^{m} 
          \frac{1}{ \tilde \mu_{j} + \sigma^{2} }
          \tilde u_{j} \tilde u_{j}^{ \top }
  \approx 
          \sum_{ j = 1 }^{m} 
          \frac{1}{ \widehat{ \mu }_{j} + \sigma^{2} } 
          \widehat{u}_{j} \widehat{u}_{j}^{ \top }=C_{m}^{ \text{EV} }
\end{align}
of \( C_{m}^{ \text{EV} } \) from Lemma \ref{lem_EVGP}, which in turn is an
approximation of the empirical precision matrix \( K_\sigma^{-1} \).
Crucially, this approximation coincides with the Lanczos version of Algorithm
\ref{alg_GPApproximation}, where we directly set the policy actions to 
\( s_{j}: = \tilde u_{j} \), \( j \le m \).

\begin{lemma}[LGP]
  \label{lem_LGP}
  For actions \( s_{j} = \tilde u_{j} \), \( j \le m \), the approximation 
  \( C_{m} = C_{m}^{ \normalfont \text{L} } \) of \( K_{ \sigma }^{-1} \) in
  Algorithm \ref{alg_GPApproximation} is given by
  \begin{align}
    C_{m}^{ \normalfont \text{L} } = \sum_{ j = 1 }^{m} 
                                     \frac{1}{ \tilde \mu_{j} + \sigma^{2} } 
                                     \tilde u_{j} \tilde u_{j}^{ \top }.
  \end{align}
\end{lemma}

\noindent The fact that the Lanczos approximation of the empirical eigenvectors
is compatible with Algorithm \ref{alg_GPApproximation} in the above  sense will
later facilitate the analysis of the resulting fully numerical posterior.
The proof of Lemma \ref{lem_LGP} can be found in Appendix \ref{prf_LGP}.


\subsubsection{Conjugate gradient posterior (CGGP)}
\label{sssec_ConjugateGradientPosterior}

Beside the Lanczos approximation, we focus on one other version
of Algorithm \ref{alg_GPApproximation} which is fully numerically tractable.
It is defined by actions stemming from a conjugate gradient approximation of the
solution of \( Y = K_{ \sigma } w \).
CG is one of the standard algorithms to solve linear equations efficiently, see,
e.g., \texttt{sparse.linalg.cg} from \texttt{SciPy}
\cite{VirtanenEtal2020SciPy}.
It is a line search method which iteratively minimizes the quadratic objective
\( \varrho(w): =  ( w^{ \top } K_{ \sigma } w ) / 2 - Y^{ \top } w \),
\( w \in \mathbb{R}^{n} \), along search directions 
\( ( d_{j}^{ \text{CG} } )_{ j \le m } \) satisfying the conjugacy condition
\begin{align}
  \label{eq_APPN_ConjugacyCondition}
  ( d^{ \text{CG} } )^{ \top }_{j} K_{ \sigma } d^{ \text{CG} }_{k} = 0
  \qquad \text{ for all } j, k \le m, j \ne k.
\end{align}
Starting at \( w_{0} = 0 \), the update \( w_{j} \) is chosen such that 
\begin{align}
  \label{eq_APPN_CGLineSearch}
    \varrho( w_{j} )
  = 
    \min_{ t \in \mathbb{R} }
    \varrho( w_{ j - 1 } + t d^{ \text{CG} }_{j} ).
\end{align}
The directions \( ( d_{j}^{CG} )_{ j \le m } \) are then computed alongside the
updates by applying the Gram-Schmidt procedure to the gradients
\begin{align}
  \nabla \varrho( w_{j} ) = K_{ \sigma } w_{j} - Y, 
  \qquad j \le m.
\end{align}
The explicit solution of the line search problem in Equation
\eqref{eq_APPN_CGLineSearch} is given by \(
  t = \| \nabla \varrho( w_{ j - 1 } ) \|^{2} \\ /
      ( ( d^{ \text{CG} }_{j} )^{ \top } K_{ \sigma } d^{ \text{CG} }_{j} ) 
\)
such that the above description defines a full numerical procedure.


\citet{WengerEtal2022ComputationalUncertainty} show that for 
\( s_{j} = d_{j}^{ \text{CG} } \), \( j \le m \), the Bayesian updating
procedure for the representer weights in Equation
\eqref{eq_APPN_BayesianUpdateMean} coincides with the CG iteration.
Consequently, the resulting matrix \( C_{m} = C_{m}^{ \text{CG} } \) from
Algorithm \ref{alg_GPApproximation} can be interpretet as the approximation of
the empirical precision matrix \( K_{ \sigma }^{-1} \), which is implicitly
determined by applying CG to \( K_{ \sigma } w = Y \).
An explicit formula is given by the following lemma, which is proven in 
Appendix \ref{prf_CGGP}.

\begin{lemma}[CGGP]
  \label{lem_CGGP}
  For actions \( s_{j} = d_{j}^{ \normalfont \text{CG} } \), \( j \le m \),
  the approximation \( C_{m} = C_{m}^{ \normalfont \text{CG} } \) of 
  \( K_{ \sigma }^{-1} \) in Algorithm \ref{alg_GPApproximation} is given by
  \begin{align}
        C_{m}^{ \normalfont \text{CG} } 
    & = 
        \sum_{ j = 1 }^{m}
        \frac{1}{ ( d_{j}^{ \normalfont \text{CG} } )^{ \top } K_{ \sigma }
                    d_{j}^{ \normalfont \text{CG} } 
                } 
          d_{j}^{ \normalfont \text{CG} } 
        ( d_{j}^{ \normalfont \text{CG} } )^{ \top }.
  \end{align}
\end{lemma}

\noindent The CG-version of Algorithm \ref{alg_GPApproximation} is of particular
importance.
Not only is it a fully numerical procedure, but CG is one of the default methods
to obtain an approximation of the posterior mean of a Gaussian process
posterior, see
\cite{PleissEtal2018ConstantTimePredictiveDistributions}
and \cite{WangEtal2019ExactGPs}.
Lemma \ref{lem_CGGP} above provides us with an explicit representation of the
approximate covariance matrix. 
This substantially speeds up the computations, as given the directions of the CG
approach \( ( d^{ \text{CG} }_{j} ) \), \( j \le m \) there is no need to run
Algorithm \ref{alg_GPApproximation} to obtain the approximate covariance matrix.
Therefore, we have an explicit, easy to compute representation of the
approximate posterior resulting from the CG descent algorithm, see also the
discussion in Section 4 of \cite{WengerEtal2022ComputationalUncertainty}.




\section{Approximate posterior contraction}
\label{sec_ApproximatePosteriorContraction}

In studying the approximation algorithms from Section
\ref{ssec_EigenvectorLanczosAndConjugateGradientPosteriors}, we take the
frequentist Bayesian perspective.
This provides an objective, universal way of quantifying the performance of
Bayesian procedures, which are inherently subjective by the choice of the
prior.

In our analysis we consider the frequentist data generating model
\begin{align}
  Y_{i} = f_{0}( X_{i} ) + \varepsilon_{i}, \qquad i = 1, \dots, n,
\end{align}
where \( f_0 \) is the underlying, true, functional parameter of interest.
We are interested in how well the posterior in our Bayesian procedure can
recover $f_0$, i.e. how fast the posterior contracts around the true function as the sample size $n$ increases. 
More concretely, for a suitable metric \( d \) on the parameter space, in our
case \( L^{2}(G) \), and a given prior, the corresponding posterior
\( \Pi_{n}( \cdot | X, Y ) \) is said to contract around the truth 
\( f_{0} \in L^{2}(G) \) with rate \( \varepsilon_{n} \), if for any sequence 
\( M_{n} \to \infty \), 
\begin{align}
  \label{eq_APC_GenericContractionRate}
  \Pi_{n} \{ f: d( f, f_{0} ) \ge M_{n} \varepsilon_{n} | X, Y \} \to 0 
\end{align}
in probability under \( \mathbb{P}_{ f_{0} }^{ \otimes n } \) and 
\( n \to \infty \).
Equation \eqref{eq_APC_GenericContractionRate} should be interpreted in the
sense that under the frequentist assumption, i.e. that the data are generated from the
true parameter \( f_{0} \), the posterior asymptotically puts its mass on a ball
of radius \( M_n\varepsilon_{n} \) around the truth \( f_{0} \).
In particular, if \( \varepsilon_{n} \) is the minimax optimal rate for the
frequentist estimation problem, Equation \eqref{eq_APC_GenericContractionRate}
implies that a minimax optimal estimator can be constructed from the Bayesian
method, see e.g. Theorem 8.7 in \cite{GhosalvdVaart2017FundamentalsOfBayes}. 

In the setting of the Gaussian process regression model from Equation
\eqref{eq_I_GPRegressionModel}, we fix \( \sigma^{2} > 0 \) and consider the set
of densities
\begin{align}
      \mathcal{P}: 
  = 
      \Big\{ 
          p_{f}( x, y ) 
        = 
          \frac{1}{ \sqrt{ 2 \pi \sigma^{2} } } 
          \exp \Big( \frac{ - ( y - f(x) )^{2} }{ 2 \pi \sigma^{2} } \Big), 
          f \in L^{2}(G)
      \Big\}, 
\end{align}
with respect to the product measure \( G \otimes \lambda \), where \( \lambda \)
denotes the Lebesgue measure on \( \mathbb{R} \).
By identifying \( f \) with \( p_{f} \), we can define a metric on 
\( L^{2}(G) \) via the Hellinger distance
\begin{align}
     d_{ \text{H} }( f, g ):
  =
     d_{ \text{H} }( p_{f}, p_{g} ): 
  = 
     \sqrt{ \int 
              ( \sqrt{ p_{f} } - \sqrt{ p_{g} } )^{2} 
            \, d G \otimes \lambda 
          },
     \qquad f, g \in L^{2}(G).
\end{align}
The posterior contraction rate in this setting is determined by the behaviour of
the kernel operator corresponding to the covariance kernel \( k \), i.e.
\begin{align}
  \label{eq_APC_KernelOperator}
  T_{k}: L^{2}(G) \to  L^{2}(G), 
  \qquad 
  f \mapsto \int f(x) k( \cdot, x ) \, G( dx ) 
          = \sum_{ j \ge 1 } 
            \lambda_{j} \langle f, \phi_{j} \rangle_{ L^{2} } \phi_{j}, 
\end{align}
where the singular value decomposition of \( T_{k} \) is determined by the
summable eigenvalues \( ( \lambda_{j} )_{ j \ge 1 } \) and a corresponding
orthonormal basis \( ( \phi_{j} )_{ j \ge 1 } \) of \( L^{2}(G) \).
More particularly, we consider the concentration function
\begin{align}
  \label{eq_M_ConcentrationFunction}
      \varphi_{ f_{0} }( \varepsilon ): 
  = 
      \inf_{ h \in \mathbb{H}: \| h - f_{0} \|_{ L^{2} } \le \varepsilon }
      \| h \|_{ \mathbb{H} }^{2} 
    - 
      \log \mathbb{P} \{ \| F \|_{2} < \varepsilon \},
  \qquad \varepsilon > 0,
\end{align}
at an element \( f_{0} \in L^{2}(G) \), where 
\( \mathbb{H}: = \ran T_{k}^{1/2} \) denotes the RKHS induced by the Gaussian
prior on \( L^{2}(G) \).
When \( k \) is sufficiently regular, this space coincides with the RKHS induced
by the kernel mentioned in Remark \ref{rem_RepresenterWeights}, see
\cite{vZantenvdVaart2008RatesForGPPriors} or Chapter 11 of
\cite{GhosalvdVaart2017FundamentalsOfBayes}.
For \( f_{0} = 0 \), the function \( \varphi_{ f_{0} } \) reduces to the small
ball exponent \( -\log \mathbb{P} \{ \| F \|_{ L^{2} } < \varepsilon \} \),
which measures the amount of mass the GP prior \( F \) puts around zero.
For \( f_{0} \ne 0 \), the additional term is referred to as the
\emph{decentering function}, which measures the decrease in mass when shifting
from the origin to \( f_{0} \).
The connection between the concentration function and the contraction rate 
\( \varepsilon_{n} \) is given by a bound on \( \varphi_{ f_{0} } \) that we
formulate as an assumption.
 
\begin{itemize}
  \item [{\color{blue} (A1)}] \textbf{{\color{blue} (CFun)}:} 
    \label{ass_ConcentrationFunctionInequality} 
    An element \( f_{0} \) in the \( L^{2}(G) \)-closure 
    \( \overline{ \mathbb{H} } \) of the reproducing kernel Hilbert space
    satisfies the concentration function inequality for the rate 
    \( \varepsilon_{n} \) if
    \begin{align}
      \label{eq_APC_AssumptionConcentrationFunction}
          \varphi_{ f_{0} }( \varepsilon_{n} ) 
      \le
          C_{ \varphi } n \varepsilon_{n}^{2}, 
      \qquad \text{ for all } n \in \mathbb{N}  
    \end{align}
    and some \( C_{ \varphi } > 0 \).
\end{itemize}

\noindent The (fastest) rate \( \varepsilon_{n} \) in Equation
\eqref{eq_APC_AssumptionConcentrationFunction} is typically determined by the
decay of the eigenvalues  \( ( \lambda_{j} )_{ j \ge 1 } \) of the kernel
operator, see also the discussion in Section \ref{sec_Examples}.
We recall a version of the classical contraction result for Gaussian Process
posteriors.

\begin{proposition} 
  [Standard contraction rate, \cite{vZantenvdVaart2008RatesForGPPriors}]
  \label{prp_StandardContractionRate}
  Assume that at some \( f_{0} \in \overline{ \mathbb{H} } \), the concentration
  function inequality from Assumption
  \hyperref[ass_ConcentrationFunctionInequality]
           {\normalfont \textbf{{\color{blue} (CFun)}}}
  holds for a sequence \( \varepsilon_{n} \to 0 \) with 
  \( n \varepsilon_{n}^{2} \to \infty \).
  Then, for any constant 
  \( C_{2} > 0 \) there exists a constant \( C_{1} > 0 \) such that
  \begin{align}
          \mathbb{E}_{ f_{0} }^{n} ( 
            \Pi_n \{    d_{ \normalfont \text{H} }( \cdot, f_{0} ) 
                    \ge M_{n} \varepsilon_{n} | X, Y 
                  \} 
            \mathbf{1}_{ A_{n} }
          ) 
    & \le
          C_{1} \exp( - C_{2} n \varepsilon_{n}^{2} ),
  \end{align}
  for \( n \) sufficiently large and a sequence 
  \( ( A_{n} )_{ n \in \mathbb{N} } \) with 
  \( \mathbb{P}_{ f_{0} }^{ \otimes n }( A_{n} ) \to 1 \).
\end{proposition}

\noindent In particular, Proposition \ref{prp_StandardContractionRate} implies
the contraction in Equation \eqref{eq_APC_GenericContractionRate}.

In our setting, however, we do not have access to the full posterior, but only
to its numerical approximations  \( ( \Psi_{m} )_{ m \in \mathbb{N} } \) from
Algorithm \ref{alg_GPApproximation}. 
To show that these approximations provide reasonable alternatives, we have to
derive similar contraction rate guarantees for them as for the original
posterior. 
More concretely, under some additional regularity assumptions, we aim to show
that for an appropriately chosen sequence \( m_{n} \to \infty \) the approximate
posterior achieves the same contraction rate \( \varepsilon_n \) as the true
posterior, i.e.,
\begin{align}
  \label{eq_APC_GenericApproximateContractionRate}
  \Psi_{ m_{n} } \{     d_{ \text{H} } ( \cdot, f_{0} ) 
                    \ge M_{n} \varepsilon_{n} 
                 \}
  \to 0 
\end{align}
in probability under \( \mathbb{P}_{ f_{0} }^{ \otimes n } \) as
\( n \to \infty \).
Such contraction rate results were derived in
\cite{NiemanEtal2022ContractionRates} in context of the empirical spectral
features inducing variable variational Bayes method proposed by
\citet{Titsias2009InducingVariablesVB,
       Titsias2009InducingVariablesVBTechnicalReport}
as well. 
Given that this variational approach is a special case of Algorithm
\ref{alg_GPApproximation}, see the derivation in the next section, we recall it
with the corresponding contraction rate results, for later reference.

The inducing variables variational Bayes approach rests on the idea of
summarizing the prior Gaussian Process \( F \) by \( m \) continuous
linear functionals \( U = ( U_{1}, \dots, U_{m} ) \) from \( L^{2}( \Pi ) \).
The true posterior can then be written as 
\begin{align}
  \label{eq_APC_TruePosteriorViaInducingVariables}
      \Pi_{n}( B | X, Y ) 
  & = 
      \int \mathbb{P}^{ F | X, Y, U = u }(B) \, \mathbb{P}^{ U | X, Y }( d u )
\end{align}
for any Borel set from \( L^{2}(G) \). 
Assuming that \( U \) summarizes the information from the data about \( F \)
well, we may approximate the distribution of \( F | X, Y, U \) in Equation
\eqref{eq_APC_TruePosteriorViaInducingVariables} simply by the distribution of
\( F | U \), which is given by the Gaussian process with mean and covariance
functions
\begin{align}
  \label{eq_APC_FConditionalOnU}
  x         \mapsto K_{ x u } K_{ u u }^{-1} U,
  \qquad 
  ( x, x' ) \mapsto k( x, x' ) - K_{ x u } K_{ u u }^{-1} K_{ u x },
\end{align}
respectively, where
\begin{align}
  K_{ux}^{ \top } : = K_{xu}: = ( \text{Cov}( F(x), U_{j} )_{ j \le m })
  \in \mathbb{R}^{m}, 
  \qquad 
  K_{ u u }: = \text{Cov}(U) \in \mathbb{R}^{ m \times m }.
\end{align}
Given that \( U | X, Y \) is distributed according to an \( m \)-dimensional
Gaussian, this motivates the variational class
\begin{align}
  \label{eq_APC_VariationalClass}
      \mathcal{Q}:
  = 
      \Big\{ 
        Q_{ \mu, \Sigma } 
        =
        \int \mathbb{P}^{ F | U = u }( \cdot ) N( \mu, \Sigma )( d u ):
        \mu \in \mathbb{R}^{m}, 0 \le \Sigma \in \mathbb{R}^{ m \times m }
      \Big\}
\end{align}
and approximating the true posterior via 
\begin{align}
  \Psi_{m} \in \argmin_{ Q \in \mathcal{Q} } \kl\big( Q, \Pi_n( \cdot | X, Y ) \big).
\end{align}
The problem above has an explicit solution \( \Psi_{m} = Q_{ \mu^{*},
\Sigma^{*} } \) with 
\begin{align}
  \label{eq_APC_ExplicitSolution}
      \mu^{*} 
  =
      \sigma^{2} 
      K_{ u u } 
      ( \sigma^{ - 2 } K_{ u f } K_{ f u } + K_{ u u } )^{-1}
      K_{ u f } Y,
  \quad 
      \Sigma^{*} 
  =
      K_{ u u } 
      ( \sigma^{ - 2 } K_{ u f } K_{ f u } + K_{ u u }  )^{-1}
      K _{ u u },
\end{align}
where \( 
    K_{uf}: = K_{fu}:
  =
    ( \text{Cov}( F( X_{i} ), U_{j} ) )_{ i \le n, j \le m }
\),
see \cite{Titsias2009InducingVariablesVB,
          Titsias2009InducingVariablesVBTechnicalReport}.

By considering a sufficiently large number of inducing variables, depending on
how well the covariance function \eqref{eq_APC_FConditionalOnU} approximates the
true posterior covariance, the same contraction rate result was derived for the
variational approximation as for the original posterior in
\cite{NiemanEtal2022ContractionRates}. 
In particular, the authors cover the setting
\begin{align}
  \label{eq_APC_EmpiricalEigenvectorVB}
  U_{j}: = \big\langle \widehat{u}_{j}, \big(F( X_{1}), \dots, F(X_{n}) \big) 
           \big\rangle, 
  \qquad j \le m,
\end{align}
where the inducing variables are based on the empirical eigenvectors of the
kernel matrix \( K \). 
This version of the variational Bayes approach is connected to all
three versions of Algorithm \ref{alg_GPApproximation} discussed in
Section \ref{sec_ApproximatePosteriorsFromNrobabilisticNumerics}.
We will repeatedly explore this connection in Sections \ref{sec_MainResults} and
\ref{sec_TechnicalAnalysis}.


\section{Main results}
\label{sec_MainResults}

Our main results establish contraction rates for the approximate posteriors
resulting from the three versions of Algorithm \ref{alg_GPApproximation}
discussed in Section
\ref{ssec_EigenvectorLanczosAndConjugateGradientPosteriors}.
We begin by drawing a connection between the empirical eigenvector version of
Algorithm \ref{alg_GPApproximation} and the variational Bayes approximation
based on the empirical spectral features inducing variables given in
\eqref{eq_APC_EmpiricalEigenvectorVB}.

\begin{lemma}[Empirical eigenvector actions and variational Bayes]
  \label{lem_EmpiricalEigenvectorActionsAndVariationalBayes}
  The approximate posterior \( \Psi^{ \normalfont \text{EV} }_{m} \) given by
  Algorithm \ref{alg_GPApproximation} based on the empirical eigenvector actions 
  \( ( \widehat{u}_{i} )_{ i \le m } \) coincides with the Variational Bayes
  approximate posterior based on the empirical spectral features
  inducing variables \eqref{eq_APC_EmpiricalEigenvectorVB}.
\end{lemma}

\noindent Lemma \ref{lem_EmpiricalEigenvectorActionsAndVariationalBayes}
establishes the equivalence of the Bayesian updating procedure based on the
empirical eigenvectors with the variational Bayes approach. 
We note that under Assumption \hyperref[ass_ConcentrationFunctionInequality]
                    {\normalfont \textbf{{\color{blue} (CFun)}}}
and the additional condition
\begin{align}
  \label{eq_M_EigenvalueCondition}
  \sum_{ j = m_{n} + 1 }^{ \infty } \lambda_{j} \le C \varepsilon_{n}^{2},
  \qquad 
  \mathbb{E} \widehat{ \lambda }_{ m_{n} + 1 } \le C n^{-1},
\end{align}
on the (empirical) eigenvalues, it was shown in Section 5 of
\cite{NiemanEtal2022ContractionRates} that the above variational approximation
contracts with rate $\varepsilon_n$ around the true parameter. 
In view of Lemma
\ref{lem_EmpiricalEigenvectorActionsAndVariationalBayes}, this implies the same
 contraction rate \( \varepsilon_n \) for the idealized empirical eigenvector
version of Algorithm \ref{alg_GPApproximation}. 
A rigorous formulation of this statement is given in Remark
\ref{rem_VariationalBayesResult}, while the proof of the above lemma is deferred
to Appendix \ref{prf_EmpiricalEigenvectorActionsAndVariationalBayes}.

The main theoretical contribution of this paper is the derivation of contraction
rates for the approximate posterior resulting from the Lanczos and the
CG-versions of Algorithm \ref{alg_GPApproximation}.
Contrary to the empirical eigenvector or variational Bayes approximation, these
constitute fully numerical procedures.
We briefly discuss the importance of this aspect of our results.
The eigenvectors \( ( \widehat{u}_{j} )_{ j \le m } \) of the matrix \( K \) are
empirical quantities, i.e., they can be computed from the observed data.
Algorithms based on explicit versions of the
\( ( \widehat{u}_{j} )_{ j \le m } \), however, still constitute substantial
idealizations, since, except for very small sample sizes \( n \), the empirical
eigenvectors can only be accessed via numerical approximation. 
Standard algorithms to obtain the singular value decompositition of an 
\( n \times n \)  matrix up to the \( m \)-th eigenvector, such as the Lanczos
iteration, have a computational complexity of \( O( m n^{2} ) \).
This is a significant improvement compared to the inversion of \( K_{ \sigma
}^{-1} \) for the true posterior.
Crucially, however, standard guarantees for these type of algorithms are
expressed in terms of spectral gaps of the target matrix, see Theorems
\ref{thm_LanczosEigenvalueBound} and \ref{thm_LanczosEigenvectorBound}.
In our case, the matrix \( K \) is itself a random object depending on the
design.
This means that the approximate singular value decomposition (SVD) has to be
computed alongside Algorithm \ref{alg_GPApproximation} for each data set. 
However, the standard guarantees for the approximation of the SVD are
uninformative in this case, since they are themselves expressed in terms of
random quantities.
Consequently, there is a true theoretical gap between results for approximate
posteriors stated in terms of the \( ( \widehat{u}_{j} )_{ j \le m } \) and
fully numerical procedures such as the Lanczos iteration and the conjugate
gradient descent algorithms.

Most of the theory developed in this work goes toward bridging the theoretical
gap discussed above.
For the Lanczos version of the algorithm, this translates to establishing that
the approximate eigenpairs 
\( ( \tilde \lambda_{j}, \tilde u_{j} )_{ j \le m } \) replicate the empirical
eigenpairs \( ( \widehat{ \lambda }_{j}, \widehat{u}_{j} )_{ j \le m } \) well
with high probability. 
Then, in Corollary \ref{cor_EquivalenceBetweenLGPandCGGP} below, we
derive a so far unexplored connection between the Lanczos iteration and CG
descent algorithms, by showing that the latter lives essentially on the same
Krylov space as the former.
This, in turn, results in the same contraction rate guarantees for the CG as for
the Lanczos iteration.

More concretely, approximation of the empirical eigenpairs requires that the
empirical eigenvalues \( ( \widehat{ \lambda }_{j} )_{ j \le m } \) concentrate
around their population counterparts \( ( \lambda_{j} )_{ j \le m } \)
well enough as to translate classical bounds for the
Lanczos algorithm in terms of the eigenpairs
\( ( \widehat{ \lambda }_{j}, \widehat{u}_{j} )_{ j \le n } \) into bounds
in terms of the corresponding population quantities with high probability.
We discuss this in detail in Section
\ref{ssec_AnalysisOfTheLanczosApproximatePosterior}.
In order to obtain this type of concentration, we need to employ recently
developed numerical analytics and spectral techniques from
\cite{JirakWahl2023RelativePerturbationBounds}, requiring additional
assumptions beside 
Assumption \hyperref[ass_ConcentrationFunctionInequality]
                    {\normalfont \textbf{{\color{blue} (CFun)}}} 
and the condition in Equation \eqref{eq_M_EigenvalueCondition} needed for the
contraction of the idealized, eigenvector version of Algorithm
\ref{alg_GPApproximation}. 

To begin with, for convenience, we assume that the eigenvalues of the population
kernel operator \( T_{k} \) are simple, the eigenvalue function is convex and
satisfies certain regularity assumptions.
\begin{itemize}
  \item [{\color{blue} (A2)}] \textbf{{\color{blue} (SPE)}:} 
    \label{ass_SimplePopulationEigenvalues} 
    The population eigenvalues \( ( \lambda_{j} )_{ j \ge 1 } \) of \( T_{k} \)
    are simple, i.e., \( \lambda_{1} > \lambda_{2} > \dots > 0 \).
  \item [{\color{blue} (A3)}] \textbf{{\color{blue} (EVD)}:} 
    \label{ass_EigenvalueDecay} 
    We assume the following decay behaviour of the population eigenvalues:
    \begin{enumerate}[label=(\roman*)]

      \item There exists a convex function \( \lambda: [ 0, \infty ) \to [ 0,
        \infty ) \) such that \( \lambda_{j} = \lambda(j) \) and \( \lim_{ j \to
        \infty } \lambda(j) = 0 \).

      \item There exists a constant \( C > 0 \) such that, \( \lambda( C j ) \le
        \lambda(j) / 2 \) for all \( j \in \mathbb{N} \).

      \item There exists a constant \( c > 0 \) such that \( \lambda_{j} \ge e^{
          - c j } \) for all \( j \in \mathbb{N} \).

    \end{enumerate}

\end{itemize}

\noindent We also impose a moment condition on the Karhunen-Lo\`eve coefficients
of the Hilbert space valued random variable \( k( \cdot, X_{1} ) \in \mathbb{H}
\). 

\begin{itemize}
  \item [{\color{blue} (A4)}] \textbf{{\color{blue} (KLMom)}:} 
    \label{ass_KarhunenLoeveMoments}
    There exists a \( p > 4 \), such that the Karhunen-Lo\`eve coefficients 
    \( \eta_{j}: = \langle k( \cdot, X_{1} ), \phi_{j} \rangle_{ \mathbb{H} } \)
    of \( k( \cdot, X_{1} ) \) satisfy
    \begin{align}
      \label{eq_M_KarhunenLoeveCondition}
      \sup_{ j \ge 0 } \mathbb{E} | \eta_{j} |^{p} < \infty,
    \end{align}
    where \( \phi_{j} \) denotes the \( j \)-th eigenfunction of the kernel
    operator \( T_{k} \) from Equation \eqref{eq_APC_KernelOperator}.
\end{itemize}

\noindent We briefly discuss the conditions above.  
Assumptions \hyperref[ass_EigenvalueDecay]
            {\normalfont \textbf{{\color{blue} (EVD)}}}(ii) and (iii)
guarantee that the eigenvalue decay is not too slow or fast respectively.
Importantly, this include the standard settings of polynomially and
exponentially decaying eigenvalues, see Section \ref{sec_Examples} for
particular examples.
Note, that by the reproducing property of the kernel 
\( \eta_{j} = \phi_{j}( X_{1} ) \), where \( \phi_{j} \) denotes the \( j \)-th
eigenfunction of the kernel operator \( T_{k} \).
Therefore, \hyperref[ass_KarhunenLoeveMoments]{\normalfont \textbf{{\color{blue}
(KLMom)}}} can also be understood as a moment condition on the eigenfunctions
of \( T_{k} \). 
Together, Assumptions
\hyperref[ass_SimplePopulationEigenvalues]
         {\normalfont {\color{blue} (A2)}}-\hyperref[ass_KarhunenLoeveMoments]
                                           {\normalfont {\color{blue} (A4)}}
are instrumental in order to guarantee that the empirical eigenvalues
concentrate sufficiently well around their population counterparts.
This is developed in detail in Section \ref{sec_TechnicalAnalysis}, which
results in the following contraction guarantees for the Lanczos and CG
algorithms.

\begin{theorem}[Contraction rates for LGP and CGGP]
  \label{thm_ApproximateContractionRates}
  Under Assumptions
  \hyperref[ass_SimplePopulationEigenvalues]
           {\normalfont \textbf{{\color{blue} (SPE)}}},
  \hyperref[ass_EigenvalueDecay]{\normalfont \textbf{{\color{blue} (EVD)}}},
  \hyperref[ass_KarhunenLoeveMoments]
           {\normalfont \textbf{{\color{blue} (KLMom)}}}, 
  let \( f_{0} \in \overline{ \mathbb{H} } \cap L^{ \infty }(G) \) satisfy the
  concentration function inequality from Assumption
  \hyperref[ass_ConcentrationFunctionInequality]
           {\normalfont \textbf{{\color{blue} (CFUN)}}},
  for a sequences $ \varepsilon_{n} \to 0 $ with 
  \( n \varepsilon_{n}^{2} \to \infty \).
  Further, let condition \eqref{eq_M_EigenvalueCondition} hold for a sequence 
  \( m_{n} \) satisfying
  \( 
    C' \log n \le m_{n} = o( \sqrt{n} / \log n \land
                             ( n^{(p/4 - 1) / 2} \log^{p/8 - 1} n )
                           )
  \)
  for some \( C' > 0 \) sufficiently large.
  Then, both the {\normalfont LGP} and the {\normalfont CGGP} approximate
  posteriors based on \( m_{n} \log n \) actions contract around \( f_{0} \)
  with rate \( \varepsilon_{n} \), i.e., for any sequence 
  \( M_{n} \to \infty \),
  \begin{align}
    \Psi_{ m_{n} \log n } 
    \{ d_{ \normalfont \text{H} }( \cdot, f_{0} ) \ge M_{n} \varepsilon_{n} \} 
    \xrightarrow[ ]{ n \to \infty } 0
  \end{align}
  in probability under \( \mathbb{P}^{ \otimes n }_{ f_{0} } \) and 
  \( n \to \infty \). 
\end{theorem}

As discussed above, in view of 
Lemma \ref{lem_EmpiricalEigenvectorActionsAndVariationalBayes}, the eigenvector
version of Algorithm \ref{alg_GPApproximation} is equivalent to the empirical
spectral features inducing variables variational approach, hence the results of
Theorem  \ref{thm_ApproximateContractionRates} hold for this, idealized case as
well. 
We discuss this in more details in the following remark.

\begin{remark}[Relation to variational Bayes]
  \label{rem_VariationalBayesResult}
  \
  \begin{enumerate}[label=(\alph*)]
    \item Based on the equivalence in Lemma
      \ref{lem_EmpiricalEigenvectorActionsAndVariationalBayes} and the reasoning
      in \cite{NiemanEtal2022ContractionRates}, the
      same result as in Theorem \ref{thm_ApproximateContractionRates} holds for
      Algorithm \ref{alg_GPApproximation} based on the empirical eigenvector
      actions with \( \Psi_{ m_{n} } \) instead of \( \Psi_{ m_{n} \log n } \)
      requiring only
      Assumption \hyperref[ass_ConcentrationFunctionInequality]
                 {\normalfont \textbf{{\color{blue} (CFun)}}} 
      and condition \eqref{eq_M_EigenvalueCondition}. 
    \item Since due to Lemma \ref{lem_LGP}, the Lanczos version of Algorithm
      \ref{alg_GPApproximation} simply replaces the empirical eigenpairs with
      their approximate counterparts, the equivalence from Lemma
      \ref{lem_EmpiricalEigenvectorActionsAndVariationalBayes}
      implies that the result in Theorem \ref{thm_ApproximateContractionRates}
      can also be interpreted as a guarantee for a fully numerical version of
      the variational Bayes posterior.
  \end{enumerate}
\end{remark}

We also note, that compared to the empirical eigenvector actions the CG and
Lanczos algorithms require an additional multiplicative slowly varying factor. 
In the above theorem for simplicity we used a logarithmic factor, but this can
be further reduced. 
The reason for the larger iteration number is due to the approximation error of
the iterative algorithms for estimating the empirical eigenpairs. 
To achieve sufficient recovery of the space spanned by the first $m_n$
eigenvectors the CG and Lanczos methods need to be run a bit longer. 
This phenomena is also investigated in our numerical analysis in Section
\ref{sec_NumericalSimulations}. 
Finally, we show in the following remark that the above contraction rate results
cannot hold in general, for arbitrary policies in 
Algorithm \ref{alg_GPApproximation}.

\begin{remark}[Inconsistency example]
  \label{rem_InconsistencyExample}
  Theorem \ref{thm_ApproximateContractionRates}, does not hold in general for
  arbitrary policies, even if \( m \) is taken to be arbitrarily close (but not
  equal) to \( n \).  
  We demonstrate this in a simple example related to the empirical eigenvalues
  method.  
  Let us consider, for instance, the policies 
  \( s_{j}: = \widehat{u}_{j + 1} \), \( j \le m \). 
  Then even for \( m = n - 1 \), Algorithm \ref{alg_GPApproximation} (with this
  choice of policy) results in an inconsistent approximation for the posterior
  and inconsistent estimator for \( f_0 \), see Section
  \ref{sec_AuxiliaryResults} for details. 
  Hence, beside the pathological case  \( m = n \), the behaviour of Algorithm
  \ref{alg_GPApproximation} cannot be assessed in full generality but only in
  specific cases.
\end{remark}

\noindent In the following section, we show how 
Theorem \ref{thm_ApproximateContractionRates} translates to minimax optimal
convergence rates in standard settings.


\section{Examples}
\label{sec_Examples}

In order to demonstrate the applicability of our general contraction rate
theorem, we assume that \( \mathfrak{X} \) is given by \( \mathbb{R}^{d} \)
or \( [ 0, 1 ]^{d} \) for some \( d \in \mathbb{N} \) and consider random
series priors, where we endow the coefficients of an orthonormal basis \( (
\phi_{j} )_{ j \ge 1 } \) of \( L^{2}(G) \) with independent mean zero
Gaussian distributions resulting in a centered Gaussian process prior, i.e.,
\begin{align}
  \label{eq_E_RSP}
  F = \sum_{ j = 1 }^{ \infty } \sqrt{ \lambda_{j} } Z_{j} \phi_{j}
  \qquad \text{ with } \qquad 
  Z_{j} \sim N( 0, 1 ) \text{ i.i.d.}
\end{align}
and a non-negative, summable sequence \( ( \lambda_{j} )_{ j \ge 1 } \).
By Lemma 2.1 in \cite{GhosalvdVaart2017FundamentalsOfBayes}, \( F \) defines a
prior on \( L^{2}(G) \).
In the following, we assume that 
\begin{align}
  \label{eq_E_KLMomForRSP}
  \sup_{ j \ge 1 } \mathbb{E} | \phi_{j}( X_{1} )  |^{p} < \infty 
\end{align}
for \( p > 4 \) from Assumption \hyperref[ass_KarhunenLoeveMoments]
                                {\normalfont \textbf{{\color{blue} (KLMom)}}}.
By taking second moments, it can be checked that
\begin{align}
  k( x, x' ): = \sum_{ j = 1 }^{ \infty } 
                \lambda_{j} \phi_{j}(x) \phi_{j}( x' )
  \quad \text{and} \quad 
  k(x, x) = \sum_{ j = 1 }^{ \infty } 
            \lambda_{j} \phi_{j}(x)^{2},
  \qquad x, x' \in \mathfrak{X}
\end{align}
converge \( G^{ \otimes 2 } \)- and \( G \)-almost surely respectively.
By setting \( F \), \( k \) and \( ( \phi_{ j } )_{ j \ge 1 } \) to zero on
suitable nullsets, \( k \) defines the covariance kernel of the process 
\( F( x )_{ x \in \mathfrak{ X } } \) which has well defined point evaluations
as required in the setup of the model in Equation
\eqref{eq_I_GPRegressionModel}.
Finally, \( ( \lambda_{j}, \phi_{j} )_{ j \ge 1 } \) is the eigensystem of the
kernel operator \( T_{k} \) confirming that the condition Equation
\eqref{eq_E_KLMomForRSP} is in fact equivalent to Assumption
\hyperref[ass_KarhunenLoeveMoments]{\normalfont \textbf{{\color{blue}
(KLMom)}}} in this setting.

We investigate the two most common structure, i.e. when the variances
\( ( \lambda_j )_{j \in \mathbb{N}} \) of the Gaussian distributions are
polynomially or exponentially decaying. 

\subsection{Polynomially decaying eigenvalues}
\label{ssec_PolynomiallyDecayingEigenvalues}

First, we consider polynomially decaying coefficients 
\( ( \lambda_j )_{j \in \mathbb{N}} \), i.e. the functional parameter \( f \) is
endowed with the prior
  \begin{align}
  \label{eq_E_PolynomiallyDecayingPrior}
  F(x) = \sum_{ j = 1 }^{ \infty } 
           \tau j^{-1/2 - \alpha / d} Z_j \phi_j(x),
           \qquad x \in \mathfrak{X},
\end{align}
where \( \alpha > 0 \) and \( \tau \) are the regularity and scale
hyperparameters of the process, respectively. 
Such polynomially decaying eigenvalues 
\( \lambda_j \asymp \tau^2 j^{-1 - 2 \alpha / d} \) are quite standard. 
For instance, the popular Mat\'ern kernel or the fractional Brownian motion
possesses such eigenstructure, see
\cite{Seeger2007addendum,Bronski2003Asymptotics}.
The theoretical properties of the corresponding posterior is also well
investigated, see for instance Section 11.4.5 of
\cite{GhosalvdVaart2017FundamentalsOfBayes} or
\cite{KnapikEtal2011bayesian,SzaboEtal2013EmpirialBayes}.
Choosing the rescaling factor 
\( \tau = n^{ ( \alpha - \beta ) / ( 2 \beta + d ) } \) results
in rate optimal contraction rate when estimating Sobolev $\beta$-smooth
functions, i.e., 
\begin{align}
  f_0 \in S^{\beta}(L): = \{ f \in L^{2}(G): \| f \|_{ S^{ \beta } }^2 \le L \}
  \qquad \text{with} \qquad  
  \| f \|_{ S^{ \beta } }^2: = \sum_{ j = 1 }^{ \infty }
                               j^{2 \beta / d} \langle f, \phi_j \rangle^2.
\end{align}
We also take this optimal choice of the scaling parameter and show below that
the approximate posterior resulting from both the Lanczos iteration and
conjugate gradient descent algorithms achieve the minimax optimal contraction
rate if they are run for at least  $m_n\geq n^{d/(2\beta+d)}\log n$ iterations.
The proof of the corollary is deferred to Section
\ref{prf_PolynomiallyDecayingEigenvalues}.

\begin{corollary}[Polynomially decaying eigenvalues]
  \label{cor_PolynomiallyDecayingEigenvalues}
  Consider the non-parametric regression model \eqref{eq_I_GPRegressionModel}
  with \( f_0 \in S^{ \beta }(L) \cap L^{ \infty }(G) \), and the GP prior
  \eqref{eq_E_PolynomiallyDecayingPrior} with fixed regularity hyperparameter
  \(  0 < \alpha \) satisfying \( d / 2 < \beta \le \alpha + d / 2 \) and scale
  hyperparameter \( \tau = n^{ ( \alpha - \beta ) / ( 2 \beta + d ) } \).
  Then, as long as condition \eqref{eq_E_KLMomForRSP} is satisfied with 
  \( p > 4 + 8 d / ( 2 \beta + d ) \), both LGP and CGGP with iteration
  number \( m_n \ge n^{d / (2 \beta + d)} \log n \) achieve minimax posterior
  contraction rates, i.e., for arbitrary \( M_{n} \to \infty \),
  \begin{align*}
    \Psi_{m_n} \{ f: d_{ \normalfont \text{H} }( f, f_0 ) 
                  \ge M_n n^{-\beta/(d+2\beta)} | X,Y
               \} 
    \rightarrow 0,
  \end{align*}
  in probability under \( \mathbb{P}_{ f_{0} }^{ \otimes n } \) and \( n \to
  \infty \).
\end{corollary}

\begin{remark}[Rescaled hyperparameter \( \tau \).]
  \label{rem_RescaledHyperparameterTau.}
  We choose \( \tau = \tau_{n} \) depending on \( n \). 
  Consequently, the underlying population eigenvalues 
  \( ( \lambda_{j} )_{j \ge 1} = ( \lambda_{j}^{(n)} )_{j \ge 1}  \) will
  themselves depend on \( n \).
  Our results, however, remain applicable as long as
  Assumptions \hyperref[ass_SimplePopulationEigenvalues]
              {\normalfont \textbf{{\color{blue} (SPE)}}},
          and \hyperref[ass_EigenvalueDecay]
              {\normalfont \textbf{{\color{blue} (EVD)}}}
  remain satisfied for the \( ( \lambda_{j}^{(n)} )_{j \ge 1} \).
\end{remark}


\subsection{Exponentially decaying eigenvalues}
\label{ssec_ExponentiallyDecayingEigenvalues}

Next, we consider exponentially decaying eigenvalues for the covariance kernel. 
Since such kernel would result in infinitely smooth functions one has to
appropriately rescale the prior. 
Therefore we consider GP priors of the form
\begin{align}
  \label{eq_E_ExponentiallyDecayingPrior}
  F(x) = \sum_{ j = 1 }^{ \infty } e^{ - \tau j^{ 1/d }} Z_j \phi_j(x),
  \qquad x \in \mathfrak{X},
\end{align}
with scale parameter \( \tau > 0 \). 
We note that the highly popular squared exponential covariance kernel (with
respect to the standard Gaussian base measure) has similar, exponentially
decaying eigenstructure.  
The theoretical properties of the posterior associated to such priors are also
well studied in the literature, see for instance
\cite{PatiBhattacharya2015Adaptive,CastilloEtal2014ThomasBayes} for contraction
and \cite{HadjiSzabo2021BayesianUQ} for frequentist coverage of the credible
sets.
We note that in principle, one could also consider various extension of this
prior, for instance by allowing an additional multiplicative (space) scaling
factor \( \tau^q \) for some \( q \in \mathbb{R} \).
Our proof techniques could be extended to such cases as well in a
straightforward, but somewhat cumbersome manner. 
However, for simplicity of presentation we do not consider the most general
class one could cover here.

\begin{corollary}[Exponentially decaying eigenvalues]
  \label{cor_ExponentiallyDecayingEigenvalues}
  Consider the non-parametric regression model \eqref{eq_I_GPRegressionModel}
  with \( f_0 \in S^{ \beta }(L) \cap L^{ \infty }(G) \) and the random
  series prior \eqref{eq_E_ExponentiallyDecayingPrior} with scale hyperparameter
  \( \tau = n^{ -1 / ( 2 \beta + d ) } \log n \) and \( \beta > d / 2 \).
  Then as long as condition \eqref{eq_E_KLMomForRSP} is satisfied with 
  \( p > 4 + 8 d / ( 2 \beta + d ) \), both LGP and CGGP with iteration
  number \( m_n \ge n^{ d / ( d + 2 \beta ) } \log n \) achieve minimax
  posterior contraction rates (up to a logarithmic factor), i.e., for any
  sequence \( M_{n} \to \infty \), 
  \begin{align*}
    \Psi_{m_n} \{ f: d_{ \normalfont \text{H} }(f, f_0) \ge M_n
    n^{-\beta/(d+2\beta)}\log^{1/2} n|X,Y
    \} 
    \rightarrow 0,
  \end{align*}
  in probability under \( \mathbb{P}_{ f_{0} }^{ \otimes n } \) and 
  \( n \to \infty \). 
\end{corollary}

\noindent The proof of the corollary is deferred to Section
\ref{prf_ExponentiallyDecayingEigenvalues}.

\begin{remark}
  The logarithmic factor in the contraction rate is an artifact of the proof technique based on the
  concentration inequality
  \hyperref[ass_ConcentrationFunctionInequality]
           {\normalfont \textbf{{\color{blue} (CFUN)}}}. 
  One can achieve minimax posterior contraction rates using kernel ridge
  regression techniques, see for instance Corollary 12 of \cite{Nieman2023UncertaintyQuantification} 
  (with \( m \) taken to be equal to \( n \) to get back the full posterior).
  However, in this case only a polynomially decaying
  upper bound is given for the expectation of the posterior mass outside of the
  \( M_n n^{ - \beta / ( d + 2 \beta ) } \)-radius ball centered at \( f_0 \). 
  This, however, does not permit the use of Proposition
  \ref{prp_ContractionOfApproximation}, see also Theorem 5 in
  \cite{RaySzabo2019VariationalBayes}, requiring exponential upper bounds for
  this probability on a large enough event.
\end{remark}



\section{Technical analysis}
\label{sec_TechnicalAnalysis}

\subsection{Approximate contraction via Kullback-Leibler bounds}
\label{ssec_ApproximateContractionViaKullbackLeiblerBound}

In view of \cite{RaySzabo2019VariationalBayes}, sufficiently controlling the
Kullback-Leibler divergence between the posterior and the approximating measure
results in the same contraction rate for the approximation as for the original
posterior. 
In the following proposition, we slightly reformulate their result adapted to
our setting.
For completeness, a proof is in Appendix \ref{prf_ContractionOfApproximation}.

\begin{proposition}[Contraction of approximation]
  \label{prp_ContractionOfApproximation}
  Under the assumptions of Proposition \ref{prp_StandardContractionRate}, let 
  \( ( \nu_{n} )_{ n \in \mathbb{N} } \) be a sequence of distributions such
  that for any sequence \( M_{n}' \to \infty \), there exist events \( A_{n}' \)
  such that 
  \begin{align}
        \kl( \nu_{n}, \Pi_{n}( \cdot | X, Y ) ) \mathbf{1}_{ A_{n}' }
    \le 
        n M_{n}'^{2} \varepsilon_{n}^{2} 
    \qquad \text{ and } \qquad 
    \mathbb{P}^{ \otimes n }_{ f_{0} }( A_{n}' ) \to 1.
  \end{align}
  Then, for all sequences \( M_{n} \to \infty \) 
  \begin{align}
    \nu_{n} \{     d_{ \normalfont \text{H} }( \cdot, f_{0} ) 
               \ge M_{n} \varepsilon_{n} 
            \} 
    \to 0
  \end{align}
  in probability under \( \mathbb{P}_{ f_{0} }^{ \otimes n } \) and 
  \( n \to \infty \). 
\end{proposition}

In order to derive Theorem \ref{thm_ApproximateContractionRates} via Proposition
\ref{prp_ContractionOfApproximation}, we need to bound the Kullback-Leibler
divergence between the approximate posterior \( \Psi_{m} \) and 
\( \Pi_{n}( \cdot | X, Y ) \).
Conveniently, it can be shown that the Kullback-Leibler divergence between the
measures on the function space \( L^{2}(G) \) coincides with the
Kullback-Leibler divergence between the finite dimensional Gaussians at the
design points, see Lemma \ref{lem_KLBetweenPosteriorProcesses}.
Therefore, we obtain
\begin{align}
  \label{eq_T_KLDDecomposition}
          2 \kl( \Psi_{m}, \Pi_{n}( \cdot | X, Y ) ) 
  & = 
          2 \kl( N( K K_{ \sigma }^{-1} Y, K - K K_{ \sigma }^{-1} K ),
                 N( K C_{m} Y,             K - K C_{m}             K ) ) 
  \\
  & = 
          \tr ( K - K K_{ \sigma }^{-1} K )^{-1} ( K - K C_{m} K ) - n 
  \notag
  \\
  & + 
          Y^{ \top } ( K_{ \sigma }^{-1} - C_{m} ) K 
                     ( K - K K_{ \sigma }^{-1} K )^{-1}
                     K ( K_{ \sigma }^{-1} - C_{m} ) Y
  \notag
  \\
  & + 
          \log \det ( [ K - K C_{m} K )^{-1} [ K - K K_{ \sigma }^{-1} K ] )
  \notag
  \\
  & =:
          (\text{I}) + (\text{II}) + (\text{III}).
  \notag
\end{align}
Since \( K_{ \sigma }^{-1} - C_{m} = \Gamma_{m} \ge 0 \), we have 
\( K - K C_{m} K \ge K - K K_{ \sigma }^{-1} K \).
This implies that the \( \log \)-determinant in the third term is negativ.
Setting 
\( 
  C_{m}^{ \text{EV} } = \sum_{ j = 1 }^{m} 
                        ( \widehat{ \mu }_{j} + \sigma^{2} )^{-1} 
                        \widehat{u}_{j} \widehat{u}_{j}^{ \top } 
\),
we can bound the remainder via
\begin{align}
  \label{eq_T_PreliminaryKLBound}
        ( \text{I} ) + ( \text{II} )
  & =
        \tr ( K - K K_{ \sigma }^{-1} K )^{-1} ( K - K C_{m} K ) - n 
      + 
        \| ( K_{ \sigma }^{-1} - C_{m} ) Y 
        \|_{ K ( K - K K_{ \sigma }^{-1} K )^{-1} K }^{2}
  \notag
  \\
  & \le 
        \tr ( K - K K_{ \sigma }^{-1} K )^{-1}          K 
            ( K_{ \sigma }^{-1} - C^{ \text{EV} }_{m} ) K
      + 
        2 \| ( K_{ \sigma }^{-1} - C^{ \text{EV} }_{m} ) Y 
          \|_{ K ( K - K K_{ \sigma }^{-1} K )^{-1} K }^{2}
  \notag
  \\
  & + 
        \tr ( K - K K_{ \sigma }^{-1} K )^{-1} K
            ( C^{ \text{EV} }_{m} - C_{m} )    K
      + 
        2 \| ( C_{m} - C^{ \text{EV} }_{m} ) Y 
          \|_{ K ( K - K K_{ \sigma }^{-1} K )^{-1} K }^{2},
\end{align}
where \( \| \cdot \|_{A} \) denotes the norm induced by the dot-product 
\( \langle \cdot, A \cdot \rangle \).
The two terms depending only on \( C_{m}^{ \text{EV} } \) will be
straightforward to analyze, since all relevant matrices are jointly
diagonalizable with respect to the true empirical projectors 
\( ( \widehat{u}_{j} )_{ j \le n } \).
The remainder crucially depends on the difference 
\( C_{m} - C_{m}^{ \text{EV} } \).
In case of the Lanczos posterior, it is given by
\begin{align}
  \label{eq_T_CDifferenceLVsEV}
    C_{m}^{ \text{L} } - C^{ \text{EV} }_{m} 
  =
    \sum_{ j = 1 }^{m} 
    \Big( \frac{1}{ \tilde \mu_{j} + \sigma^{2} } 
          \tilde u_{j} \tilde u_{j}^{ \top } 
        - 
          \frac{1}{ \widehat{ \mu } + \sigma^{2} } 
          \widehat{u}_{j} \widehat{u}_{j}^{ \top }
    \Big).
\end{align}
In Section \ref{ssec_AnalysisOfTheLanczosApproximatePosterior}, we develop a
rigorous analysis of the Lanczos algorithm that allows us to treat this
difference, resulting in the following Kullback-Leibler bound.

\begin{proposition}[Kullback-Leibler bound]
  \label{prp_KullbackLeiblerBound}
  Under Assumptions
  \hyperref[ass_SimplePopulationEigenvalues]
           {\normalfont \textbf{{\color{blue} (SPE)}}},
  \hyperref[ass_EigenvalueDecay]{\normalfont \textbf{{\color{blue} (EVD)}}}, and
  \hyperref[ass_KarhunenLoeveMoments]
           {\normalfont \textbf{{\color{blue} (KLMom)}}},
  let \( f_{0} \in \overline{ \mathbb{H} } \cap L^{ \infty }(G) \) satisfy the
  concentration function inequality from Assumption
  \hyperref[ass_ConcentrationFunctionInequality]
           {\normalfont \textbf{{\color{blue} (CFUN)}}}
  for a sequence \( \varepsilon_{n} \to 0 \) with 
  \( n \varepsilon_{n}^{2} \to \infty \).
  Additionally, let \( m_{n} \) be a sequence that satisfies 
  \( 
    C' \log n \le m_{n} 
                = o( ( \sqrt{n} / \log n ) \land 
                     ( n^{ ( p/4 - 1 ) / 2 } ( \log n )^{ p/8 - 1 } ) )
  \)
  for some \( C' > 0 \) sufficiently large and consider the Lanczos algorithm
  \ref{alg_Lanczos} iterated for \( m_{n} \log n \) steps initialized at 
  \( v_{0} \in \{ Y / \| Y \|, Z / \| Z \| \} \), where \( Z \) is a 
  \( n \)-dimensional standard Gaussian.
  Then, for any sequence \( M_{n} \to \infty \), the approximate posterior 
  \( \Psi_{m} \) from Algorithm \ref{alg_GPApproximation} based on 
  \( m = m_{n} \log n \) Lanczos actions satisfies the bound
  \begin{align}
            \kl( \Psi_{ m_{n} \log n }, \Pi_{n}( \cdot | X, Y ) ) 
    & \le 
            \frac{ M_{n} n }{ \sigma^{2} } 
            \Big( 
              \varepsilon_{n}^{2}
            + 
              \sum_{ j = m_{n} + 1 }^{ \infty } \lambda_{j}
            + 
              n \varepsilon_{n}^{2} \mathbb{E} \widehat{ \lambda }_{ m_{n} + 1 } 
            \Big) 
  \end{align}
  with probability converging to one under 
  \( \mathbb{P}_{ f_{0} }^{ \otimes n } \) and \( n \to \infty \).
\end{proposition}

\noindent The proof of Proposition \ref{prp_KullbackLeiblerBound} is deferred to
Appendix \ref{prf_KullbackLeiblerBound}.
Propositions \ref{prp_ContractionOfApproximation} and
             \ref{prp_KullbackLeiblerBound} then together
imply Theorem \ref{thm_ApproximateContractionRates} for the Lanczos version of
Algorithm \ref{alg_GPApproximation}.

Since the conjugate gradient actions span the same Krylov space as the Lanczos
actions,  the CG version of Algorithm \ref{alg_GPApproximation} obtains the same
upper bound for the KL divergence as in Proposition
\ref{prp_KullbackLeiblerBound}.
We formalize this statement in the following corollary, which is proven in
Appendix \ref{prf_EquivalenceBetweenLGPandCGGP}.

\begin{corollary}[Equivalence of LGP and CGGP]
  \label{cor_EquivalenceBetweenLGPandCGGP}
  For any integer \( m \ge 1 \), the approximate posterior from Algorithm
  \ref{alg_GPApproximation} based on \( m \) CG-actions is identical to the one
  resulting in from the Lanczos iteration with $m$ steps and starting value
  $v_0=Y / \| Y \|$.
  Consequently, the bound from Proposition \ref{prp_KullbackLeiblerBound} also
  holds for the CG-approximate posterior under the same conditions.
\end{corollary}

\begin{remark}[CGGP as an approximation of variational Bayes]
  \label{rem_CGAsVariationalBayes}
  In settings in which the numerical inversion of \( K_{ \sigma } \) is
  infeasible, one of the standard approaches in order to compute the posterior
  mean is to apply CG to \( K_{ \sigma } w = Y \),
  see \cite{PleissEtal2018ConstantTimePredictiveDistributions} and
      \cite{WangEtal2019ExactGPs}.
  There, CGGP is often interpreted as an exact version of the posterior with a
  preset tolerance level.
  Corollary \ref{cor_EquivalenceBetweenLGPandCGGP}, however, provides a new
  interpretation for this approach:
  EVGP is equivalent to the variational Bayes algorithm based on empirical
  spectral features inducing variables,
  see \cite{Titsias2009InducingVariablesVB,burt2019rates} and 
  Lemma \ref{lem_EmpiricalEigenvectorActionsAndVariationalBayes}.
  Since LGP is a numerical approximation of EVGP and LGP is equivalent to CGGP,
  the conjugate gradient algorithm can also be interpreted as an implicit
  implementation of a specific variational Bayes method.
  Given its numerical advantages, in many circumstances, CGGP may therefore even
  be preferable to an explicit implementation of a variational procedure.
\end{remark}


\subsection{Analysis of the Lanczos approximate posterior}
\label{ssec_AnalysisOfTheLanczosApproximatePosterior}

In this section, we analyze the application of the Lanczos algorithm to the
kernel matrix \( K \) in our probabilistic setting.
In the following, most of the results will apply to the first \( m = m_{n} \)
eigenvalues or eigenvectors under the assumption that we use a Krylov space
\begin{align}
      \mathcal{K}_{ \tilde m }:
  =
      \vspan \{ v_{0}, K v_{0}, \dots, K^{ \tilde m - 1 } v_{0} \}, 
  \qquad \text{ with } \qquad 
  \| v_{0} \| = 1 
\end{align}
and \( \tilde m = m_n \log n \) in the Lanczos algorithm.
For notational convenience, we will also use the normalized kernel matrix 
\( A: = n^{-1} K \), i.e. we consider the empirical eigenpairs 
\( ( \widehat{ \lambda }_{j}, \widehat{u}_{j} )_{ j \le m } \) and their Lanczos
counterparts \( ( \tilde \lambda_{j}, \tilde u_{j} )_{ j \le m } \) of this normalized matrix.  
In Appendix \ref{prf_KrylovSpaceDimension}, we prove that under the following
assumption, the Krylov space \( K_{ \tilde m } \) has full dimension.

\begin{enumerate}
  \item [{\color{blue} (A5)}] \textbf{{\color{blue} (LWdf)}:} 
    \label{ass_LanczosWelldefined} 
    The eigenvalues of \( A = n^{-1} K \) satisfy 
    \( \widehat{ \lambda }_{1} > \dots > \widehat{ \lambda }_{ \tilde m } > 0 \)
    and \( \langle v_{0}, \widehat{u}_{j} \rangle \ne 0 \) for all
    \( j \le \tilde m \), where \( v_{0} \) is the initial vector of the Lanczos
    algorithm.
\end{enumerate}

\begin{lemma}[Krylov space dimension]
  \label{lem_KrylovSpaceDimension}
  Under Assumption 
  \hyperref[ass_LanczosWelldefined]{\normalfont \textbf{{\color{blue} (LWdf)}}},
  \( \dim \mathcal{K}_{ \tilde m } = \tilde m \).
\end{lemma}

\noindent In this setting, consequently, the following formal algorithm is well
defined.

\begin{center}
\begin{minipage}{.8\linewidth}
  \begin{algorithm}[H]
  \caption{Lanczos algorithm}
  \label{alg_Lanczos}
  \begin{algorithmic}[1]
    \Procedure{ITERLanczos\(( K, v_{0}, \tilde m ) \) }{}
    \vspace{2px}
    \State Initialize \( v_{0} \) with \( \| v_{0} \| = 1 \).
    \State Compute ONB \( v_{1}, \dots, v_{ \tilde m } \) of 
           \( \mathcal{K}_{ \tilde m } \). 
    \State \( V \gets ( v_{1}, \dots, v_{ \tilde m } ) \).
    \State \( A \gets n^{-1} K \). 
    \State Compute eigenpairs 
           \( ( \tilde \lambda_{j}, \tilde u_{j} )_{ j \le \tilde m } \)
           of \( V^{ \top } A V  \). 
    \State \( \tilde u_{j}  \gets V \tilde u_{j} \), \( j \le \tilde m \). 
    \EndProcedure
    \State \Return
           \( ( \tilde \lambda_{j}, \tilde u_{j} )_{ j \le \tilde m } \).
\end{algorithmic}
\end{algorithm}
\end{minipage}
\end{center}

\noindent Since \( V^{ \top } A V \) is the restriction of \( A \) onto 
\( \vspan \{ v_{j}: j \le \tilde m \} \) in terms of the basis 
\( ( v_{j} )_{ j \le \tilde m } \), the Lanczos algorithm computes the
eigenpairs of \( P_{ \tilde m } A |_{ K_{ \tilde m } } \), where 
\( |_{ K_{ \tilde m } } \) denotes the restriction onto 
\( K_{ \tilde m } \) and \( P_{ \tilde m } \) is the orthogonal projection
onto \( K_{ \tilde m } \).

\begin{lemma}[Elementary properties of Lanczos eigenquantities]
  \label{lem_ElementaryPropertiesOfLanczosEigenquantities}
  Under Assumption \hyperref[ass_LanczosWelldefined]{\normalfont
  \textbf{{\color{blue} (LWdf)}}}, the approximate objects \( ( \tilde
  \lambda_{j}, \tilde u_{j} )_{ j \le \tilde m } \) from Algorithm
  \ref{alg_Lanczos} satisfy the following properties:
  \begin{enumerate}[label=(\roman*)]
    \item \( \tilde u_{1}, \dots, \tilde u_{ \tilde m } \) are an ONB of 
      \( \mathcal{K}_{ \tilde m } \),

    \item \( 
      ( A - \tilde \lambda_{j} I ) \tilde u_{j} \perp \mathcal{K}_{ \tilde m } 
      \) for all \( j = 1, \dots, \tilde m \),

    \item \( \tilde \lambda_{j} \le \widehat{ \lambda }_{j} \) for all 
      \( j = 1, \dots, \tilde m \).
  \end{enumerate}
\end{lemma}

\begin{proof}[Proof]
  Property (i) is immediate.
  (ii) is the Galerkin condition (4.17) from
  \cite{Saad2011LargeEigenvalueProblems}.
  (iii) follows from the restriction of \( A \) onto 
  \( \mathcal{K}_{ \tilde m } \) in Algorithm \ref{alg_Lanczos}.
  A formal derivation can be found in Corollary 4.4 in
  \cite{Saad2011LargeEigenvalueProblems}.
\end{proof}

\noindent We state two classical bounds for the Lanczos eigenpairs in our
setting, which we have adapted from \cite{Saad1980LanczosConvergenceRates}.
The derivations are in Appendix \ref{prf_LanczosEigenvalueBound}.
For the eigenvalues, the following bound holds, where 
\( \tan( \hat u_{i}, v_{0}  ) \) denotes the tangens of the acute angle between
\( \hat u_{i} \) and \( v_{0} \).

\begin{theorem} 
  [Lanczos: Eigenvalue bound, \cite{Saad1980LanczosConvergenceRates}]
  \label{thm_LanczosEigenvalueBound}
  Under
  Assumption \hyperref[ass_LanczosWelldefined]
                      {\normalfont \textbf{{\color{blue} (LWdf)}}},
  for any fixed integer \( i \le \tilde m < n \) with 
  \( \tilde \lambda_{ i - 1 } > \widehat{ \lambda }_{i}  \) if \( i > 1 \)
  and any integer 
  \( \tilde p \le \tilde m - i  \), the eigenvalue approximation from Algorithm
  \ref{alg_Lanczos} satisfies
  \begin{align}
    \label{eq_LanczosEigenvalueBound}
        0 
    \le
        \widehat{ \lambda }_{i} - \tilde \lambda_{i} 
    \le 
        ( \widehat{ \lambda }_{i} - \widehat{ \lambda }_{n} ) 
        \Big( 
          \frac{ \tilde \kappa_{i} \kappa_{ i, \tilde p }
                 \tan( \widehat{u}_{i}, v_{0} )
               }
               { T_{ \tilde m - i - \tilde p }( \gamma_{i} ) }
        \Big)^{2},
  \end{align}
  where 
  \( 
      \gamma_{i}:
    =
      1 
    + 2 ( \widehat{ \lambda }_{i} - \widehat{ \lambda }_{ i + \tilde p + 1 } )
      / ( \widehat{ \lambda }_{ i + \tilde p + 1 } - \widehat{ \lambda }_{n} ) 
  \),
  \begin{align}
        \tilde \kappa_{i}: 
    = 
        \prod_{ j = 1 }^{ i - 1 } 
        \frac{ \tilde \lambda_{j} - \widehat{ \lambda }_{n} }
             { \tilde \lambda_{j} - \widehat{ \lambda }_{i} },
    \qquad 
        \kappa_{ i, \tilde p }: 
    =
        \prod_{ j = i + 1 }^{ i + \tilde p } 
        \frac{ \widehat{ \lambda }_{j} - \widehat{ \lambda }_{n} }
             { \widehat{ \lambda }_{i} - \widehat{ \lambda }_{j} },
    \qquad 
  \end{align}
  and \( T_{l} \) denotes the \( l \)-th Tschebychev polynomial.
\end{theorem}

\noindent We comment on the interpretation of the above bound.
The Tschebychev polynomials satisfy the lower bound
\begin{align}
  \label{eq_T_TschebychevLowerBound}
  T_{l}(x) & \ge c | x |^{l}, \qquad | x | \ge 1, 
\end{align}
see Chapter 4 in \cite{Saad2011LargeEigenvalueProblems}.
Then, for fixed kernel matrix $K$ and index $i$, the quantities
$\tilde{\kappa}_i$ and $ \kappa_{ i, \tilde p }$ can be considered as constants,
and since \( \gamma_{i} > 1 + c \) is bounded away from one, the upper bound in
\eqref{eq_LanczosEigenvalueBound} decreases exponentially fast in \( \tilde m \).

Noting that the Hilbert-Schmidt norm between two eigenprojectors can be
expressed as the sine  of the acute angle between the two vectors, i.e.
\begin{align}
      \| 
         \widehat{u}_{j} \widehat{u}_{j}^{ \top }
       - 
         \tilde u_{j} \tilde u_{j}^{ \top }
      \|_{ \text{HS} } 
  & = 
      2 - 2 \langle \widehat{u}_{j}, \tilde u_{j} \rangle
      \tr ( \widehat{u}_{j} \tilde u_{j}^{ \top } )
  = 
      2 ( 1 - \langle \widehat{u}_{j}, \tilde u_{j} \rangle^{2} ) 
  = 
      2 \sin^{2} ( \tilde u_{j}, \widehat{u}_{j} ),
\end{align}
 a similar bound holds for the above difference as well.

\begin{theorem} 
  [Lanczos: Eigenvector bound \cite{Saad1980LanczosConvergenceRates}]
  \label{thm_LanczosEigenvectorBound}
  Under Assumption
  \hyperref[ass_LanczosWelldefined]{\normalfont \textbf{{\color{blue} (LWdf)}}},
  for any fixed \( i \le \tilde m \) let 
  \( ( \tilde \lambda^{*}, \tilde u^{*} ) \) be the approximate eigenpair from
  Algorithm \ref{alg_Lanczos} that satisfies 
  \( 
      \widehat{ \lambda }_{i} - \tilde \lambda^{*} 
    =
      \min_{ j \le \tilde m } \widehat{ \lambda }_{i} - \tilde \lambda_{j} 
  \). 
  Then, for any integer \( \tilde p \le \tilde m - i  \), we have
  \begin{align}
          \frac{1}{2} 
          \| \tilde u^{*} \tilde u^{ * \top } 
           - \widehat{u}_{i} \widehat{u}_{i}^{ \top }
          \|_{ \text{HS} }^{2}
    & = 
          \sin^{2}( \tilde u^{*}, \widehat{u}_{i} ) 
    \le 
          \Big( 1 + \frac{ \| K \|_{ \text{op} } }{ n \delta_{i}^{2} }
          \Big)
          \Big( \frac{ \kappa_{i} \kappa_{ i, \tilde p } 
                       \tan( \widehat{u}_{i}, v_{0} )
                     }
                     { T_{ \tilde m - i - \tilde p }( \gamma_{i} ) }
          \Big)^{2},
  \end{align}
  where 
  \( 
      \delta_{i}^{2}:
    =
      \min_{ \tilde \lambda_{j} \ne \tilde \lambda^{*} } 
      | \widehat{ \lambda }_{i} - \tilde \lambda_{j} |
  \),
  \( 
        \gamma_{i}:
    =
        1 
      +
        2 ( \widehat{ \lambda }_{i} - \widehat{ \lambda }_{ i + \tilde p + 1 } ) 
        / ( \widehat{ \lambda }_{ i + \tilde p + 1 } - \widehat{ \lambda }_{n} ) 
  \),
  \begin{align}
        \kappa_{i}:
    =
        \prod_{ j = 1 }^{ i - 1 } 
        \frac{ \widehat{ \lambda }_{j} - \widehat{ \lambda }_{n} }
             { \widehat{ \lambda }_{j} - \widehat{ \lambda }_{i} },
    \qquad 
        \kappa_{ i, \tilde p }: 
    =
        \prod_{ j = i + 1 }^{ i + \tilde p } 
        \frac{ \widehat{ \lambda }_{j} - \widehat{ \lambda }_{n} }
             { \widehat{ \lambda }_{i} - \widehat{ \lambda }_{j} }
  \end{align}
  and \( T_{l} \) denotes the \( l \)-th Tschebychev polynomial.
\end{theorem}

\noindent For a fixed kernel matrix $K$ and index $i$, this yields the same
geometric convergence in \( \tilde m \) as before.

For the results in Theorem \ref{thm_ApproximateContractionRates}, however, we
have to treat the kernel matrix as a random object, i.e., the bounds in Theorems
\ref{thm_LanczosEigenvalueBound} and \ref{thm_LanczosEigenvectorBound} are
themselves random and only provide guarantees insofar they can be restated with 
high probability in terms of deterministic population quantities.
Further, we need guarantees for the Lanczos eigenpairs up to index 
\( m_{n} \), which grows when \( n \to \infty \).
We shortly illustrate the essential challenge this poses via the inverse
empirical eigengap
\( 
  1 / ( \widehat{ \lambda }_{ m_{n} - 1 }
      - \widehat{ \lambda }_{ m_{n} } 
      )
\)
that appears in the term \( \kappa_{ m_{n} } \).
Note that \( A = n^{-1} K \) has the same eigenvalues as the operator
\begin{align}
  \widehat{ \Sigma }:
  \mathbb{H} \to     \mathbb{H}, \qquad 
           h \mapsto \frac{1}{n} \sum_{ j = 1 }^{n} 
                     \langle h, k( \cdot, X_{i} ) \rangle_{ \mathbb{H} } 
                     k( \cdot X_{i} ), 
\end{align}
which is the empirical version of the non-centered covariance operator 
\(
  \Sigma:
  \mathbb{H} \to \mathbb{H}, 
  h \mapsto \mathbb{E} ( \langle h, k( \cdot, X_{i} ) \rangle_{ \mathbb{H} }
                         k( \cdot, X_{i} ) 
                       )
\).
\( \Sigma \) is the restriction of the kernel operator \( T_{k} \) to \(
\mathbb{H} \) and has the same eigenvalues \( ( \lambda_{j} )_{ j \ge 1 } \),
see also the proof of Proposition \ref{prp_RelativePerturbatonBounds}.
In this sense the \( ( \widehat{ \lambda }_{j} )_{ j \ge 1 } \) are empirical
versions of the \( ( \lambda_{j} )_{ j \ge 1 } \).
In order to control the inverse empirical eigengap above, up to a constant, we
therefore need to be able to replace the empirical eigengap with its population
counterpart \( \lambda_{ m_{n} - 1 } - \lambda_{ m_{n} } \) in the bound with
high probability.
This, however, requires that the empirical eigengap converges to the population
eigengap with a faster rate than the population gap converges to zero, since 
\( \lambda_{ m_{n} - 1 } - \lambda_{ m_{n} } \to 0 \) for 
\( m_{n} \to \infty \).
This is a strong requirement in the sense that classical concentration results
for the spectrum of \( n^{-1} K \), see for instance
\cite{ShaweTaylorWilliams2002Stability} and
\cite{ShaweTaylorEtal2002Concentration}, do not provide a sharp enough control. 
To the best of our knowledge, only the recently developed theory 
in \cite{JirakWahl2023RelativePerturbationBounds} is able to deliver
relative perturbation bounds for the eigenvalues which are precise enough to
address this problem.
Proposition \ref{prp_RelativePerturbatonBounds} is adapted to our setting from
Corollary 4 in \cite{JirakWahl2023RelativePerturbationBounds} and its derivation
is deferred to Appendix \ref{prf_RelativePerturbationBounds}.

\begin{proposition} 
  [Relative perturbaton bounds, \cite{JirakWahl2023RelativePerturbationBounds}]
  \label{prp_RelativePerturbatonBounds}
  Under Assumptions 
  \hyperref[ass_SimplePopulationEigenvalues]
           {\normalfont \textbf{{\color{blue} (SPE)}}} and
  \hyperref[ass_KarhunenLoeveMoments]
           {\normalfont \textbf{{\color{blue} (KLMom)}}},
  fix \( m < m_{0} \le n \) such that 
  \( \lambda_{ m_{0} } \le \lambda_{m} / 2 \) and further assume that
  \begin{align}
    \label{eq_RelativeRankCondition}
    \mathbf{r}_{i}( \Sigma ): 
    = 
    \sum_{ k \ne i } 
    \frac{ \lambda_{k} }{ | \lambda_{i} - \lambda_{k} | } 
    + 
    \frac{ \lambda_{i} }
         { ( \lambda_{ i - 1 } - \lambda_{i} ) \land
           ( \lambda_{i} - \lambda_{ i + 1 } )
         }
    \le 
    C \sqrt{ \frac{n}{ \log n } },
    \\
    \qquad \text{ for all } i \le m.
    \notag
  \end{align}
  Then, the eigenvalues of \( A = n^{-1} K \) satisfy the relative perturbation
  bound
  \begin{align}
    \label{eq_RelativePerturbationBound}
    \Big| \frac{ \widehat{ \lambda }_{i} - \lambda_{i} }{ \lambda_{i} } \Big| 
    & \le C \sqrt{ \frac{ \log n }{n} } 
    \qquad \text{ for all } i \le m 
  \end{align}
  with probability at least 
  \( 1 - m_{0}^{2} ( \log n )^{ - p / 4 } n^{ 1 - p / 4 } \).
\end{proposition}

In our setting, under Assumption
\hyperref[ass_EigenvalueDecay]{\normalfont \textbf{{\color{blue} (EVD)}}},
the relative rank \( \mathbf{r}_{i}( \Sigma ) \) can then be bounded up to a
constant by \( m \log m \) uniformly in \( i \le m \), see Lemma
\ref{lem_ConvexFunctionDecay}.
Proposition \ref{prp_RelativePerturbatonBounds} immediately yields that the
Krylov space \( \mathcal{K}_{ \tilde m } \) is well defined asymptotically
almost surely, see Lemma \ref{lem_WellDefinednessOfTheKrylovSpace}.
From there, we obtain control over the quantities 
\( \gamma_{i}, \tilde \kappa_{i} \kappa_{i}, \kappa_{ i, \tilde p } \)
and \( \delta_{i} \) in Theorems \ref{thm_LanczosEigenvalueBound} and
                                     \ref{thm_LanczosEigenvectorBound}.
Finally, this translates to high probability bounds for the Lanczos eigenpairs
purely in terms of population quantities.

\begin{proposition}[Probabilistic bounds for Lanczos eigenpairs]
  \label{prp_ProbabilisticBoundsForLanczosEigenpairs}
  Under Assumptions 
  \hyperref[ass_SimplePopulationEigenvalues]
           {\normalfont \textbf{{\color{blue} (SPE)}}},
  \hyperref[ass_EigenvalueDecay]{\normalfont \textbf{{\color{blue} (EVD)}}} and
  \hyperref[ass_KarhunenLoeveMoments]
           {\normalfont \textbf{{\color{blue} (KLMom)}}},
  set the Lanczos iteration number to \( \tilde m = m \log n \)
  with 
  \( 
    m = o( ( \sqrt{n} / \log n ) \land 
           ( n^{ ( p/4 - 1 ) / 2 } ( \log n )^{ p/8 - 1 } ) ) 
  \).
  Then, the approximate eigenvalues $ ( \tilde \lambda_{i} )_{ i \le m } $ from
  Algorithm \ref{alg_Lanczos} started at 
  \( v_{0} \in \{ Z / \| Z \|, Y / \| Y \| \} \) from Lemma
  \ref{lem_WellDefinednessOfTheKrylovSpace} with 
  \( f_{0} \in L^{ \infty } (G) \) satisfy
  \begin{align}
          0 
    \le 
          \widehat{ \lambda }_{i} - \tilde \lambda_{i} 
    \le 
          \lambda_{i}  ( 1 + c )^{ - m }  
    \qquad \text{ and } \qquad 
        \| \tilde u_{i}    \tilde u_{i}^{ \top } 
         - \widehat{u}_{i} \widehat{u}_{i}^{ \top }
        \|_{ \text{HS} } 
    \le 
        ( 1 + c )^{ - m }  
    \\
    \qquad \text{ for all } i \le m
    \notag
  \end{align}
  with probability converging to one.
\end{proposition}

\noindent A proof of Proposition
\ref{prp_ProbabilisticBoundsForLanczosEigenpairs} is given in Appendix
\ref{prf_ProbabilisticBoundsBorLanczosEigenpairs}.
Finally, the established bounds are sufficient to treat the difference in
Equation \eqref{eq_T_CDifferenceLVsEV} and establish Proposition
\ref{prp_KullbackLeiblerBound}.



\section{Numerical simulations}
\label{sec_NumericalSimulations}

In this section, we illustrate our theoretical findings via numerical
simulations based on synthetic data sets.
We consider two of the most frequently used kernels, the Mat\'ern and squared
exponential kernel. 
Although strictly speaking, these are not fully covered by our theoretical
analysis (as there are no known upper bounds derived for the corresponding
Karhunen-Lo\`eve coefficients in the literature), they possess polynomial and
exponentially decaying eigenvalues, respectively, considered in Section
\ref{sec_Examples}.
The \texttt{python} code for our simulations is available from the website of
the corresponding author.\footnote{
  \href{https://www.bstankewitz.com}{bstankewitz.com}
}

\subsection{Mat\'ern covariance kernel}
\label{ssec_MaternCovarianceKernel}

We consider design points \( X_{i} \stackrel{iid}{\sim} \text{Unif}( 0, 1 ) \)
and observations $Y_i$, $i=1,...,n$, from the model
\begin{align}
  \label{eq_N_MaternModel}
  Y_{i} = f_{0}( X_{i} ) + \varepsilon_{i}, \qquad i = 1, \dots, n,
\end{align}
where \( \varepsilon_{i} \stackrel{iid}{\sim} N( 0, \sigma^{2} ) \) with 
\( \sigma = 0.2 \) and
\begin{align}
  f_{0}(x) = | x - 0.4 |^{ \beta } - | x - 0.2 |^{ \beta },
             \qquad x \in \mathbb{R},
\end{align}
with regularity hyperparameter \( \beta = 0.6 \).
Then, we endow the functional parameter with a centered GP prior defined by the
Mat\'ern kernel
\begin{align}
  k(x, x')  = \frac{1}{ \Gamma( \alpha ) 2^{ \alpha - 1 } }
              ( \sqrt{ 2 \alpha } | x' - x | )^{ \alpha } 
              B_{ \alpha }( \sqrt{ 2 \alpha } | x' - x | ), 
              \qquad x, x' \in \mathbb{R},
\end{align}
where \( B_{ \alpha } \) is a modified Bessel function and $\alpha$ is the
regularity hyperparameter of the prior, see \cite{RasmussenWilliams2006GPForML}.
To obtain optimal posterior inference we match the regularity of the prior with
the regularity of the true parameter by choosing $\alpha=\beta$.

We begin by comparing the results of LGP and CGGP for small \( n \) and 
\( m \) to illustrate the equivalence from Corollary
\ref{cor_EquivalenceBetweenLGPandCGGP}.
Figure \ref{fig_N_LGPEquivalentCGGPMatern} clearly demonstrates that the
posteriors are identical.
In all pictures we plot the posterior means in solid and the $95\%$ pointwise
credible bands by dashed lines. 
The true posterior is denoted by green, the approximation by blue and the true
function by black. 
In the following, we therefore focus only on the CGGP version of Algorithm
\ref{alg_GPApproximation}.

\begin{figure}[H]
  \centering
  \begin{subfigure}[t]{0.40\textwidth}
      \includegraphics[width=\textwidth]{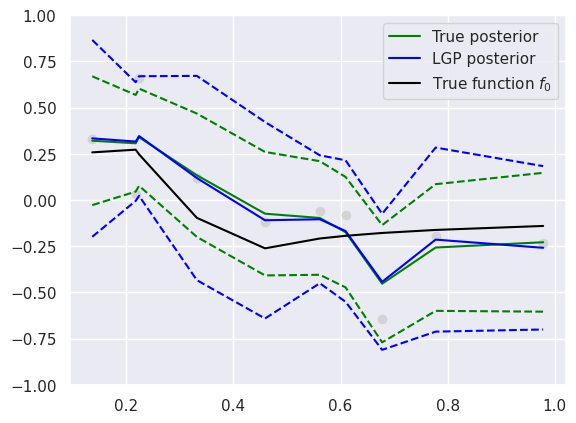}
      \caption{LGP posterior \( \Psi^{ \text{L} }_{5} \).}
  \end{subfigure}
  \begin{subfigure}[t]{0.40\textwidth}
      \includegraphics[width=\textwidth]{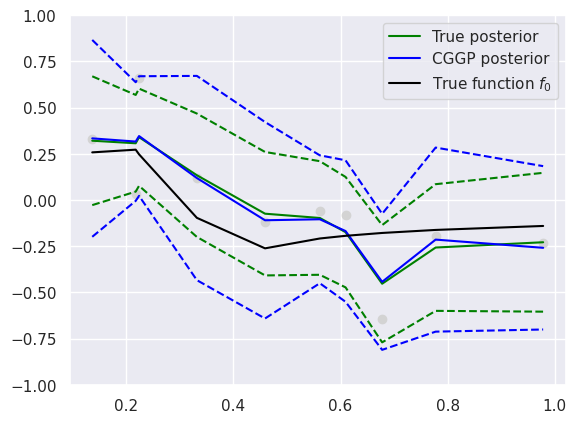}
      \caption{CGGP posterior \( \Psi^{ \text{CG} }_{5} \).}
  \end{subfigure}
  \captionsetup{width=.8\linewidth}
  \caption{Equivalence of LGP and CGGP for \( n = 10 \) observations from model
           \eqref{eq_N_MaternModel}.
           LGP and CGGP posteriors \( \Psi_{m} \) with \( m = 5 \).
          }
  \label{fig_N_LGPEquivalentCGGPMatern}
\end{figure}

Then, we investigate the accuracy of the approximation resulting from different
number of iterations. 
We recall from \cite{vdVaartvZanten2011InformationRates} that for  
\( f_{0} \in C^{ \beta } [ 0, 1 ] \cap H^{ \beta } [ 0, 1 ] \) and a Mat\'ern
process with $\alpha=\beta$,
Assumption \hyperref[ass_ConcentrationFunctionInequality] 
           {\normalfont \textbf{{\color{blue} (CFun)}}} 
holds, resulting in the minimax  contraction rate 
\( \varepsilon_{n} = n^{ -\beta / ( 2 \beta + 1 )} \) for the true posterior,
where \( C^{ \beta } [ 0, 1 ] \) and \( H^{ \beta } [ 0, 1 ] \) denote
the spaces of \( \beta \)-H\"older continuous functions and \( \beta \)-Sobolev
functions respectively. 
Furthermore, Corollary 7 of \cite{NiemanEtal2022ContractionRates} together with
the equivalence in Lemma
\ref{lem_EmpiricalEigenvectorActionsAndVariationalBayes} guarantees that the
EVGP posterior \( \Psi^{ \text{EV} }_{m} \) also contracts with the rate 
\( \varepsilon_{n} = n^{ - \beta / ( 1 + 2 \beta)} \) as long as 
\begin{align}
  m \ge n^{ 1 / ( 1 + 2 \beta ) } \approx 40
  \qquad \text{ for } n = 3000.
\end{align}
This is confirmed by the plot in Figure \ref{fig_N_MaternKernel_M40}, which
reproduces the results from \cite{NiemanEtal2022ContractionRates}. 
For the empirical eigenvectors, we use the full SVD of the kernel matrix
computed via \texttt{linalg.svd} from \texttt{NumPy} \cite{HarrisEtal2020NumPy}.

Based on our theoretical results we can expect the same optimal contraction rate
to hold for the CGGP posterior as well, as long as the Lanczos approximation of
the eigenpairs of the kernel matrix is sufficient.
For a Krylov space of dimension \( 40 \), however, we cannot expect convergence
for all 40 eigenpairs, which is the reason that Theorem
\ref{thm_ApproximateContractionRates} requires an additional slowly increasing
multiplicative factor (in the theorem a logarithmic term was introduced, but a
slower factor is also sufficient).
Correspondingly, the CGGP posterior for \( m = 40 \), which is equivalent to the
LGP posterior with \( m = 40 \) based on the Krylov space 
\( \mathcal{K}_{ 40 } \), displays slightly worse approximation behaviour than
the idealized EVGP one.
This slight suboptimality can be seen by the wider credible bands even if the
posterior mean is already replicated close to exactly and the mean squared error
(MSE)
\begin{align}
      \text{MSE} 
  =
      \frac{1}{n} \sum_{ i = 1 }^{n} 
      ( \mu_{m}^{ \text{CG} }( X_{i} ) - f_{0}( X_{i} )  )^{2} 
\end{align}
of the CGGP posterior mean \( \mu_{m}^{ \text{CG} } \) is essentially identical
to the MSE of the true posterior mean, see Table \ref{tab_N_MaternMSE}.

To incorporate the slightly larger iteration number needed for the CG and
Lanczos algorithms we also run CGGP twice as long, for $m=80$ iterations. 
Note that the logarithmic factor from our theorem would imply a $\log n=8$ (for
\( n = 3000 \)) multiplier compared to the EVGP case. 
However, as discussed above, a smaller term is also sufficient. 
In fact, the documentation of the implementation \texttt{sparse.linalg.svds} of
the Lanczos algorithm from \texttt{SciPy} \cite{VirtanenEtal2020SciPy}
suggests also to use twice as large dimensional Krylov space as the number of
approximated eigenpairs needed by the user to guarantee convergence.
The results for the CGGP posterior with \( m = 80 \) iterations then provide an
approximation that is highly similar to the EVGP, see Figure
\ref{fig_N_MaternKernel_M80M20} and Table \ref{tab_N_MaternMSE}.

\begin{figure}[H]
  \centering
  \begin{subfigure}[t]{0.40\textwidth}
      \includegraphics[width=\textwidth]{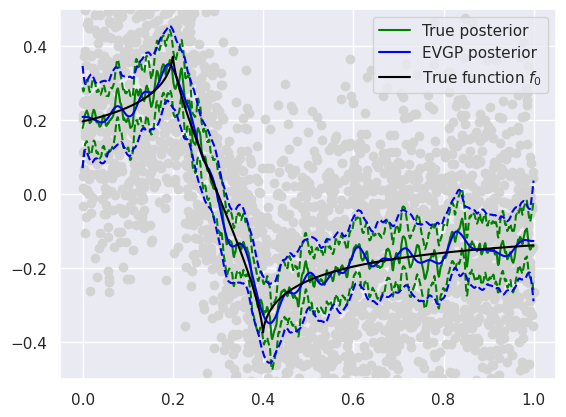}
      \caption{EVGP \( \Psi_{40}^{ \text{EV} } \).}
  \end{subfigure}
  \begin{subfigure}[t]{0.40\textwidth}
      \includegraphics[width=\textwidth]{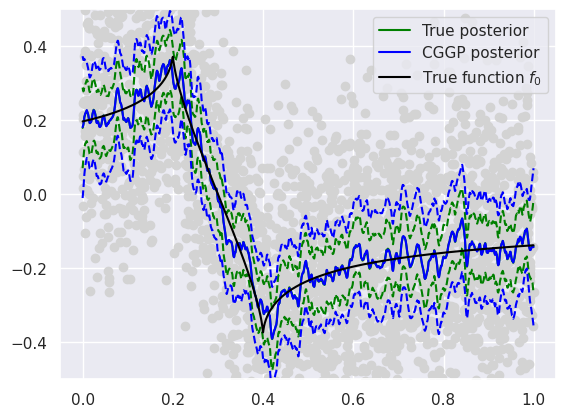}
      \caption{CGGP \( \Psi_{ 40 }^{ \text{CG} } \).}
  \end{subfigure}
  \captionsetup{width=.8\linewidth}
  \caption{Simulation results for the Matern kernel with \( n = 3000 \)
           observations from model \eqref{eq_N_MaternModel}.
           EVGP and CGGP posteriors \( \Psi_{m} \) with \( m = 40 \).
          }
  \label{fig_N_MaternKernel_M40}
\end{figure}

\begin{figure}[H]
  \centering
  \begin{subfigure}[t]{0.40\textwidth}
      \includegraphics[width=\textwidth]{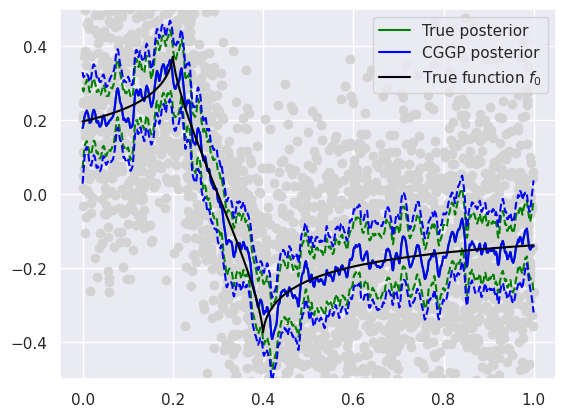}
      \caption{CGGP \( \Psi_{80}^{ \text{CG} } \).}
  \end{subfigure}
  \begin{subfigure}[t]{0.40\textwidth}
      \includegraphics[width=\textwidth]{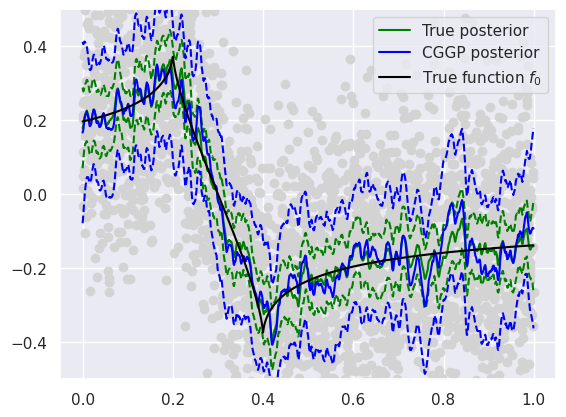}
      \caption{CGGP \( \Psi_{20}^{ \text{CG} } \).}
  \end{subfigure}
  \captionsetup{width=.8\linewidth}
  \caption{Simulation results for the Matern kernel with \( n = 3000 \)
           observations from model \eqref{eq_N_MaternModel}.
           CGGP posteriors \( \Psi^{ \text{CG} }_{m} \) with \( m = 80 \) and $20$, respectively.
          }
  \label{fig_N_MaternKernel_M80M20}
\end{figure}

Considering only half as many iterations than for the EVGP method, i.e., \( m =
20 \), as suggested by our theory, results in an approximation that is
substantially worse than the true posterior. 
This is indicated by an order of magnitude larger MSE and credible bands that
are wider by about a factor of two than for the true posterior.
Increasing the number of iterations substantially beyond what is suggested by
our theory allows to recover the true posterior more precisely, however, without
qualitatively improving the resulting inference on \( f_{0} \), see Figure
\ref{fig_N_MaternKernel_M160_Times} and Table \ref{tab_N_MaternMSE}.

Finally, the \( \log \)-\( \log \) plot in Figure
\ref{fig_N_MaternKernel_M160_Times} illustrates that the computational
cost of the CGGP posterior with \( m = 2 n^{ 1 / ( 2 \alpha + 1 ) } \)
iterations scales like \( O( m n^{2} ) \) instead of \( O( n^{3} ) \).  

\begin{figure}[H]
  \centering
  \begin{subfigure}[t]{0.40\textwidth}
      \includegraphics[width=\textwidth]{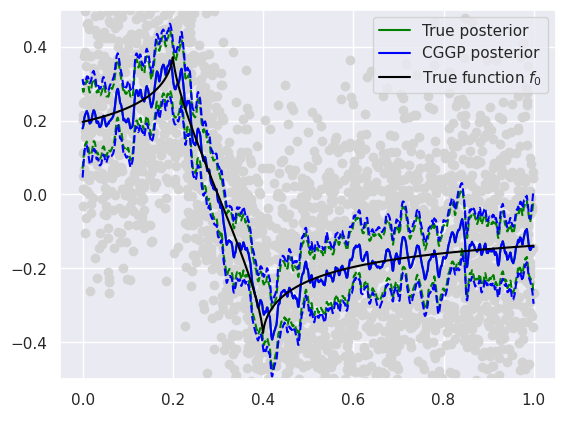}
      \caption{CGGP \( \Psi_{160}^{ \text{EV} } \),
               MSE \( = 7.5 \text{e}^{-04} \).
              }
  \end{subfigure}
  \begin{subfigure}[t]{0.40\textwidth}
      \includegraphics[width=\textwidth]{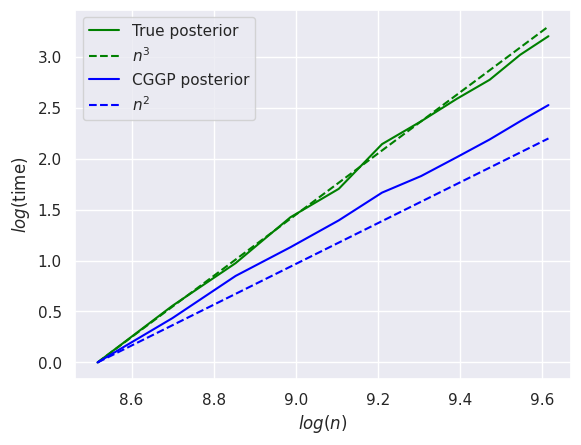}
      \caption{Log-log plot of computation times.
              }
  \end{subfigure}
  \captionsetup{width=.8\linewidth}
  \caption{Simulation results for the Matern kernel with \( n = 3000 \)
           observations from model \eqref{eq_N_MaternModel}.
           CGGP posteriors \( \Psi^{ \text{CG} }_{m} \) with \( m = 160 \).
           Log-log plot of computation times for the true posterior compared
           with the CGGP posterior for the optimal choice of \( m \).
           \( n \in \{ 5000, 6000, \dots, 15000 \} \) and intercept shifted to
           \( 0 \).
          }
  \label{fig_N_MaternKernel_M160_Times}
\end{figure}

\begin{table}[H]
  \centering
  \begin{tabular}{lcccccc}
    \\[-1.8ex]\hline
    \hline\\[-1.8ex]
    Posterior & True                          & \( \Psi_{40}^{ \text{EV} } \) & \( \Psi_{20}^{ \text{CG} } \) 
              & \( \Psi_{40}^{ \text{CG} } \) & \( \Psi_{80}^{ \text{CG} } \) & \( \Psi_{160}^{ \text{CG} } \) \\
    \hline\\[-1.8ex]
    MSE       & \( 8 \text{e}^{-04} \)        & \( 6 \text{e}^{-04} \)        & \( 2 \text{e}^{-03} \)        
              & \( 9 \text{e}^{-04} \)        & \( 8 \text{e}^{-04} \)        & \( 8 \text{e}^{-04} \) \\
    \\[-1.8ex]\hline
    \hline\\[-1.8ex]
  \end{tabular}
  \captionsetup{width=.8\linewidth}
  \caption{MSE for the posteriors displayed in Figures
          \ref{fig_N_MaternKernel_M40}, \ref{fig_N_MaternKernel_M80M20} and
          \ref{fig_N_MaternKernel_M160_Times}.
          }
  \label{tab_N_MaternMSE} 
\end{table}


\subsection{Squared exponential covariance kernel}
\label{ssec_SquaredExponentialCovarianceKernel}

We consider the regression model \eqref{eq_N_MaternModel} with 
\( \sigma = 0.2 \) and
\begin{align}
  \label{eq_N_TrueFunctionSqrExpKernel}
  f_{0}(x): = | x + 1 |^{ \beta} - | x - 3/2 |^{ \beta }, 
  \qquad x \in \mathbb{R},
\end{align}
with regularity \( \beta = 0.8 \).
As a prior, we choose a centered GP with squared exponential covariance kernel
\begin{align}
  \label{eq_N_SqrExpKernel}
  k( x, x' ): = \exp \Big( \frac{ - ( x' - x )^{2} }{ b^{2} } \Big),
  \qquad x, x' \in \mathbb{R}.
\end{align}
The corresponding eigenvalues are exponentially decaying, see Section
\ref{sec_Examples} for discussion and references.
By setting the bandwidth parameter 
\( b = b_{n} = 4 n^{ - 1 / ( 1 + 2 \beta ) } \), the corresponding true
posterior achieves the minimax contraction rate 
\( \varepsilon_{n} = n^{ - \beta / ( 1 + 2 \beta ) } \) for
any \( f_{0} \in C^{ \beta }( \mathbb{R} ) \cap L^{2}( \mathbb{R} ) \).
Since with high probability, the design points \( X_{i} \), \( i \le n \) are
contained in a compact interval and the tails of \( f_{0} \) can be adjusted to
guarantee \( f_{0} \in L^{2}( \mathbb{R} ) \), these assumptions are essentially
satisfied in our setting.

As in Section \ref{ssec_MaternCovarianceKernel}, we focus on results for the
CGGP posterior after checking the equivalence between LGP and CGGP from
Corollary \ref{cor_EquivalenceBetweenLGPandCGGP} for small \( n \) and \( m \),
see Figure \ref{fig_N_LGPEquivalentCGGPExponential}.

\begin{table}[H]
  \centering
  \begin{tabular}{lcccccc}
    \\[-1.8ex]\hline
    \hline\\[-1.8ex]
    Posterior & True                   & \( \Psi_{80}^{ \text{EV} } \) & \( \Psi_{40}^{ \text{CG} } \) & \( \Psi_{80}^{ \text{CG} } \) & \( \Psi_{160}^{ \text{CG} } \) & \( \Psi_{320}^{ \text{CG} } \) \\
    \hline\\[-1.8ex]
    MSE       & \( 6 \text{e}^{-04} \) & \( 6 \text{e}^{-04} \)        & 0.01                          & \( 6 \text{e}^{-04} \)        & \( 6 \text{e}^{-04} \)         & \( 6 \text{e}^{-04} \)         \\
    \\[-1.8ex]\hline
    \hline\\[-1.8ex]
  \end{tabular}
  \captionsetup{width=.8\linewidth}
  \caption{MSE for the posteriors displayed in Figures
           \ref{fig_N_SquaredExponentialKernel_M80},
           \ref{fig_N_SquaredExponentialKernel_M160M40} and
           \ref{fig_N_SquaredExponentialKernel_M320Times}.
          }
  \label{tab_N_ExponentialMSE} 
\end{table}

\begin{figure}[H]
  \centering
  \begin{subfigure}[t]{0.40\textwidth}
      \includegraphics[width=\textwidth]{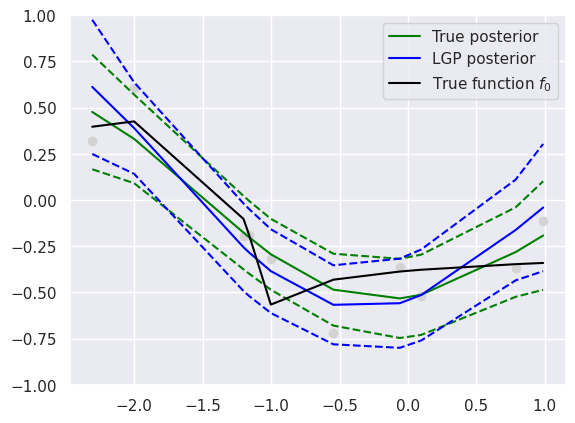}
      \caption{LGP posterior \( \Psi^{ \text{L} }_{3} \).}
  \end{subfigure}
  \begin{subfigure}[t]{0.40\textwidth}
      \includegraphics[width=\textwidth]{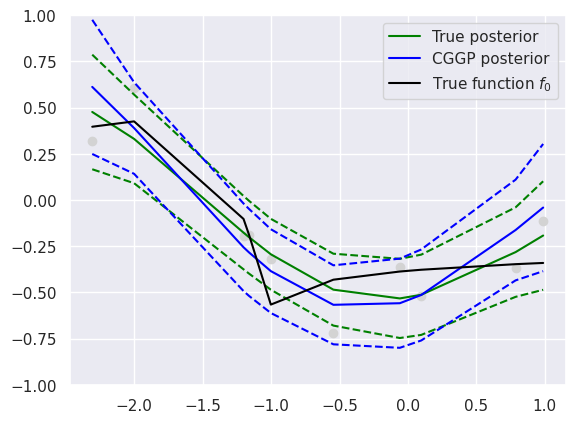}
      \caption{CGGP posterior \( \Psi^{ \text{CG} }_{3} \).}
  \end{subfigure}
  \captionsetup{width=.8\linewidth}
  \caption{Equivalence of LGP and CGGP for \( n = 10 \) observations in the regression model with $f_0$ given in
           \eqref{eq_N_TrueFunctionSqrExpKernel}. The
           LGP and CGGP posteriors \( \Psi_{m} \) are computed with \( m = 3 \) iterations.
           The Lanczos algorithm is initialized at \( w_{0} = Y / \| Y \| \).
          }
  \label{fig_N_LGPEquivalentCGGPExponential}
\end{figure}

\begin{figure}[H]
  \centering
  \begin{subfigure}[t]{0.40\textwidth}
      \includegraphics[width=\textwidth]{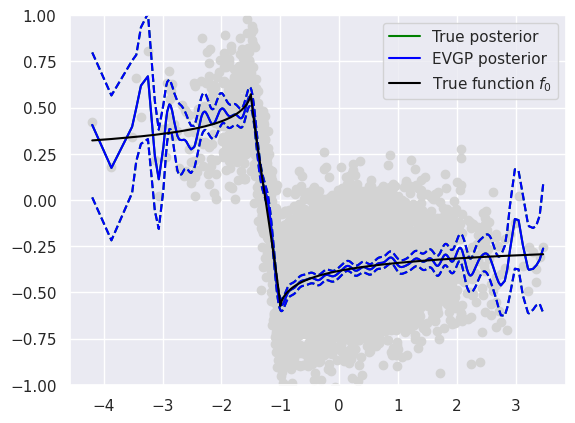}
      \caption{EVGP \( \Psi_{80}^{ \text{EV} } \)}
  \end{subfigure}
  \begin{subfigure}[t]{0.40\textwidth}
      \includegraphics[width=\textwidth]{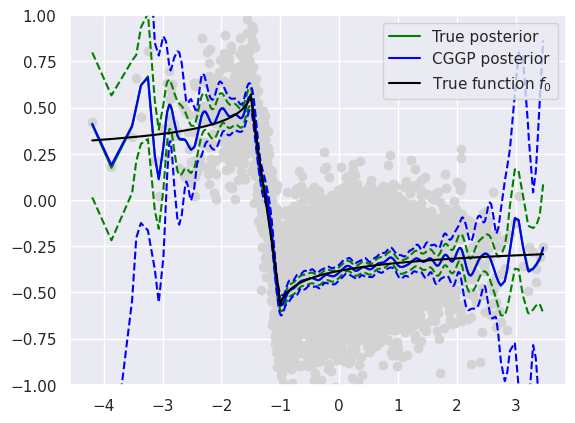}
      \caption{CGGP \( \Psi_{ 80 }^{ \text{CG} } \), }
      \label{fig_N_ExponentialCG80}
  \end{subfigure}
  \captionsetup{width=.8\linewidth}
  \caption{CGGP approximation in the regression model with $f_0$ given in
           \eqref{eq_N_TrueFunctionSqrExpKernel}, sample size
           \( n = 5000 \) and squared exponential covariance kernel.
           The approximate EVGP and CGGP posteriors \( \Psi_{m} \) are computed
           with \( m = 80 \) iterations.
          }
  \label{fig_N_SquaredExponentialKernel_M80}
\end{figure}

For the EVGP posterior \( \Psi_{m}^{ \text{EV} } \),  Corollary 9 in
\cite{NiemanEtal2022ContractionRates} together with Lemma
\ref{lem_EmpiricalEigenvectorActionsAndVariationalBayes} guarantee that the
contraction rate of the approximate posterior is (nearly) optimal 
\( n^{ - \beta / ( 1 + 2 \beta ) } \log n \) for a choice 
\begin{align}
  m \ge 2^{-3/2} n^{ 1 / ( 1 + 2 \beta ) } \log n \approx 80
  \qquad\text{for \( n = 5000 \).}
\end{align}
This is confirmed in Figure \ref{fig_N_SquaredExponentialKernel_M80},
showing that \( \Psi^{ \text{EV} }_{80} \) already recovers the true posterior
along the whole interval that contains the observations.
The CGGP posterior provides a good approximation within two standard deviations
of the standard normal design distribution, see Figure
\ref{fig_N_ExponentialCG80}.
It turns out that $m=80$ iterations are enough for the CGGP posterior mean to
achieve essentially the same MSE as the true posterior, see Table
\ref{tab_N_ExponentialMSE}.
Following the same reasoning as in Section \ref{ssec_MaternCovarianceKernel}, we
add a factor two to the iteration number.
Then, the CGGP has  essentially the same performance, including the width of the
credible bands, as the EVGP posterior.

\begin{figure}[H]
  \centering
  \begin{subfigure}[t]{0.40\textwidth}
      \includegraphics[width=\textwidth]{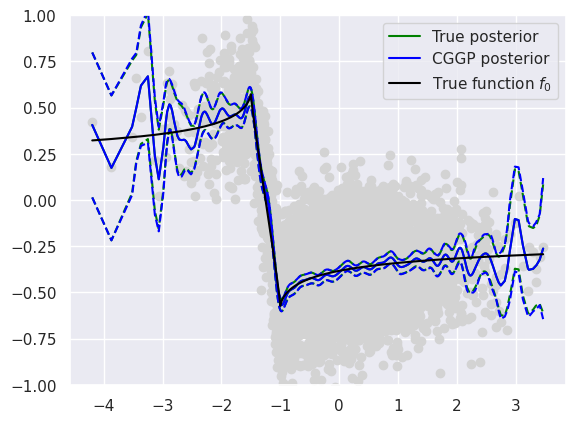}
      \caption{CGGP \( \Psi_{ 160 }^{ \text{CG} } \).}
  \end{subfigure}
  \begin{subfigure}[t]{0.40\textwidth}
      \includegraphics[width=\textwidth]{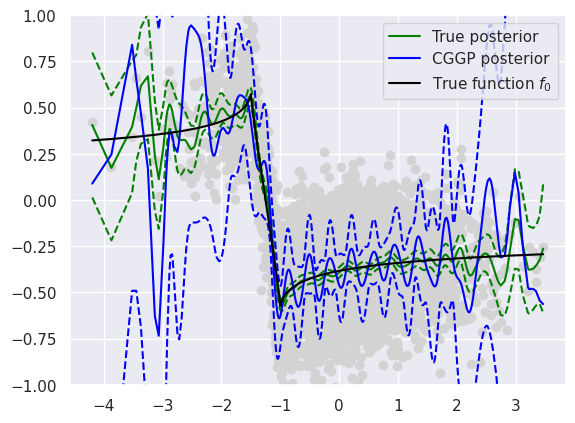}
      \caption{EVGP \( \Psi_{40}^{ \text{EV} } \)}
  \end{subfigure}
  \captionsetup{width=.8\linewidth}
  \caption{Comparing approximation accuracy for other iteration numbers in the
           setting of Figure \ref{fig_N_SquaredExponentialKernel_M80}.
           The EVGP and CGGP posteriors \( \Psi_{m} \) are computed with $m=160$
           and $m=40$ iterations, respectively.
          }
  \label{fig_N_SquaredExponentialKernel_M160M40}
\end{figure}

\begin{figure}[H]
  \centering
  \begin{subfigure}[t]{0.40\textwidth}
      \includegraphics[width=\textwidth]{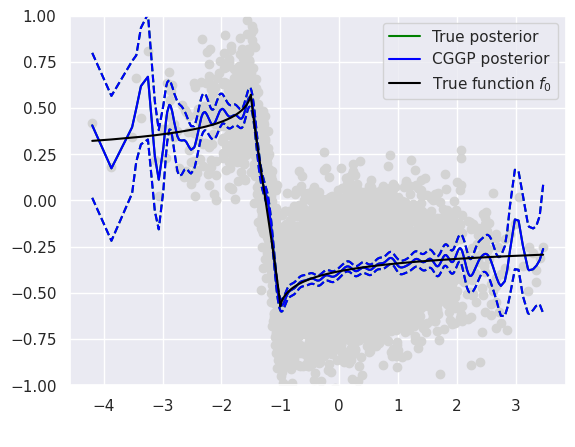}
      \caption{CGGP \( \Psi_{320}^{ \text{EV} } \)}
  \end{subfigure}
  \begin{subfigure}[t]{0.40\textwidth}
      \includegraphics[width=\textwidth]{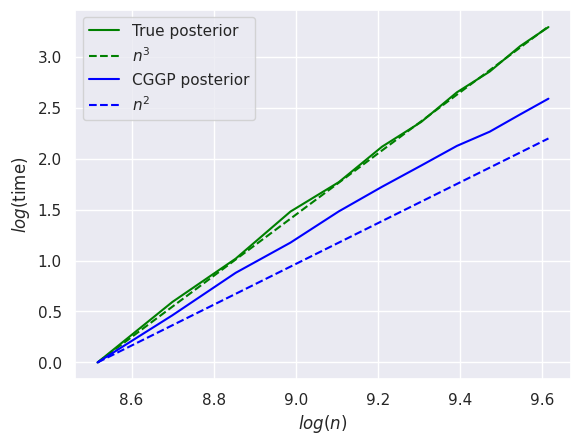}
      \caption{Log-log plot of computation times.
              }
  \end{subfigure}
  \captionsetup{width=.8\linewidth}
  \caption{CGGP posterior \( \Psi^{ \text{CG} }_{m} \) with \( m = 320 \).
           Log-log plot of computation times for the true posterior compared
           with the CGGP posterior for the optimal choice of \( m \).
           \( n \in \{ 5000, 6000, \dots, 15000 \} \) and intercept shifted to
           \( 0 \).
          }
  \label{fig_N_SquaredExponentialKernel_M320Times}
\end{figure}

Considering half the optimal amount of iterations \( m = 40 \), this
relationship is lost as the CGGP posterior cannot approximate the true posterior
even within one standard deviation of the design distribution, see Figure
\ref{fig_N_SquaredExponentialKernel_M160M40}.
This is also reflected by the substantially larger MSE of the posterior
mean.
Finally, increasing the number of iterations beyond the indicated size, here we
choose \( m = 320 \), visually recovers the true posterior exactly, however,
without improving the MSE, see Figure
\ref{fig_N_SquaredExponentialKernel_M320Times}.

In conclusion, the numerical analysis above indicates that the approximate
posterior, based on sufficiently many iterations, produces similarly reliable
inference on the true functions \( f_{0} \) while reducing the computational
complexity of the procedure substantially.
The amount of reduction is illustrated in Figure
\ref{fig_N_SquaredExponentialKernel_M320Times}, which shows that the
computation time of the true posterior scales like \( O( n^{3}) \), whereas
the approximation scales like \( O( m n^{2} ) \).



\newpage

\begin{appendices}

\section{Proofs of main results}
\label{sec_ProofsOfMainResults}

\begin{proof} 
  [\normalfont \textbf{Proof of Lemma
               \ref{lem_EmpiricalEigenvectorActionsAndVariationalBayes}}
               (Empirical eigenvector actions and variational Bayes)]
  \label{prf_EmpiricalEigenvectorActionsAndVariationalBayes}
  Since two \\ Gaussian processes are uniquely determined by their mean and
  covariance functions, it suffices to check that these are equal.
  The mean and covariance functions of the empirical eigenvector posterior from
  Algorithm \ref{alg_GPApproximation} are given by
  \begin{align}
    x & \mapsto k( X, x )^{ \top } C_{m}^{ \text{EV} } Y 
        = 
                k( X, x )^{ \top }
                \sum_{ j = 1 }^{m} 
                \frac{ \langle \widehat{u}, Y \rangle }
                     { \widehat{ \mu }_{j} + \sigma^{2} } 
                \widehat{u}_{j},
    \\
    ( x, x' ) & \mapsto k( x, x' ) - k( X, x )^{ \top } C_{m} k( X, x ) 
                = 
                        k( x, x' ) 
                        - 
                        \sum_{ j = 1 }^{m} 
                        \frac{ \langle k( X, x  ), \widehat{u}_{j} \rangle
                               \langle k( X, x' ), \widehat{u}_{j} \rangle
                             }
                             { \widehat{ \mu }_{j} + \sigma^{2} }.
    \notag
  \end{align}
  By combining Equations \eqref{eq_APC_FConditionalOnU},
                         \eqref{eq_APC_VariationalClass} and
                         \eqref{eq_APC_ExplicitSolution},
  the variational Bayes posterior mean and covariance functions, in the notation
  from there, are given by
  \begin{align}
    x & \mapsto \sigma^{2} 
                K_{xu} ( \sigma^{ - 2 } K_{uf} K_{fu} + K_{uu} )^{-1} K_{uf} Y
    \\
    ( x, x' ) & \mapsto   k( x, x' ) 
                        - 
                          K_{xu} K_{uu}^{-1} K_{ux} 
                        + 
                          K_{xu} (\sigma^{-2} K_{fu} K_{uf} + K_{uu })^{-1}
                          K_{ux}. 
    \notag
  \end{align}
  where from Equation \eqref{eq_APC_EmpiricalEigenvectorVB}, it follows that \( 
      K_{uu} = \text{Cov}(U) 
    =
      \diag( \widehat{ \mu }_{1}, \dots, \widehat{ \mu }_{m} ) 
  \),
  \begin{align}
        K_{ u f } 
    & = 
        ( \text{Cov}( U_{j}, F( X_{i} ) ) )_{ j \le m, i \le n } 
    = 
        \begin{pmatrix}
          \widehat{ \mu }_{1} \widehat{u}_{1}^{ \top } \\
          \vdots \\
          \widehat{ \mu }_{m} \widehat{u}_{m}^{ \top } \\
        \end{pmatrix}
    \in 
        \mathbb{R}^{ m \times n }
    \\
    \qquad \text{ and } \qquad 
        K_{ x u }
    & = 
        ( \text{Cov}( F(x), U_{j} ) )_{ j \le m }
    = 
        k( X, x )^{ \top } 
        \begin{pmatrix}
          \widehat{u}_{1}
          \dots
          \widehat{u}_{m}
        \end{pmatrix}
    \in
        \mathbb{R}^{ 1 \times m }.
    \notag
  \end{align}
  Consequently, the mean function is 
  \begin{align}
    x 
    & \mapsto 
    k( X, x )^{ \top } 
    \begin{pmatrix}
      \vdots              &       & \vdots              \\
      \widehat{u}_{1} & \dots & \widehat{u}_{m} \\
      \vdots              &       & \vdots  
    \end{pmatrix}
    \begin{pmatrix}
      ( \widehat{ \mu }_{1}^{2} + \sigma^{2} \widehat{ \mu }_{1} )^{-1} &        &  \\
                                                                            & \ddots &  \\
                                                                            &        & ( \widehat{ \mu }_{m}^{2} + \sigma^{2} \widehat{ \mu }_{m} )^{-1} 
    \end{pmatrix}
    \begin{pmatrix}
      \widehat{ \mu }_{1} \langle Y, \widehat{u}_{1} \rangle \\
      \vdots \\
      \widehat{ \mu }_{m} \langle Y, \widehat{u}_{m} \rangle \\
    \end{pmatrix}
    \notag
    \\
    & = 
    k( X, x )^{ \top }
    \sum_{ j = 1 }^{m} 
    \frac{ \langle \widehat{u}, Y \rangle }
         { \widehat{ \mu }_{j} + \sigma^{2} } 
    \widehat{u}_{j}.
  \end{align}
  An analogous calculation shows that also the coavariance functions coincide.
\end{proof}

\begin{proof} 
  [\normalfont \textbf{Proof of Proposition \ref{prp_KullbackLeiblerBound}}
   (Kullback-Leibler bound)]
  \label{prf_KullbackLeiblerBound}
  \
  We pick up the reasoning at the decomposition of the Kullback-Leibler
  divergence from Equation \eqref{eq_T_PreliminaryKLBound}.
  For notational convenience, we omit the dependency of \( m_{n} \) and \(
  \tilde m_{n} = m_{n} \log n \) on \( n \).

  We note that \( C_{ \tilde m } - C_{m} \ge 0 \) implies
  \begin{align}
    \tr ( K - K K_{ \sigma }^{-1} K )^{-1} K ( C_{ \tilde m } - C_{m} ) K  
    & \ge 0
  \end{align}
  and
  \begin{align}
    & \ \ \ \
            \| ( K_{ \sigma }^{-1} - C_{ \tilde m } ) Y
            \|_{ K ( K - K K_{ \sigma }^{-1} K )^{-1} K }
    = 
            \| \Gamma_{ \tilde m } K_{ \sigma } K_{ \sigma }^{-1} Y
            \|_{ K ( K - K K_{ \sigma }^{-1} K )^{-1} K }
    \\
    & =  
            \| \Gamma_{ \tilde m } K_{ \sigma } \Gamma_{m} K_{ \sigma } 
               K_{ \sigma }^{-1} Y
            \|_{ K ( K - K K_{ \sigma }^{-1} K )^{-1} K }
    \le 
            \| \Gamma_{m} K_{ \sigma } K_{ \sigma }^{-1} Y
            \|_{ K ( K - K K_{ \sigma }^{-1} K )^{-1} K },
    \notag
    \\
    & = 
            \| ( K_{ \sigma }^{-1} - C_{m} ) Y
            \|_{ K ( K - K K_{ \sigma }^{-1} K )^{-1} K },
    \notag
  \end{align}
  where we have used the fact that \( \Gamma_{m} K_{ \sigma } \) is the 
  \( K_{ \sigma } \)-orthogonal projection onto 
  \( \vspan \{ s_{j}: j \le m \}^{ \perp } \).
  Using this, we may assume without loss of generality that 
  \( 
        \kl( \Psi_{ \tilde m }, \Pi_{n}( \cdot | X, Y ) ) 
    \le
        \kl( \Psi_{m}, \Pi_{n}( \cdot | X, Y ) )
  \),
  i.e., we can consider the approximate posterior based on 
  \( ( \tilde u_{i})_{ i \le m } \) but from the Lanczos algorithm computed
  based on \( \tilde m \).

  \textbf{Step 1: Empirical eigenvector part.}
  Treating the first part of the decomposition from Equation
  \eqref{eq_T_KLDDecomposition}
  \begin{align}
    \label{eq_A_FirstPartOfTheKLDecomposition}
    (\dag): = \tr ( K - K K_{ \sigma }^{-1} K )^{-1}
                  K ( K_{ \sigma }^{-1} - C_{m}^{ \text{EV} } ) K
            + 
              2 \| ( K_{ \sigma }^{-1} - C_{m}^{ \text{EV} } ) Y 
              \|_{ K ( K - K K_{ \sigma }^{-1} K )^{-1} K }^{2},
  \end{align}
  which only depends on the idealized empirical eigenvalues 
  \( ( \widehat{u}_{i} )_{ i \le m } \), is straightforward.
  The trace term satisfies
  \begin{align}
    & \ \ \ \
        \tr ( K - K K_{ \sigma }^{-1} K )^{-1}
            K ( K_{ \sigma }^{-1} - C_{m}^{ \text{EV} } ) K
    \\
    & = 
        \tr ( K - K K_{ \sigma }^{-1} K )^{-1}
            \Big( \sum_{ j = m + 1 }^{n} 
                  \frac{ \widehat{ \mu }_{j}^{2}          }
                       { \widehat{ \mu }_{j} + \sigma^{2} } 
                  \widehat{u}_{j} \widehat{u}_{j}^{ \top }
            \Big)
    = 
        \frac{n}{ \sigma^{2} } 
        \sum_{ j = m + 1 }^{n} \widehat{ \lambda }_{j},
    \notag 
  \end{align}
  since the matrix \( ( K - K K_{ \sigma }^{-1} K ) \) can also be decomposed in terms of the projectors \( ( \widehat{u}_{j} \widehat{u}_{j}^{ \top } )_{ j \le n } \) with eigenvalues 
  \begin{align}
      \Big(  \widehat{ \mu }_{j} 
           - 
             \frac{ \widehat{ \mu }_{j}^{2}          }
                  { \widehat{ \mu }_{j} + \sigma^{2} } 
      \Big)
    = 
      \frac{ \widehat{ \mu }_{j} \sigma^{2}   }
           { \widehat{ \mu }_{j} + \sigma^{2} } 
    \qquad \text{ for all } j \le n.
  \end{align}

  For the term involving \( Y = \mathbf{f}_{0} + \varepsilon \), we set 
  \( \Gamma_{m}^{ \text{EV} }: = K_{ \sigma }^{-1} - C_{m}^{ \text{EV} } \) and
  estimate
  \begin{align}
    & \ \ \ \
              \| ( K_{ \sigma }^{-1} - C_{m}^{ \text{EV} } ) Y 
              \|_{ K ( K - K K_{ \sigma }^{-1} K ) K }^{2} 
    = 
              Y^{ \top } \Gamma^{ \text{EV} }_{m} 
              K ( K - K K_{ \sigma }^{-1} K )^{-1} K
              \Gamma_{m}^{ \text{EV} } Y
    \notag
    \\
    & \le 
              2 \mathbf{f}_{0}^{ \top }
              \Gamma_{m}^{ \text{EV} } K ( K - K K_{ \sigma }^{-1} K )^{-1} K 
              \Gamma_{m}^{ \text{EV} }
              \mathbf{f}_{0} 
            + 
              2 \varepsilon^{ \top }
              \Gamma_{m}^{ \text{EV} } K ( K - K K_{ \sigma }^{-1} K )^{-1} K
              \Gamma_{m}^{ \text{EV} }
              \varepsilon.
    \notag
  \end{align}
  We further split 
  \( \mathbf{f}_{0}: = \mathbf{h} + \mathbf{f}_{0} - \mathbf{h} \), where 
  \( h \in \mathbb{H} \) is an element of the RKHS induced by the prior process
  and estimate both individually again.
  Then,
  \begin{align}
    & \ \ \ \
            \mathbf{h}^{ \top } 
            \Gamma_{m}^{ \text{EV} } K ( K - K K_{ \sigma }^{-1} K )^{-1} K
            \Gamma_{m}^{ \text{EV} }
            \mathbf{h}
    \\
    & = 
            \mathbf{h}^{ \top } K_{ \sigma }^{-1}
            ( I - K_{ \sigma } C_{m}^{ \text{EV} } ) 
            K ( K - K K_{ \sigma }^{-1} K )^{-1} K \Gamma_{m}^{ \text{EV} }
            \mathbf{h}
    \notag
    \\
    & \le 
            \mathbf{h}^{ \top } K_{ \sigma }^{-1} \mathbf{h}
            \|
              ( I - K_{ \sigma } C_{m}^{ \text{EV} } ) 
              K ( K - K K_{ \sigma }^{-1} K )^{-1} K \Gamma_{m}^{ \text{EV} }
            \|_{ \text{op} } 
    \\
    & \le 
            \| h \|_{ \mathbb{H} }^{2}
            \| K ( K - K K_{ \sigma }^{-1} K )^{-1} K \Gamma_{m}^{ \text{EV} } 
            \|_{ \text{op} } 
      \le 
            \frac{n}{ \sigma^{2} }
            \| h \|_{ \mathbb{H} }^{2} \widehat{ \lambda }_{ m + 1 },
    \notag
  \end{align}
  where we have used the fact that 
  \( 
        h^{ \top } K_{ \sigma }^{-1} h 
    \le 
        \| P_{n, k} h \|_{ \mathbb{H} } 
  \)
  with \( P_{n, k} \) the projection onto 
  \( \vspan \{ k( X_{i}, \cdot): i = 1, \dots, n \} \) and the singular value
  decomposition of 
  \( K ( K - K K_{ \sigma }^{-1} K )^{-1} K \Gamma_{m}^{ \text{EV} } \).

  An analogous estimate shows
  \begin{align}
    & \ \ \ \
            ( \mathbf{f}_{0} - \mathbf{h} )^{ \top } 
            \Gamma_{m}^{ \text{EV} } K ( K - K K_{ \sigma }^{-1} K )^{-1} K
            \Gamma_{m}^{ \text{EV} }
            ( \mathbf{f}_{0} - \mathbf{h} )
    \\
    & \le 
            n \| f_{0} - h \|_{n}^{2}  
            \| \Gamma_{m}^{ \text{EV} } K ( K - K K_{ \sigma }^{-1} K )^{-1} K
               \Gamma_{m}^{ \text{EV} } 
            \|_{ \text{op} } 
    \le 
            \frac{n}{ \sigma^{2} } \| f_{0} - h \|_{n}^{2}.
    \notag
  \end{align}
  Plugging all estimates into Equation 
  \eqref{eq_A_FirstPartOfTheKLDecomposition} yields that
  \begin{align}
            ( \dag ) 
    & \le
            \frac{ C n }{ \sigma^{2} } 
            \Big( \sum_{ j = m + 1 }^{n} \widehat{ \lambda }_{j}
                + \| h \|_{ \mathbb{H} }^{2} \widehat{ \lambda }_{ m + 1 } 
                + \| f_{0} - h \|_{n}^{2}
            \Big)
    \\
    & + 
            2 \varepsilon^{ \top }
            \Gamma^{ \text{EV} }_{m} K ( K - K K_{ \sigma }^{-1} K )^{-1} K
            \Gamma_{m}^{ \text{EV} }
            \varepsilon
    \notag
  \end{align}
  for any \( h \in \mathbb{H} \). 

  \textbf{Step 2: Lanczos eigenvector part.}
  We treat the second part of the decomposition in Equation
  \eqref{eq_T_PreliminaryKLBound}
  \begin{align}
        ( \dag\dag ):
    = 
        \tr ( K - K K_{ \sigma }^{-1} K )^{-1} 
            K ( C^{ \text{EV} }_{m} - C_{m} ) K
      + 
        2 \| ( C_{m} - C^{ \text{EV} }_{m} ) Y 
          \|_{ K ( K - K K_{ \sigma }^{-1} K ) K }^{2},
  \end{align}
  with \( C_{m} = C_{m}^{ \text{L} } \) and set 
  \( D_{m}: = C_{m}^{ \text{EV} } - C_{m}^{ \text{L} } \).
  The trace term satisfies
  \begin{align}
    & \ \ \ \
          \tr ( K - K K_{ \sigma }^{-1} K )^{-1} 
              K ( C^{ \text{EV} }_{m} - C_{m}^{ \text{L} } ) K
    = 
          \tr K ( K - K K_{ \sigma }^{-1} K )^{-1} K D_{m}
    \\
    & \le 
          \| K ( K - K K_{ \sigma }^{-1} K )^{-1} K \|_{ HS } 
          \| D_{m} \|_{ \text{HS} } 
    \le 
          \frac{ n ( \| K \|_{ \text{op} }^{2} + \sigma^{2} \| K \|_{ \text{op} } ) }
               { \sigma^{2} }
          \| D_{m} \|_{ \text{HS} },
    \notag
  \end{align}
  where we have used the singular value decomposition of \( K \) again.

  For the other term, we again split 
  \( Y = \mathbf{f}_{0} - \mathbf{h} + \mathbf{h} + \mathbf{ \varepsilon } \)
  for an arbitrary element \( h \in \mathbb{H} \) and estimate 
  \begin{align}
    & \ \ \ \
            \| ( C_{m}^{ \text{L} } - C_{m}^{ \text{EV} } ) Y 
            \|_{ K ( K - K K_{ \sigma }^{-1} K ) K }^{2} 
    \\
    & \le 
            C \| K ( K - K K_{ \sigma }^{-1} K ) K \|_{ \text{op} }
            ( \| \varepsilon \|^{2} + n \| f_{0} - h \|_{n}^{2} + 
              \| \mathbf{h} \|^{2}
            ) 
            \| D_{m} \|_{ \text{HS} }^{2} 
    \\
    & \le 
            C \frac{ \| K \|_{ \text{op} }^{2} + \sigma^{2} \| K \|_{ \text{op} } }
                   { \sigma^{2} } 
            ( \| \varepsilon \|^{2} + n \| f_{0} - h \|_{n}^{2} 
            + ( \| K \|_{ \text{op} } + \sigma^{2} ) \| h \|_{ \mathbb{H} }) 
            \| D_{m} \|_{ \text{HS} }^{2},
    \notag
  \end{align}
  where we have used again that 
  \( 
    \| \mathbf{h} \|^{2} \le ( \| K \|_{ \text{op} } + \sigma^{2} ) 
                             \| K_{ \sigma }^{ - 1 / 2 } \mathbf{h} \|^{2} 
                         \le ( \| K \|_{ \text{op} } + \sigma^{2} )
                             \| h \|_{ \mathbb{H} }^{2} 
  \). 
  Together, this yields
  \begin{align}
    ( \dag \dag ) 
    & \le 
    C \Big( \frac{ n ( \| K \|_{ \text{op} }^{2} + \sigma^{2} \| K \|_{ \text{op} } ) }
                 { \sigma^{2} }
            \| D_{m} \|_{ \text{HS} }
          \\
    & \qquad + 
            \frac{ \| K \|_{ \text{op} }^{2} + \sigma^{2} \| K \|_{ \text{op} } }
                 { \sigma^{2} } 
            ( \| \varepsilon \|^{2} + n \| f_{0} - h \|_{n}^{2} 
            + ( \| K \|_{ \text{op} } + \sigma^{2} ) \| h \|_{ \mathbb{H} }^{2} 
            ) 
            \| D_{m} \|_{ \text{HS} }^{2}
      \Big)
    \notag 
  \end{align}
  for any \( h \in \mathbb{H} \).

  \textbf{Step 3: Probabilistic part.} 
  Via Markov's inequality, for a fixed sequence \( M_{n} \to \infty \), we can
  restrict to an event with probability converging to one, on which
  \begin{align}
    \sum_{ j = m + 1 }^{n} 
    \widehat{ \lambda }_{j} \le M_{n} \sum_{ j = m + 1 }^{n} \lambda_{j},
    \quad 
    \widehat{ \lambda }_{ m + 1 } \le M_{n} 
                                      \mathbb{E} \widehat{ \lambda }_{ m + 1 }, 
    \quad 
    \| f_{0} - h \|_{n}^{2} \le M_{n} \| f_{0} - h \|_{ L^{2} }^{2},
  \end{align}
  where we have also used Proposition \ref{prp_ExpectationOfTheTrace} for the
  first inequality.
  By the same reasoning, we can also assume that 
  \begin{align}
    & \ \ \ \
          \varepsilon^{ \top }
          \Gamma_{m}^{ \text{EV} } K ( K - K K_{ \sigma }^{-1} K )^{-1} K
          \Gamma_{m}^{ \text{EV} }
          \varepsilon
    \le 
          M_{n} \mathbb{E}_{ \varepsilon }
          \tr 
          ( \varepsilon \varepsilon^{ \top } )
          \Gamma_{m}^{ \text{EV} } K ( K - K K_{ \sigma }^{-1} K )^{-1} K
          \Gamma_{m}^{ \text{EV} } 
    \notag
    \\
    & = 
          M_{n} \sigma^{2}
          \tr \Gamma_{m}^{ \text{EV} } K ( K - K K_{ \sigma }^{-1} K )^{-1} K
              \Gamma_{m}^{ \text{EV} }
    \le 
          \frac{ M_{n} n }{ \sigma^{2} } 
          \sum_{ j = m + 1 }^{n} \widehat{ \lambda }_{j}
  \end{align}
  and 
  \( 
    \| \varepsilon \|^{2} \le M_{n} \mathbb{E} \| \varepsilon \|^{2} 
                            = M_{n} \sigma^{2} n
  \). 

  By further restricting to the event from Proposition
  \ref{prp_RelativePerturbatonBounds} with \( m_{0} = C \tilde m = C m \log n \)
  and using Proposition \ref{prp_ProbabilisticBoundsForLanczosEigenpairs}, we
  finally have
  \begin{align}
            \| D_{m} \|_{ \text{HS} } 
    & \le 
            \sum_{ j = 1 }^{m} 
            \Big\|
              \frac{1}{ \widehat{ \mu }_{j} + \sigma^{2} } 
              \widehat{u}_{j} \widehat{u}_{j}^{ \top } 
              - 
              \frac{1}{ \tilde \mu_{j} + \sigma^{2} } 
              \tilde u_{j} \tilde u_{j}^{ \top } 
            \Big\|_{ \text{HS} }
    \\
    & \le
            \sum_{ j = 1 }^{m} 
            \Big| 
              \frac{1}{ \widehat{ \mu }_{j} + \sigma^{2} } 
            - 
              \frac{1}{ \tilde \mu_{j} + \sigma^{2} } 
            \Big|
            + 
            \frac{ \| \widehat{u}_{j} \widehat{u}_{j}^{ \top } 
                    - \tilde u_{j} \tilde u_{j} ^{ \top } 
                   \|_{ \text{HS} }
                 }
                 { \sigma^{2} }
    \notag
    \\
    & \le 
            \Big( 
              \frac{ m n }{ \sigma^{4} }
            + 
              \frac{m}{ \sigma^{2} } 
            \Big)
            ( 1 + c )^{-m}
    \notag
  \end{align}
  and 
  \( \| K \|_{ \text{op} } =   \widehat{ \mu }_{1} 
                           =   n \widehat{ \lambda }_{1} 
                           \le C n \lambda_{1} 
  \) as in Step 1 of the proof of Proposition
  \ref{prp_ProbabilisticBoundsForLanczosEigenpairs}.

  Now, Assumption \hyperref[ass_ConcentrationFunctionInequality]
                  {\normalfont \textbf{{\color{blue} (CFUN)}}} 
  implies that we can choose \( h \in \mathbb{H} \) such that 
  \begin{align}
    \| h \|_{ \mathbb{H} }^{2} \le n \varepsilon_{n}^{2} 
    \qquad \text{ and } \qquad 
    \| f_{0} - h \|_{ L^{2} }^{2} \le \varepsilon_{n}^{2}.
  \end{align}
  In combining all of our estimates, we note that since \( C' \log n \le m \)
  for \( C' \) sufficiently large, all terms involving 
  \( \| D_{m} \|_{ \text{HS} } \) are of lower order.
  The remaining terms yield the estimate
  \begin{align}
            \kl( \Psi_{ \tilde m }, \Pi_{n}( \cdot | X, Y ) ) 
    & \le 
            ( \dag ) + ( \dag \dag ) 
    \le 
            \frac{ C M_{n}^{2} n }{ \sigma^{2} }
            \Big( \varepsilon_{n}^{2}
                + \sum_{ j = m + 1 }^{ \infty } \lambda_{j} 
                + n \varepsilon_{n}^{2} \mathbb{E} \widehat{ \lambda }_{ m + 1 }
            \Big). 
  \end{align}
  The result now follows by adjusting the sequence \( M_{n} \) accordingly.
  This finishes the proof.
\end{proof}

\begin{proof} 
  [\normalfont \textbf{Proof of Corollary
   \ref{cor_EquivalenceBetweenLGPandCGGP}} (Equivalence between LGP and CGGP)]
  \label{prf_EquivalenceBetweenLGPandCGGP}
  The approximate posterior from Algorithm \ref{alg_GPApproximation} only
  depends on the approximation \( C_{m} \) of \( K_{ \sigma }^{-1} \).
  Additionally, since \( C_{m} K_{ \sigma }^{-1} \) is the \( K_{ \sigma }^{-1}
  \)-orthogonal projection onto \( \vspan \{ s_{j}: j \le m \} \), see Lemma S1
  in \citet{WengerEtal2022ComputationalUncertainty}, \( C_{m} \) only depends on
  \( \vspan \{ s_{j}: j \le m \} \).
  Consequently, if for two versions of the algorithm \( \vspan \{ s_{j}: j \le m \} \) coincide, then the resulting approximate posteriors are identical.

  Now, for the Lanczos algorithm initiated at \( Y / \| Y \| \) iterated for \( m \) steps
  \begin{align}
      \vspan \{ s_{j}: j \le m \} 
    = 
      \vspan \{ \tilde u_{j}: j \le m \} 
    =
      \vspan \{ Y, K Y, \dots K^{ m - 1 } Y \}
  \end{align}
  and since the conjugate gradients of the CG-method span the Krylov space
  associated with \( K_{ \sigma } \) and \( Y \), we also have 
  \begin{align}
      \vspan \{ s_{j}: j \le m \} 
    = 
      \vspan \{ Y, K_{ \sigma } Y, \dots, K_{ \sigma }^{ m - 1 } Y \} 
    = 
      \vspan \{ Y, K Y, \dots K^{ m - 1 } Y \}.
    \notag
  \end{align}
  This proves that the resulting approximate posteriors are identical.
  The last statement of Corollary
  \ref{cor_EquivalenceBetweenLGPandCGGP} then follows immediately.
\end{proof}


\section{Auxiliary results}
\label{sec_AuxiliaryResults}

\begin{lemma}[Multivariate conditional Gaussians]
  \label{lem_MultivariateConditionalGaussians}
  For a random variable \( ( X, Y )^{ \top } \in \mathbb{R}^{ m + n } \) with joint distribution
  \begin{align}
    \begin{pmatrix}
      X \\
      Y
    \end{pmatrix}
    \sim 
    N \left( \begin{pmatrix}
              \mu \\ \nu 
            \end{pmatrix},
            \begin{pmatrix}
              R          & C \\ 
              C^{ \top } & S
            \end{pmatrix}
    \right)
  \end{align}
  where \( R \) is nonsingular, the conditional distribution \( Y | X \) is given by \( N( \nu + C^{ \top } ( X - \mu ), S - C^{ \top } R^{-1} C ) \).
\end{lemma}

\begin{proof}[Proof]
  See Proposition 3.13 in \cite{Eaten2007MultivariateStatistics}.
\end{proof}

\begin{proof}[\normalfont \textbf{Proof of Lemma \ref{lem_EVGP}} (EVGP)]
  \label{prf_EVGP}
  We establish the claim via induction over \( m \in \mathbb{N} \).
  For \( m = 1 \), the definition of Algorithm \ref{alg_GPApproximation} states
  that
  \begin{align}
        d_{1} 
    =
        ( I - 0 \cdot K_{ \sigma } ) s_{1} = \widehat{u}_{1}, 
    \qquad \text{ and } \qquad 
        \eta_{1} 
    = 
        \widehat{u}_{1}^{ \top } K_{ \sigma } \widehat{u}_{1} 
    = 
        \widehat{ \mu }_{1} + \sigma^{2}, 
  \end{align}
  which yields 
  \( 
    C^{ \text{EV} }_{1} = ( \widehat{ \mu }_{1} + \sigma^{2}  )^{-1}
                          \widehat{u}_{1} \widehat{u}_{1}^{ \top } 
  \).
  Suppose now that we have shown the claim for \( m \). 
  Then, we obtain
  \begin{align}
    d_{ m + 1 } 
    = 
    ( I - C_{m}^{ \text{EV} } K_{ \sigma } ) \widehat{u}_{ m + 1 } 
    = 
    \widehat{u}_{ m + 1 } 
    \quad \text{ and } \quad 
    \eta_{ m + 1 } 
    = 
    \widehat{u}_{ m + 1 }^{ \top } K_{ \sigma } \widehat{u}_{ m + 1 } 
    = 
    \widehat{ \mu }_{ m + 1 } + \sigma^{2},
  \end{align}
  where we have used that
  \(
    C^{ \text{EV} }_{m} = \sum_{ j = 1 }^{m} 
                            ( \widehat{ \mu }_{j} + \sigma^{2}  )^{-j}
                            \widehat{u}_{j} \widehat{u}_{j}^{ \top } 
  \).
  This finishes the proof.
\end{proof}

\begin{proof}[\normalfont \textbf{Proof of Lemma \ref{lem_LGP}} (LGP)]
  \label{prf_LGP}
  We establish the result via induction over \( m \in \mathbb{N} \). 
  For \( m = 1 \), the definition of Algorithm \ref{alg_GPApproximation} and
  Lemma \ref{lem_ElementaryPropertiesOfLanczosEigenquantities} (ii) yield that 
  \begin{align}
    d_{1} = s_{1} = \tilde u_{1}
    \qquad \text{ and } \qquad 
    \eta_{1} = \tilde u_{1}^{ \top } K_{ \sigma } \tilde u_{1} 
             = ( \tilde \mu_{1} + \sigma^{2} ) 
               \tilde u_{1}^{ \top } ( \tilde u_{1} + u^{ \perp } ) 
             = \tilde \mu_{1} + \sigma^{2},
  \end{align}
  for some \( u^{ \perp } \perp \mathcal{K}_{ \tilde m } \).

  Suppose, we have established the claim for \( m \in \mathbb{N} \).
  Then, we obtain
  \begin{align}
            d_{ m + 1 } 
    & = 
            ( I - C_{m}^{ \text{L} } K_{ \sigma } ) \tilde u_{ m + 1 } 
    = 
            \tilde u_{ m + 1 } 
            -
            C_{m}^{ \text{L} }
            ( ( \tilde \mu_{ m + 1 } + \sigma^{2} ) \tilde u_{ m + 1 } 
            +
              u_{1}^{ \perp } 
            ) 
    = 
            \tilde u_{ m + 1 },
    \\
            \eta_{ m + 1 } 
    & = 
            \tilde u_{ m + 1 }^{ \top } K_{ \sigma } \tilde u_{ m + 1 } 
    = 
            \tilde u_{ m + 1 }^{ \top } 
            ( ( \widehat{ \mu }_{ m + 1 } + \sigma^{2} ) \tilde u_{ m + 1 } 
            + 
              u_{2}^{ \perp } 
            ) 
    = 
            \widehat{ \mu }_{ m + 1 } + \sigma^{2},
    \notag
  \end{align}
  where again, \( u_{1}^{ \perp }, u_{2}^{ \perp } \) are elements in 
  \( \mathcal{K}_{m}^{ \perp } \) and we have used that
  \(
    C^{ \text{L} }_{m} = \sum_{ j = 1 }^{m} 
                           ( \tilde \mu_{j} + \sigma^{2}  )^{-j}
                           \tilde u_{j} \tilde u_{j}^{ \top } 
  \).
\end{proof}

\begin{proof}[\normalfont \textbf{Proof of Lemma \ref{lem_CGGP}} (CGGP)]
  \label{prf_CGGP}
  We establish the result via induction over \( m \in \mathbb{N} \). 
  For \( m = 1 \), the definition of Algorithm \ref{alg_GPApproximation}
  yields that
  \begin{align}
    d_{1} = s_{1} = d_{1}^{ \text{CG} } 
    \qquad \text{ and } \qquad
    \eta_{1} = ( d_{1}^{ \text{CG} } )^{ \top } K_{ \sigma } d_{1}^{ \text{CG} }
  \end{align}
  Suppose, we have established the claim for \( m \in \mathbb{N} \).
  Then, we obtain
  \begin{align}
    d_{ m + 1 } = ( I - C_{m}^{ \text{CG} } K_{ \sigma } ) 
                                           d_{ m + 1 }^{ \text{CG} }  
                = d_{ m + 1 }^{ \text{CG} }
    \qquad \text{ and } \qquad 
    \eta_{ m + 1 } = ( d_{ m + 1 }^{ \text{CG} } )^{ \top } K_{ \sigma } 
                       d_{ m + 1 }^{ \text{CG} },
  \end{align}
  where we have used that 
  \( 
       C_{m}^{ \text{CG} } 
    =
       \sum_{ j = 1 }^{m}
       ( ( d_{j}^{ \text{CG} } )^{ \top } K_{ \sigma } 
           d_{j}^{ \text{CG} } 
       )^{-1} 
         d_{j}^{  \text{CG} } 
       ( d_{j}^{  \text{CG} } )^{ \top }
  \)
  and the conjugacy condition from Equation \eqref{eq_APPN_ConjugacyCondition}.
\end{proof}

\begin{lemma}[KL between posterior processes]
  \label{lem_KLBetweenPosteriorProcesses}
  Conditional on the design \( X \) and the observations \( Y \), the
  Kullback-Leibler divergence \( \kl( \Psi_{m} , \Pi_{n}( \cdot | X, Y ) ) \)
  between the approximate posterior GP from Algorithm \ref{alg_GPApproximation}
  and the true posterior GP is equal to the Kullback-Leibler divergence 
  \(
    \kl( \Psi_{m} \{ \mathbf{f} \in \cdot \},
         \Pi_{n}\{ \mathbf{f} \in \cdot | X, Y  \} ) 
       )
  \)
  between the finite dimensional Gaussians given by  
  \( \mathbf{f} = ( f( X_{1} ), \dots f( X_{n} ) ) \), i.e., the function 
  \( f \) evaluated at the design points.
\end{lemma}

\begin{proof}[Proof]
  Initially, we introduce some helpful notation.
  Let \( ( \Omega, \mathscr{A}, \mathbb{P} ) \) denote the probability space on
  which the Gaussian regression model from Equation
  \eqref{eq_I_GPRegressionModel} is defined conditional on the design \( X \).
  In the following, superscripts will denote push forward measures, see the
  remarks on notation in Section \ref{sec_Introduction}.
  Let \( W^{*} = K_{ \sigma }^{-1} Y \) denote the true representer weights and
  without loss of generality assume that on \( ( \Omega, \mathscr{A}, \mathbb{P}
  ) \) there exists a random variable \( W_{m} \sim N( C_{m} Y, \Gamma_{m} ) \)
  representing our believes at iteration \( m \) of the Bayesian updating scheme
  in Equations \eqref{eq_APPN_BayesianUpdateMean} and
               \eqref{eq_APPN_BayesianUpdateCov}.
  The approximate posterior \( \Psi_{m} \) is the push forward of \( F \) under
  the measure
  \begin{align}
    \mathbb{Q}(A): = \int 
                     \mathbb{P}(A | W^{*} = w)
                     \, \mathbb{P}^{ W_{m} }(d w), 
                     \qquad A \in \mathscr{A},
  \end{align}
  i.e., \( \Psi_{m} = \mathbb{Q}^{F} \).

  Assuming all quantities exist, we may now write the Kullback-Leibler
  divergence in terms of Radon-Nikodym derivatives, i.e.,
  \begin{align}
    \label{eq_B_KLOfGPs}
            \kl( \Psi_{m}, \Pi_{n}( \cdot | X, Y ) ) 
    & = 
            \int \frac{ d \Psi_{m} }{ d \Pi } \, d \Psi_{m} 
          + 
            \int \frac{ d \Pi }{ d \Pi_{n}( \cdot | X, Y ) } \, d \Psi_{m}
    \\
    & = 
            \int \frac{ d \mathbb{Q}^{F} }{ d \mathbb{P}^{F} } 
            \, d \mathbb{Q}^{F} 
          + 
            \int \frac{ d \mathbb{P}^{F} }{ d \mathbb{P}^{ F | Y } } 
            \, d \mathbb{Q}^{F}
    \notag
  \end{align}
  and analogously
  \begin{align}
    \label{eq_B_KLOfFidis}
           \kl( \Psi_{m} \{ \mathbf{ f \in \cdot } \},
                \Pi_{n} \{ \mathbf{f} \in \cdot | X, Y \} 
              ) 
    & = 
           \int \frac{ d \mathbb{Q}^{ \vec{F} } }
                     { d \mathbb{P}^{ \vec{F} } } 
           \, d \mathbb{Q}^{ \vec{F} } 
         + 
           \int \frac{ d \mathbb{P}^{ \vec{F} }     }
                     { d \mathbb{P}^{ \vec{F} | Y } } 
           \, d \mathbb{Q}^{ \vec{F} },
  \end{align}
  where \( \vec{F}: = ( F( X_{1} ), \dots, F( X_{n} ) ) \) denotes the
  finite dimensional evaluation of the process \( F \). 
  Therefore, it is sufficient to show that the integrands in Equations
  \eqref{eq_B_KLOfGPs} and \eqref{eq_B_KLOfFidis} and exist and coincide
  respectively.

  For the first integrand, we initially note that \( \mathbb{Q} \) is absolutely
  continuous with respect to \( \mathbb{P} \) with 
  \begin{align}
    \label{eq_B_RadonNikodymRepresentation}
      \frac{ d \mathbb{Q} }{ d \mathbb{P} } 
    = 
      \Big( \frac{ d \mathbb{P}^{ W_{m} } }{ d \mathbb{P}^{ W^{*} } } \Big) 
      \circ W^{*},
  \end{align}
  where
  \( 
      d \mathbb{P}^{ W_{m} } / d \mathbb{P}^{ W^{*} } 
    = d \mathbb{Q}^{ W_{m} } / d \mathbb{P}^{ W^{*} } 
  \)
  exists, since under \( \mathbb{P} \) both \( W_{m} \) and \( W^{*} \) are
  non-degenerate Gaussians.
  For Equation \eqref{eq_B_RadonNikodymRepresentation}, consider that indeed,
  for any \( A \in \mathscr{A} \), we have
  \begin{align}
    & \ \ \ \
      \int_{A} 
        \frac{ d \mathbb{P}^{ W_{m} } }{ d \mathbb{P}^{ W^{*} } }(W^{*}) 
      \, d \mathbb{P}  
    = 
      \int
        \frac{ d \mathbb{P}^{ W_{m} } }{ d \mathbb{P}^{ W^{*} } }(w) 
      \, \mathbb{P}( A | W^{*} = w ) \mathbb{P}^{ W^{*} }( d w ) 
    = 
      \mathbb{Q}(A).
  \end{align}
  This can now be used to prove 
  \( 
      d \mathbb{Q}^{F} / d \mathbb{P}^{F} 
    = d \mathbb{Q}^{ \vec{F} } / d \mathbb{P}^{ \vec{F} } 
  \):
  For any Borel set \( B \subset L^{2}(G) \), 
  \begin{align}
        \mathbb{Q} \{ F \in B \} 
    & = 
        \int_{ F \in B } 
          \frac{ d \mathbb{Q}^{ W_{m} } }{ d \mathbb{P}^{ W^{*} } }( W^{*} )
        \, d \mathbb{P} 
    = 
        \iint 
          \mathbf{1}_{B}(f) 
          \frac{ d \mathbb{Q}^{ W_{m} } }{ d \mathbb{P}^{ W^{*} } }
          ( K_{ \sigma }^{-1} ( \mathbf{f} + \varepsilon ) )
        \, \mathbb{P}^{ \varepsilon }( d \varepsilon )  
        \, \mathbb{P}^{F}( d f ),
  \end{align}
  where we have used the fact that 
  \( W^{*} = K_{ \sigma }^{-1} Y \) is a function of 
  \( \mathbf{f} \) and \( \varepsilon \).
  This implies
  \begin{align}
        \frac{ d \mathbb{Q}^{F}  }{ d \mathbb{P}^{F} } 
    & = 
        \int \frac{ d \mathbb{Q}^{ W_{m} } }{ d \mathbb{P}^{ W^{*} } } 
             ( \mathbf{f} + \varepsilon )
        \, \mathbb{P}^{ \varepsilon }( d \varepsilon ), 
  \end{align}
  which only depends on \( \mathbf{f} = f( X_{1}, \dots, X_{n} ) \).
  It follows that for any Borel set \( B \subset \mathbb{R}^{n} \), 
  \begin{align}
    \int_{B} \frac{ d \mathbb{Q}^{F} } { d \mathbb{P}^{F} }( \mathbf{f} ) 
    \, d \mathbb{P}^{ \vec{F} }( d \mathbf{f} ) 
    & = 
    \int \mathbf{1}_{B}( \mathbf{f} ) 
         \frac{ d \mathbb{Q}^{F} } { d \mathbb{P}^{F} }(f) 
    \, d \mathbb{P}^{F}( d f ) 
    = 
    \mathbb{Q} \{ \vec{F} \in B \},
  \end{align}
  which establishes existence and equality of the first integrands.

  An analogous argument applies for the second integrands, since
  \begin{align}
    \frac{ d \mathbb{P}^{F} }{ d \mathbb{P}^{ F | Y } } 
    & 
    = 
    \Big( \frac{ d \mathbb{P}^{ F | Y } }{ d \mathbb{P}^{F} }
    \Big)^{-1}
    = 
    \frac{ p(Y) }{ p( Y | F ) } 
    = 
    \frac{ p(Y) }{ p( Y | \vec{F} ) }, 
  \end{align}
  where \( p \) is the likelihood of our model and we have used the fact that distribution of \( Y \) only depends on \( F \) evaluated at the design.
\end{proof}

\begin{proof} 
  [\normalfont \textbf{Proof of Lemma \ref{lem_KrylovSpaceDimension}} 
   (Krylov space dimension)]
  \label{prf_KrylovSpaceDimension}
  We prove the linear independence of the vectors spanning the Krylov space.
  From the linear independence of the \( ( \widehat{u}_{j} )_{ j \le n } \),
  it follows that for any \( \beta \in \mathbb{R}^{ \tilde m } \),
  \begin{align}
        0 
    = 
        \sum_{ k = 0 }^{ \tilde m - 1 } 
        \beta_{k} 
        K^{k} v_{0} 
    & = 
        \sum_{ j = 1 }^{n} 
        \langle v_{0}, \widehat{u}_{j} \rangle
        \sum_{ k = 0 }^{ \tilde m - 1 } 
        \beta_{k} \widehat{ \lambda }_{j}^{k}
        \widehat{u}_{j},
  \end{align}
  implies that 
  \( 
    \sum_{ k = 0 }^{ \tilde m - 1 } \beta_{k} \widehat{ \lambda }_{j}^{k} = 0
  \)
  for all \( j \le \tilde m \).
  This can be written as \( \Lambda \beta = 0 \), where 
  \( 
    \Lambda = ( \widehat{ \lambda }_{j}^{k} )_{ j, k } 
    \in \mathbb{R}^{ \tilde m \times \tilde m } 
  \)
  is a Vandermonde matrix.
  Since 
  \( 
      \widehat{ \lambda }_{1} 
    > \dots > \widehat{ \lambda }_{ \tilde m } > 0
  \), 
  \( \Lambda \) is invertible, which implies \( \beta = 0 \).
\end{proof}

\begin{lemma} 
  [Convex function decay,
   \cite{CardotEtal2007CLTInFunctionalLinearRegression}]
  \label{lem_ConvexFunctionDecay}
  Assume that there is a convex function 
  \( \lambda: [ 0, \infty ) \to [ 0, \infty ) \) such that 
  \( \lambda_{j} = \lambda(j) \) and \( \lim_{ j \to \infty } \lambda(j) = 0 \).
  Then, for a fixed \( m \) 
  \begin{align}
            \sum_{ k \ne m } 
            \frac{ \lambda_{k} }{ | \lambda_{m} - \lambda_{k} | } 
    & \le 
            C m \log m 
    \qquad \text{ and } \qquad 
            \lambda_{j} - \lambda_{k} \ge \Big( 1 - \frac{j}{k} \Big) \lambda_{j} 
  \end{align}
  for all \( j < k \).
\end{lemma}

\begin{lemma}[Well definedness of the Krylov space]
  \label{lem_WellDefinednessOfTheKrylovSpace}
  Under Assumptions \hyperref[ass_SimplePopulationEigenvalues]
                             {\normalfont \textbf{{\color{blue} (SPE)}}},
  \hyperref[ass_EigenvalueDecay]{\normalfont \textbf{{\color{blue} (EVD)}}} and
  \hyperref[ass_KarhunenLoeveMoments]
           {\normalfont \textbf{{\color{blue} (KLMom)}}},
  consider \( m \) such that
  \( 
    m = o( ( \sqrt{n} / \log n ) \land 
           ( n^{ ( p/4 - 1 ) / 2 } ( \log n )^{ p/8 - 1 } ) ) 
  \).
  Then, the following hold:
  \begin{enumerate}[label=(\roman*)]

    \item The empirical eigenvalues satisfy 
      \( 
        \widehat{ \lambda }_{1} > \widehat{ \lambda }_{2}          > \dots
                                > \widehat{ \lambda }_{ \tilde m } > 0 
      \) with probability converging to one.

    \item For \( v_{0}: = Z / \| Z \| \), where \( Z \sim N( 0, I ) \)
      independent of everything else, the tangens of the acute angles between 
      \( v_{0} \) and \( ( \widehat{u}_{i} )_{ i \le \tilde m } \) satisfies
      \begin{align}
        0 \le \tan( v_{0}, \widehat{u}_{i} ) \le C n^{2} 
        \qquad \text{ for all } i \le \tilde m
      \end{align}
      with probability converging to one.

    \item As long as \( f_{0} \in L^{ \infty }(G) \), the same result as in (ii)
      also holds for \( v_{0}: = Y / \| Y \| \).

  \end{enumerate}
  In particular, in both the settings of (ii) and (iii), Assumption
  \hyperref[ass_LanczosWelldefined]{\normalfont \textbf{{\color{blue} (LWdf)}}}
  is satisfied with probability converging to one.
\end{lemma}

\begin{proof}
  \label{prf_WellDefinednessOfTheKrylovSpace}
  For (i), assume that there exists an \( i < \tilde m \) such that 
  \( \widehat{ \lambda }_{i} = \widehat{ \lambda }_{ i + 1 } \). 
  Then, for \( m_{0} = C \tilde m = C m \log n \) in Proposition 
  \ref{prp_RelativePerturbatonBounds}, the relative rank \( \mathbf{r}
_{i} \) satisfies the bound 
  \begin{align}
    \mathbf{ r}_{i} \le C \tilde m \log \tilde m 
                      \le C \sqrt{ \frac{n}{ \log n } }
    \qquad \qquad \text{ for all } i \le \tilde m 
  \end{align}
  according to Lemma \ref{lem_ConvexFunctionDecay} and our assumption on 
  \( m \). 
  Consequently, there is an event with probability converging to one on which
  \begin{align}
            \frac{ \lambda_{i} - \lambda_{ i + 1 } }{ \lambda_{i} }
    & \le 
            \Big| 
              \frac{ \widehat{ \lambda }_{i} - \lambda_{i} }{ \lambda_{i} }
            \Big| 
          + 
            \Big| 
              \frac{ \widehat{ \lambda }_{ i + 1 } - \lambda_{ i + 1 } }{ \lambda_{i} }
            \Big| 
    \le 
            C \sqrt{ \frac{ \log n }{n} }.
  \end{align}
  At the same time, however, Lemma \ref{lem_ConvexFunctionDecay} guarantees that
  \( 
        ( \lambda_{i} - \lambda_{ i + 1 } ) / \lambda_{i} 
    \ge
        1 - i / ( i + 1 ) 
    \ge
        1 / \tilde m 
  \).
  This contradicts that by our assumptions, 
  \( \tilde m = o( \sqrt{ n / \log n } ) \). 
  Similarly, \( \widehat{ \lambda }_{ \tilde m } = 0 \) implies that 
  \( \lambda_{ \tilde m } \le \lambda_{ \tilde m } C \sqrt{ \log(n) / n } \),
  which contradicts Assumption
  \hyperref[ass_SimplePopulationEigenvalues]
           {\normalfont \textbf{{\color{blue} (SPE)}}} for \( n \) sufficiently
  large.

  For (ii), fix \( i \le \tilde m \) and \( t > 0 \).
  We then have
  \( 
    \tan( v_{0}, \widehat{u}_{i} ) \le 1 / \cos( v_{0}, \widehat{u}_{i} ) 
  \) 
  and since we are considering the acute angle 
  \( 
    \cos( v_{0}, \widehat{u}_{i} ) = | \langle v_{0}, \widehat{u}_{i} \rangle |
  \).
  With \( v_{0} = Z / \| Z \| \), we have
  \begin{align}
            \mathbb{P} \{ | \langle v_{0}, \widehat{u}_{i} \rangle | \le t \} 
    & \le 
            \mathbb{P} \{ | \langle Z, \widehat{u}_{i} \rangle | \le 2 \sqrt{n} t \} 
          + 
            \mathbb{P} \{ \| Z \| \ge 2 \sqrt{n} \}.
  \end{align}
  Since \( \| Z \| - \sqrt{n} \) is subgaussian, see
  \cite{Vershynin2018HDProb}, the second probability is smaller than 
  \( \exp( - c n ) \) and since 
  \( 
    \langle Z, \widehat{u}_{i} \rangle \sim N( 0, \| \widehat{u}_{i} \|^{2} )
  \),
  the first probability is
  \begin{align}
          \frac{1}{ \sqrt{ 2 \pi } }
          \int_{ - 2 \sqrt{n} t }^{ 2 \sqrt{n} t } 
             e^{ \frac{ - x^{2} }{2} } 
          \, d x 
    \le 
          c \sqrt{n} t.
  \end{align}
  The result now follows from setting \( t = n^{ - 2 } \) and applying a union
  bound noting that \( \tilde m n^{ - 3 / 2 } \to 0 \).

  For (iii), analogously,
  \begin{align}
            \mathbb{P} \{ | \langle v_{0}, \widehat{u}_{i} \rangle | \le t \} 
    & \le 
            \mathbb{P} \{ 
              | \langle
                \mathbf{f}_{0} + \boldsymbol{\varepsilon},
                \widehat{u}_{i}
              \rangle | 
      \le
              C ( \| f_{0} \|_{ L^{2} } \lor \sigma ) \sqrt{n} t
            \} 
    \\
    & + 
            \mathbb{P} \{ 
              \| \mathbf{f}_{0} + \boldsymbol{ \varepsilon } \| 
              \ge 
              C ( \| f_{0} \|_{ L^{2} } \lor \sigma ) \sqrt{n}
            \} 
    \notag
    \\
    & =:    
            (\text{I}) + (\text{II}).
    \notag
  \end{align}
  Via independence, the first term satisfies
  \begin{align}
            (\text{I})
    & = 
            \int 
              \mathbb{P}^{ \varepsilon } \{ 
                | \langle 
                  \mathbf{f}_{0} + \boldsymbol{ \varepsilon },
                  \widehat{u}_{i} 
                \rangle | 
                \le C ( \| f_{0} \|_{ L^{2} } \lor \sigma ) \sqrt{n} t
              \}
            \, d \mathbb{P}^{X_{1}, \dots, X_{n}}
    \\
    & = 
            \int 
              N( \langle \mathbf{f}_{0}, \widehat{u}_{i} \rangle, \sigma^{2} ) 
              ( [ - C ( \| f_{0} \|_{ L^{2} } \lor \sigma  ) \sqrt{n} t, 
                    C ( \| f_{0} \|_{ L^{2} } \lor \sigma  ) \sqrt{n} t ] ) 
            \, d \mathbb{P}^{X_{1}, \dots, X_{n}}
    \notag
    \\
    & \le 
            \int 
              N( 0, \sigma^{2} ) 
              ( [ - C ( \| f_{0} \|_{ L^{2} } \lor \sigma  ) \sqrt{n} t, 
                    C ( \| f_{0} \|_{ L^{2} } \lor \sigma  ) \sqrt{n} t ] ) 
            \, d \mathbb{P}^{X_{1}, \dots, X_{n}}
    \notag
    \\
    & \le 
            C ( \| f_{0} \|_{ L^{2} } / \sigma + 1  ) \sqrt{n} t.
    \notag
  \end{align}
  The second term can be estimated via
  \begin{align}
            (\text{II})
    & \le 
            \mathbb{P} \{ \| \boldsymbol{ \varepsilon } / \sigma \| 
                          \ge C \sqrt{n}
                       \} 
            + 
            \mathbb{P} \{ 
              \| f_{0} \|_{n}^{2} - \| f_{0} \|_{ L^{2} }^{2}
              \ge 
              C \| f_{0} \|_{ L^{2} }^{2}
            \}
    \\
    & \le 
            e^{ - c n } 
            + 
            \exp \Big( 
              \frac{ - c n \| f_{0} \|_{ L^{2}  }^{4} }
                   {       \| f_{0} \|_{ \infty }^{4} }
            \Big),
    \notag
  \end{align}
  using the Gaussianity of \( \boldsymbol{ \varepsilon } \) as before and the
  norm concentration result in \cite{Wainwright2019HDStats} after Equation
  (14.2).
  The result now follows exactly as in the setting of \( v_{0} = Z / \| Z \| \).
\end{proof}

\begin{lemma}[Eigenvector product terms]
  \label{lem_EigenvectorProductTerms}
  Under Assumptions \hyperref[ass_SimplePopulationEigenvalues]
                             {\normalfont \textbf{{\color{blue} (SPE)}}},
                    \hyperref[ass_EigenvalueDecay]
                             {\normalfont \textbf{{\color{blue} (EVD)}}} and 
                    \hyperref[ass_KarhunenLoeveMoments]
                             {\normalfont \textbf{{\color{blue} (KLMom)}}},
  consider 
  \( 
    m = o( ( \sqrt{n} / \log n ) \land 
           ( n^{ ( p/4 - 1 ) / 2 } ( \log n )^{ p/8 - 1 } ) ) 
  \).
  Then for \( \tilde p = m \) in Theorems \ref{thm_LanczosEigenvalueBound} and
  \ref{thm_LanczosEigenvectorBound}, there exist constants \( c, C > 0 \) such
  that
  \begin{align}
            \gamma_{i} \ge 1 + c,
    \qquad 
            \kappa_{i} \le C^{i} 
    \qquad \text{ and } \qquad 
            \kappa_{ i, m } \le C^{m}
    \qquad \text{ for all } i \le m 
  \end{align}
  with probability converging to one.
\end{lemma}

\begin{proof}[Proof]
  As in Lemma \ref{lem_WellDefinednessOfTheKrylovSpace}, we consider the event
  from Proposition \ref{prp_RelativePerturbatonBounds} with 
  \( m_{0} = C \tilde m = C m \log n \).

  For 
  \( 
    \gamma_{i} 
  = 
    1 + 2 ( \widehat{ \lambda }_{i} - \widehat{ \lambda }_{ i + m + 1 } ) /
          ( \widehat{ \lambda }_{ i + m + 1 } - \widehat{ \lambda }_{n} ) 
  \), the relative perturbation bound yields
  \begin{align}
            \frac{ \widehat{ \lambda }_{i} - \widehat{ \lambda }_{ i + m + 1 } }
                 { \widehat{ \lambda }_{ i + m + 1 } - \widehat{ \lambda }_{n} } 
    & \ge 
            \frac{ \lambda_{i} - \lambda_{ i + m + 1 } - C \lambda_{i} \sqrt{ \log(n) / n } }
                 { \lambda_{ i + m + 1 } + C \lambda_{ i + m + 1 } \sqrt{ \log(n) / n } }
      \ge 
            \frac{ \lambda_{i} - \lambda_{ i + m + 1 } }{ 2 \lambda_{i} }
          - 
            C \sqrt{ \frac{ \log n }{n} }
  \end{align}
  for \( n \) sufficiently large, which establishes the first claim, since
  \begin{align}
            \frac{ \lambda_{i} - \lambda_{ i + m + 1 } }{ \lambda_{i} } 
    & \ge 
            1 - \frac{i}{ i + m + 1 }
      = 
            \frac{ m + 1 }{ i + m + 1 } 
      \ge 
            \frac{1}{2}
  \end{align}
  according to Lemma \ref{lem_ConvexFunctionDecay}.

  For \( \kappa_{i} \), we have by analogous reasoning using Proposition
  \ref{prp_RelativePerturbatonBounds} that for \( n \) sufficiently large
  \begin{align}
            \kappa_{i} 
    & =  
            \prod_{ j = 1 }^{ i - 1 } 
            \frac{ \widehat{ \lambda }_{j} - \widehat{ \lambda }_{n} }
                 { \widehat{ \lambda }_{j} - \widehat{ \lambda }_{i} } 
      \le 
            \prod_{ j = 1 }^{ i - 1 } 
            \frac{ \widehat{ \lambda }_{j} }
                 { \widehat{ \lambda }_{j} - \widehat{ \lambda }_{i} } 
      \le 
            \prod_{ j = 1 }^{ i - 1 } 
            \frac{ \lambda_{j} + C \lambda_{j} \sqrt{ \log(n) / n } }
                 { \lambda_{j} - \lambda_{i} - 
                   C \lambda_{j} \sqrt{ \log(n) / n } 
                 }
    \\
    & \le 
            \prod_{ j = 1 }^{ i - 1 } 
            \frac{ C \lambda_{j} } { \lambda_{j} - \lambda_{i} }
      \le 
            \prod_{ j = 1 }^{ i - 1 } 
            \frac{ C i }{ i - j } 
      \le 
            C^{i} \frac{ i^{ i + 1 } }{ i ! } 
      \le 
            \sqrt{i} C^{i}
      \le 
            C^{i},
    \notag 
  \end{align}
  where we have used that by Lemma \ref{lem_ConvexFunctionDecay}
  \( 
        ( \lambda_{j} - \lambda_{i}  ) / \lambda_{j}
    \ge 
        i / ( i + j ) 
    \ge 
        m^{-1}
  \),
  \( m = o(\sqrt{n} / \log n) \) and the estimate
  \( i ! \ge \sqrt{ 2 \pi i } ( i / e )^{i} \) from Sterling's approximation of
  the factorial.

  For the final inequality, we have analogously that
  \begin{align}
            \kappa_{ i, m }
    & = 
            \prod_{ j = i + 1 }^{ i + m } 
            \frac{ \widehat{ \lambda }_{j} - \widehat{ \lambda }_{n} }
                 { \widehat{ \lambda }_{i} - \widehat{ \lambda }_{j} } 
      \le 
            \prod_{ j = i + 1 }^{ i + m } 
            \frac{ \widehat{ \lambda }_{j} }
                 { \widehat{ \lambda }_{i} - \widehat{ \lambda }_{j} } 
      \le 
            \prod_{ j = i + 1 }^{ i + m } 
            \frac{ C \lambda_{j} }{ \lambda_{i} - \lambda_{j} }
      \le 
            \prod_{ j = i + 1 }^{ i + m } 
            \frac{ C j }{ j - i }
      \le 
            \frac{ ( C m )^{m} }{ m ! }
      \le 
            C^{m}.
  \end{align}
  on the respective event for \( n \) sufficiently large.
  This finishes the proof.
\end{proof}

\begin{proof} 
  [\normalfont \textbf{Proof of Proposition 
   \ref{prp_ProbabilisticBoundsForLanczosEigenpairs}}
   (Probabilistic bounds for Lanczos eigenpairs)]
  \label{prf_ProbabilisticBoundsBorLanczosEigenpairs}
  We separate the proof into steps.

  \textbf{Step 1: Eigenvalue bound.}
  Under our assumptions, Lemma \ref{lem_WellDefinednessOfTheKrylovSpace}
  guarantees that \( \mathcal{K}_{ \tilde m } \) is well defined with high
  probability and with \( \tilde p = m \), Theorem
  \ref{thm_LanczosEigenvalueBound} yields that for any \( i \le m \), we have
  \begin{align}
    \label{eq_EigenvalueBound}
          0 
    & \le
          \widehat{ \lambda }_{i} - \tilde \lambda_{i} 
      \le 
          \widehat{ \lambda }_{i} 
          \Big( 
            \frac{ \tilde \kappa_{i} \kappa_{ i, m } }
                 { T_{ \tilde m - i - m }( \gamma_{i} ) } 
            \tan( \widehat{u}_{i}, v_{0} )
          \Big)^{2} 
    \quad \text{ with } \quad
          \tilde \kappa_{i}: 
    = 
          \prod_{ j = 1 }^{ i - 1 } 
          \frac{ \tilde \lambda_{j} - \widehat{ \lambda }_{n} }
               { \tilde \lambda_{j} - \widehat{ \lambda }_{i} }
  \end{align}
  as long as \( \tilde \lambda_{ i - 1 } > \widehat{ \lambda }_{i} \) for \(
  i > 1 \).
  By Lemmas \ref{lem_WellDefinednessOfTheKrylovSpace} and
  \ref{lem_EigenvectorProductTerms},
  \begin{align}
          \sup_{ i \le m }
          \Big( 
            \frac{ \kappa_{ i, m } \tan( \widehat{u}_{i}, v_{0} ) }
                 { T_{ \tilde m - i - m }( \gamma_{i} ) }
          \Big)^{2} 
    & \le 
          \frac{ C^{m} }{ ( 1 + c )^{ \tilde m } } 
      \le 
          ( 1 + c )^{ - \tilde m }
  \end{align}
  with probability converging to one on the event from Proposition
  \ref{prp_RelativePerturbatonBounds} with 
  \( m_{0} = C \tilde m = C m \log n \).
  Note that we have used the lower bound on the Tschebychev polynomial from
  Equation \eqref{eq_T_TschebychevLowerBound} and changed the constants from
  inequality to inequality.

  We now show inductively that on the same event
  \begin{align}
    \label{eq_B_InductiveBoundTildeKappa}
        \tilde \kappa_{i} 
    & = 
        \prod_{ j = 1 }^{ i - 1 } 
        \frac{ \tilde \lambda_{j} - \widehat{ \lambda }_{n} }
             { \tilde \lambda_{j} - \widehat{ \lambda }_{i} } 
    \le
        C^{i} i!  \qquad \text{ for all } i \le m
  \end{align}
  for n sufficiently large.
  For \( i = 1 \), we have \( \tilde \kappa_{i} = 1 \) by definition.
  We now assume the statement is true up to the index \( i \) and all
  denominators in the product are positve.
  Then, by Lemma \ref{lem_ElementaryPropertiesOfLanczosEigenquantities} (iii),
  \begin{align}
    \label{eq_B_TildeKappaInductionStep}
          \tilde \kappa_{ i + 1 } 
    & \le 
          \frac{ \widehat{ \lambda }_{i} }
               { \tilde \lambda_{i} - \widehat{ \lambda }_{ i + 1 } } 
          \tilde \kappa_{i}
      \le 
          \frac{ \widehat{ \lambda }_{i} }
               { \tilde \lambda_{i} - \widehat{ \lambda }_{ i + 1 } } 
        C^{i} i!.
  \end{align}
  Further, using the relative perturbation bound from Proposition
  \ref{prp_RelativePerturbatonBounds}, we obtain
  \begin{align}
            \tilde \lambda_{i} - \widehat{ \lambda }_{ i + 1 } 
    & = 
            \tilde \lambda_{i} - \widehat{ \lambda }_{i} 
          + 
            \widehat{ \lambda }_{i} - \widehat{ \lambda }_{ i + 1 } 
      \ge 
            \tilde \lambda_{i} - \widehat{ \lambda }_{i} 
          + 
            \lambda_{i} - \lambda_{ i + 1 } 
          - 
            C \lambda_{i} \sqrt{ \frac{ \log n }{n} } 
    \\
    & \ge   
            \lambda_{i} - \lambda_{ i + 1 } 
          - 
            C \lambda_{i} \sqrt{ \frac{ \log n }{n} } 
          - 
            \widehat{ \lambda }_{i} ( 1 + c )^{ - \tilde m },
    \notag
  \end{align}
  where we have plugged the induction hypothesis into Equation
  \eqref{eq_EigenvalueBound} for the last inequality.
  Using the relative perturbation bound again and applying Lemma
  \ref{lem_ConvexFunctionDecay} then yields
  \begin{align}
          \tilde \lambda_{i} - \widehat{ \lambda }_{i} 
      \ge 
            \frac{ \lambda_{i} }{ i + 1 }
          - 
            C \lambda_{i} \sqrt{ \frac{ \log n }{n} } 
          - 
            \lambda_{i} ( 1 + c )^{ - \tilde m } 
      \ge 
            c \widehat{ \lambda }_{i} ( i + 1 )^{-1}.
    \notag
  \end{align}
  Plugging this into Equation \eqref{eq_B_TildeKappaInductionStep}, yields the
  claim in Equation \eqref{eq_B_InductiveBoundTildeKappa}.

  Together, on the event in question, we have
  \begin{align}
          0 
    & \le
          \widehat{ \lambda }_{i} - \tilde \lambda_{i} 
    \le
          \widehat{ \lambda}_{i}
          C^{i} i! ( 1 + c )^{ - \tilde m }
    \le 
          \lambda_{i} ( 1 + c )^{ - m }
          \qquad \text{ for all } i \le m, 
  \end{align}
  and \( n \) sufficiently large.

  \textbf{Step 2: \( \delta_{i} \)-term.}
  In order to prove the second statement in Proposition
  \ref{prp_ProbabilisticBoundsForLanczosEigenpairs}, we need to show that 
  the eigenpair \( ( \lambda^{*}, \tilde u^{*} ) \) from Theorem
  \ref{thm_LanczosEigenvectorBound} is given by 
  \( ( \tilde \lambda_{i}, \tilde u_{i} ) \) and control the remaining term
  \( 
    \delta_{i}^{2} = \min_{ \tilde \lambda_{j} \ne \tilde \lambda^{*} } 
                       | \widehat{ \lambda }_{i} - \tilde \lambda_{j} |
  \).
  We show that on the event from Proposition \ref{prp_RelativePerturbatonBounds} 
  \begin{align}
        \min_{ j \le \tilde m } | \widehat{ \lambda }_{i} - \tilde \lambda_{j} |
    = 
        \widehat{ \lambda }_{i} - \tilde \lambda_{i}
        \qquad \text{ and } \qquad 
        \delta_{i}^{2} 
    = 
        | \widehat{ \lambda }_{i} - \tilde \lambda_{i} | 
        \ge c \lambda_{i} i^{-1}
        \qquad \text{ for all } i \le m
  \end{align}
  with probability converging to one.

  We check the first statement inductively.
  Note that for \( i = 1 \), 
  \( 
    \tilde \lambda_{2} \le \tilde \lambda_{1} 
                       \le \widehat{ \lambda }_{1}
  \), see Lemma \ref{lem_ElementaryPropertiesOfLanczosEigenquantities} (iii),
  implies that the minimizer is taken by \( \tilde \lambda_{1} \). 
  Now, assume the statement is correct for any integer strictly smaller than 
  \( i \).
  Then, 
  \( 
        \tilde \lambda_{ i + 1 } \le \tilde \lambda_{i} 
    \le \widehat{ \lambda }_{i} \le \widehat{ \lambda }_{ i - 1 } 
  \)
  implies that the minimizer is either taken by \( \tilde \lambda_{i} \) or
  \( \tilde \lambda_{ i - 1 } \).

  In case of the latter, however, Step 1 implies that on the event in question
  \begin{align}
            \widehat{ \lambda }_{ i - 1 } - \widehat{ \lambda }_{i} 
    & \le 
            \widehat{ \lambda }_{ i - 1 } - \tilde \lambda_{ i - 1 } 
          + 
            \widehat{ \lambda }_{i} - \tilde \lambda_{ i - 1 }
    \le 
            \widehat{ \lambda }_{ i - 1 } - \tilde \lambda_{ i - 1 } 
          + 
            \widehat{ \lambda }_{i} - \tilde \lambda_{i}
    \le 
            \lambda_{ i - 1 } ( 1 + c )^{ - m }.
  \end{align}
  This contradicts the fact that on the same event by Proposition
  \ref{prp_RelativePerturbatonBounds} 
  \begin{align}
            \widehat{ \lambda }_{ i - 1 } - \widehat{ \lambda }_{i} 
    & \ge 
            \lambda_{ i - 1 } - \lambda_{i} 
          - 
            C \lambda_{ i - 1 } \sqrt{ \frac{ \log n }{n} } 
      \ge 
            \frac{ \lambda_{ i - 1 } }{ 2 i }
  \end{align}
  by Lemma \ref{lem_ConvexFunctionDecay} and \( n \) sufficiently large.
  Consequently, for \( n \) sufficiently large, the minimizer has to be taken by
  \( \tilde \lambda_{i} \).

  For the second statement, we then obtain that on the respective event
  \begin{align}
        \delta_{i}^{2} 
    & = 
        \min( 
          | \widehat{ \lambda }_{i} - \tilde \lambda_{ i + 1 } |,
          | \widehat{ \lambda }_{i} - \tilde \lambda_{ i - 1 } | 
        ). 
  \end{align}
  On the event from Proposition \ref{prp_RelativePerturbatonBounds}, then
  \begin{align}
            | \widehat{ \lambda }_{i} - \tilde \lambda_{ i + 1 } | 
    & \ge 
            | \lambda_{i} - \lambda_{ i + 1 } | 
          - 
            | \widehat{ \lambda }_{i} - \lambda_{i} | 
          - 
            | \widehat{ \lambda }_{ i + 1 } - \lambda_{ i + 1 } | 
          - 
            | \widehat{ \lambda }_{ i + 1 } - \tilde \lambda_{ i + 1 } | 
    \\
    & \ge 
            | \lambda_{i} - \lambda_{ i + 1 } | 
          - 
            C \lambda_{i} \sqrt{ \frac{ \log n }{n} }
          - 
            \lambda_{ i + 1 } ( 1 + c )^{ - m }
    \notag
    \\
    & \ge 
            c \lambda_{i} i^{-1}
          - 
            C \lambda_{i} \sqrt{ \frac{ \log n }{n} }
          - 
            c
            \lambda_{i} ( 1 + c )^{ - m }
    \notag
    \\
    & \ge 
            c \lambda_{i} i^{-1},
    \notag
  \end{align}
  where we have again used Lemma \ref{lem_ConvexFunctionDecay} and Step 1.
  An analogous bound holds for 
  \( | \widehat{ \lambda }_{i} - \tilde \lambda_{ i - 1 } | \). 


  \textbf{Step 3: Eigenvector bound.} 
  Combining Theorem \ref{thm_LanczosEigenvectorBound}, Lemma
  \ref{lem_EigenvectorProductTerms} and Step 2 yields that
  on the event from Proposition \ref{prp_RelativePerturbatonBounds}, 
  \begin{align}
            \frac{1}{2} 
            \| u_{i} u_{i}^{ \top } - \tilde u^{*} \tilde u^{* \top} \|_{ \text{HS} } 
    & = 
            \frac{1}{2} 
            \| u_{i} u_{i}^{ \top } - \tilde u_{i} \tilde u_{i}^{ \top } \|_{ \text{HS} } 
    \le 
            \Big( 1 + \frac{ \| K \|_{ \text{op} } }{ n \delta_{i}^{2} } \Big)
            C^{m} ( 1 + c )^{ - \tilde m }
    \\
    & \le 
            \| n^{-1} K \|_{ \text{op} }
            \lambda_{i}^{-1} C^{m} ( 1 + c )^{ \tilde m }.
    \notag
  \end{align}
  From the relative perturbation bound, on the same event
  \begin{align}
    \| n^{-1} K \|_{ \text{op} } 
    & = 
    \widehat{ \lambda }_{1} 
    \le 
    2 \lambda_{1}
  \end{align}
  and by Assumption 
  \hyperref[ass_EigenvalueDecay]{\normalfont \textbf{{\color{blue} (EVD)}}},
  \( \lambda_{i} \ge e^{ - c i } \).
  Together, this yields the second statement in Proposition
  \ref{prp_ProbabilisticBoundsForLanczosEigenpairs}.
\end{proof}

\begin{proposition} 
  [Expectation of the partial traces, \cite{ShaweTaylorWilliams2002Stability}]
  \label{prp_ExpectationOfTheTrace}
  For any \( m = 1, \dots, n \), 
  \begin{align}
        \mathbb{E} 
        \sum_{ j = 1 }^{m} \widehat{ \lambda }_{j} 
    \ge 
        \sum_{ j = 1 }^{m} \lambda_{j} 
    \qquad \text{ and } \qquad 
          \mathbb{E} 
          \sum_{ j = m + 1 }^{n} \widehat{ \lambda }_{j}
    & \le 
          \sum_{ j = m + 1 }^{ \infty } \lambda_{j}.
  \end{align}
\end{proposition}

\begin{proof} 
  [\normalfont \textbf{Proof of Remark \ref{rem_InconsistencyExample}}
   (Inconsistency example)]
  \label{prf_InconsistencyExample} 
  We consider the policy choices \( s_{j}: = \widehat{u}_{ j + 1 } \), 
  \( j \le n - 1 \).
  Analogous to the proof of Lemma \ref{lem_EVGP}, we obtain the approximation 
  \begin{align}
    C_{ > 1 }: = \sum_{ j = 2 }^{n} 
                 \frac{1}{ \widehat{ \mu }_{j} + \sigma^{2} } 
                 \widehat{u}_{j} \widehat{u}_{j}^{ \top }
  \end{align}
  of \( K_{ \sigma }^{-1} \) in Algorithm \ref{alg_GPApproximation}, i.e. 
  \( C_{ > 1 } \) includes all but the information from the first empirical
  eigenvector.
  The squared difference between the true approximate and the approximate
  posterior mean function is then given by
  \begin{align}
    & \ \ \ \
            \| k( X, \cdot )^{ \top } K_{ \sigma }^{-1} Y
             -
               k( X, \cdot )^{ \top } C_{ > 1 } Y
            \|_{ \mathbb{H} }^{2}
    = 
            \| k( X, \cdot )^{ \top } 
               ( \widehat{ \mu }_{1} + \sigma^{2} )^{-1} 
               \langle \widehat{u}_{1}, Y \rangle \widehat{u}_{1}
            \|_{ \mathbb{H} }^{2}
    \\
    & = 
            \frac{ \langle \widehat{u}_{1}, Y \rangle^{2} }
                 { ( \widehat{ \mu }_{1} + \sigma^{2} )^{2} }
            \langle \widehat{u}_{1}, K \widehat{u}_{1}
            \rangle
    = 
            \frac{ \widehat{ \mu }_{1} }
                 { ( \widehat{ \mu }_{1} + \sigma^{2} )^{2} }
            \langle \widehat{u}_{1}, Y \rangle^{2}
    = 
            \frac{ \widehat{ \mu }_{1} }
                 { ( \widehat{ \mu }_{1} + \sigma^{2} )^{2} }
            ( \langle \widehat{u}_{1}, \mathbf{f}_{0} \rangle 
            + \langle \widehat{u}_{1}, \varepsilon \rangle
            )^{2}.
    \notag
  \end{align}
  With probability \( 1 / 2 \),
  \( \langle \widehat{u}_{1}, \varepsilon \rangle \) has the same sign as 
  \( \langle \widehat{u}_{1}, \mathbf{f}_{0}\rangle \). 
  If we further assume that, using the notation from the proof of Proposition
  \ref{prp_RelativePerturbatonBounds}, the true data are sampled from 
  \( \mathbf{f}_{0} = \widehat{ \varphi }_{1} \), we obtain
  \begin{align}
        \langle \widehat{u}_{1}, \mathbf{f}_{0} \rangle^{2}
    & = 
        n^{2} \langle \widehat{u}_{1}, S_{n} f_{0} \rangle_{n}^{2}
    = 
        n^{2} \langle S_{n}^{*} \widehat{u}_{1}, f_{0} 
              \rangle_{ \mathbb{H} }^{2}
    = 
        n^{2} \| \widehat{ \varphi }_{1} \|_{ \mathbb{H} }^{2} 
    = 
        n^{2}.
  \end{align}
  Consequently on an event with probability larger than \( 1 / 2 \) 
  \begin{align}
            \| k( X, \cdot )^{ \top } K_{ \sigma }^{-1} Y
             -
               k( X, \cdot )^{ \top } C_{ > 1 } Y
            \|_{ \mathbb{H} }^{2}
    & \ge 
            \frac{ n^{2} \widehat{ \mu }_{1} }
                 { ( \widehat{ \mu }_{1} + \sigma^{2} )^{2} } 
    = 
            \frac{ n \widehat{ \lambda }_{1} }
                 { ( \widehat{ \lambda }_{1} + \sigma^{2} / n )^{2} } 
            \xrightarrow[ ]{ n \to \infty } \infty,  
  \end{align}
  as long as \( \widehat{ \lambda }_{1} \) is larger than a constant.
  Therefore, the above choice of actions will not produce a consistent estimator
  for the mean.
\end{proof}


\section{Proofs for the Examples}
\label{sec_ProofsForTheExamples}

In this section we collect the proofs for the considered examples.

\begin{proof} 
  [\normalfont \textbf{Proof of Corollary
  \ref{cor_PolynomiallyDecayingEigenvalues}} (Polynomially decaying
  eigenvalues)]
  \label{prf_PolynomiallyDecayingEigenvalues}

  We show below that the GP prior from Equation
  \eqref{eq_E_PolynomiallyDecayingPrior} satisfies the conditions of our general
  Theorem \ref{thm_ApproximateContractionRates} and hence the contraction rate
  of the approximate posteriors is a direct consequence.

  First note that Assumptions 
  \hyperref[ass_SimplePopulationEigenvalues]
           {\normalfont \textbf{{\color{blue} (SPE)}}},
  \hyperref[ass_EigenvalueDecay]
           {\normalfont \textbf{{\color{blue} (EVD)}}} and 
  \hyperref[ass_KarhunenLoeveMoments]
           {\normalfont \textbf{{\color{blue} (KLMom)}}} are all satisfied in
  the setting of Equation \eqref{eq_E_PolynomiallyDecayingPrior}.

  Next we verify Assumption \hyperref[ass_ConcentrationFunctionInequality]
                            {\normalfont \textbf{{\color{blue} (CFUN)}}}.
  First note that the small ball probability can be bounded as
  \begin{align}
             - \log \Pi_{n} \{ \| f \|_{ L^{2} } \le \varepsilon \}
    = 
             - \log \mathbb{P} \Big\{ \sum_{ j = 1 }^{ \infty } 
                                      j^{ - 1 - 2 \alpha / d } Z_{j}^{2} 
                                      \le \frac{ \varepsilon^{2} }{ \tau^{2} }
                               \Big\} 
    \lesssim 
             \Big( \frac{ \varepsilon }{ \tau } \Big)^{ d / \alpha },
  \end{align}
  see  Section 11.4.5 of \cite{GhosalvdVaart2017FundamentalsOfBayes}.
  Therefore, for 
  \( \epsilon = \varepsilon_n \asymp n^{- \beta / (2 \beta + d)} \) the small
  ball exponent term is bounded from above by a multiple of
  \( n^{d / (2 \beta + d)} \asymp n \epsilon_n^2 \).
  Furthermore, by taking \( h = \sum_{j = 1}^J f_{0, j} \phi_j \in \mathbb{H} \)
  with \( J = n^{d / (2 \beta + d)} \) and 
  \( f_{0, j} =  \langle f_0, \phi_j \rangle \), we get that 
  \begin{align}
        \| h \|_{ \mathbb{H} }^2
    & =
        \sum_{j = 1}^J \frac{ f_{0, j}^2 }{ \lambda_j }
    \le
        J^{1+2(\alpha-\beta)/d}  \tau^{-2}
        \sum_{j = 1}^J 
        f_{0, j}^2 j^{2 \beta / d} 
    \le
        \| f_0 \|_{ S^{\beta} }^2 n^{d / (2 \beta + d)},
    \\
        \| h - f_0 \|_{2}^2
    & =
        \sum_{j = J + 1}^{\infty} f_{0, j}^2
    \le
        J^{-2 \beta / d} \| f_{0} \|_{ S^{\beta} }^2
    =
        \| f_0 \|_{ S^{\beta} }^2.  n^{-2 \beta / (2 \beta + d)}.
    \notag
  \end{align}
  Hence the decentralization term with
  \( \varepsilon_n = n^{-\frac{\beta}{2\beta+d}}\|f_0\|_{S^{\beta}} \)  
  can be bounded from above by
  \begin{align*}
             \inf_{ h \in \mathbb{H}:\, \| h - f_0 \|_{2} \le \varepsilon_n }
             \| h \|_{ \mathbb{H} }^2 
    \lesssim 
             n \varepsilon_n^2.
  \end{align*}
  Together with the upper bound on the small ball exponent this yields
  Assumption \hyperref[ass_ConcentrationFunctionInequality]
                      {\normalfont \textbf{{\color{blue} (CFUN)}}}.

  By the definition of the eigenvalues
  \begin{align*}
             \sum_{j = m + 1}^\infty \lambda_j 
    =
             \sum_{j = m + 1}^\infty \tau^2 j^{-1 - 2 \alpha / d}
    \lesssim
             n^{2 ( \alpha - \beta ) / (2 \beta + d)} m^{-2 \alpha / d}
  \end{align*}
  and by Lemma 4 of \cite{NiemanEtal2022ContractionRates}
  \( 
             \mathbb{E} \widehat{ \lambda }_{m} 
    \lesssim \tau^{2} m^{-1 - 2 \alpha / d} 
  \).
  Hence for \( m = m_{n} \gtrsim n^{d / (2 \beta + d)} \) condition
  \eqref{eq_M_EigenvalueCondition} holds.
  Finally, the conditions \( \beta > d/2 \) and \( p > 4 + 8d / (2 \beta + d) \)
  guarantee that 
  \(
    m_{n} = o( ( \sqrt{n} / \log n ) \land 
                 ( n^{ ( p / 4 - 1 ) / 2 ( \log n )^{ p / 8 - 1 } } )
                )
  \). 
  This concludes the proof of the corollary.
\end{proof}

\begin{proof} 
  [\normalfont \textbf{Proof of Corollary
  \ref{cor_ExponentiallyDecayingEigenvalues}} (Exponentially decaying
  eigenvalues)]
  \label{prf_ExponentiallyDecayingEigenvalues}

  We again proceed by verifying that the conditions of Theorem
  \ref{thm_ApproximateContractionRates} hold for the prior
  \eqref{eq_E_ExponentiallyDecayingPrior}, directly implying our contraction
  rate results.
  First, we note that by construction, similarly to the polynomial case the
  Assumptions \hyperref[ass_SimplePopulationEigenvalues]
                       {\normalfont \textbf{{\color{blue} (SPE)}}},
              \hyperref[ass_KarhunenLoeveMoments]
                       {\normalfont \textbf{{\color{blue} (KLMom)}}} are all
  satisfied.
  Assumption \hyperref[ass_EigenvalueDecay]
             {\normalfont \textbf{{\color{blue} (EVD)}}} is not satisfied, since
  for any \( C > 1 \), 
  \( 
    \lambda( C j ) =  e^{ - \tau_{n} ( C - 1 ) } \lambda(j)
                  \to \lambda(j)
  \) 
  for \( n \to \infty \).
  It is, however, satisfied for 
  \( j = m_{0} = n^{ d / ( 2 \beta + d ) } \log n \), since
  \begin{align}
        \lambda( C m_{0} ) 
    = 
        \exp \big( - (C - 1) \tau_{n} m_{0}^{ 1 / d } \big)
        \lambda( m_{0} ) 
    \le
        \exp \big( -C ( \log n )^{d} \big) 
        \lambda( m_{0} ) 
    \le 
        \frac{1}{2} \lambda( m_{0} )
  \end{align}
  for \( C > 1 \) sufficiently large independent of \( n \in \mathbb{N} \).
  Since we only need to apply 
  Assumption \hyperref[ass_EigenvalueDecay]
             {\normalfont \textbf{{\color{blue} (EVD)}}} to this \( j = m_{0} \)
  to use Proposition \ref{prp_RelativePerturbatonBounds}, we may still apply 
  Theorem \ref{thm_ApproximateContractionRates} in this setting.

  Next we show that Assumption \hyperref[ass_ConcentrationFunctionInequality]
                               {\normalfont \textbf{{\color{blue} (CFUN)}}}
  holds.
  First, we give a lower bound for the prior small ball probability.
  Note that by independence, for any \( J \in \mathbb{N} \)
  \begin{align*}
          \mathbb{P} \Big\{ 
                       \sum_{j = 1}^{ \infty } 
                       e^{ - \tau j^{1/d} } Z_{j}^{2} \le \varepsilon^{2}
                     \Big\} 
    & \ge
          \mathbb{P} \Big\{ 
                       \sum_{j = 1}^{J} 
                       e^{ - \tau j^{1/d} } Z_{j}^{2} 
                       \le \frac{ \varepsilon^{2} }{2}
                     \Big\} 
          \mathbb{P} \Big\{ 
                       \sum_{ j = J + 1 }^{ \infty } 
                       e^{ - \tau j^{1/d} } Z_{j}^{2}
                       \le \frac{ \varepsilon^{2} }{2}
                     \Big\}. 
  \end{align*}
  Let us take \( J = ( \tau_n^{-1} \log n)^d = n^{d / (2 \beta + d)} \) and 
  \( \varepsilon_n^2 = C_0 n^{-2 \beta / (2 \beta + d)} \log n \), for some
  large enough \( C_0 > 0 \) to be specified later. 
  Then by noting that for \( 0 < \sigma_1 \le \sigma_2 \), the fraction of the
  centered Gaussian densities \( g_{ \sigma_{1} }, g_{ \sigma_{2} } \) with
  variances \( \sigma_1 \) and \( \sigma_2 \) satisfy that 
  \( g_{ \sigma_2 } / g_{ \sigma_1 } \ge \sigma_1 / \sigma_2 \),
  we get that
  \begin{align}
          \mathbb{P} 
          \Big\{ 
                \sum_{ j = 1 }^{J} e^{ - \tau j^{ 1 / d } } Z_{j}^{2} 
            \le \frac{ \varepsilon_{n}^{2} }{2} 
          \Big\} 
    & \ge 
          \exp \Big( -\tau \sum_{j = 1}^{J} j^{1/d} - J^{1 + 1/d} \Big) 
          \mathbb{P} \Big\{ 
                           \sum_{ j = 1 }^{J} Z_{j}^{2} 
                       \le \frac{ \varepsilon_{n}^{2} \exp(\tau J^{1/d}) }{2} 
                     \Big\} 
    \notag
    \\
    & \ge 
            \exp(- c \tau J^{ 1 + 1 / d }) 
            \mathbb{P} \Big\{ 
                             \frac{1}{J} \sum_{ j = 1 }^{J} Z_{j}^{2} 
                         \le \frac{ C_{0} \log n }{2} 
                       \Big\} 
    \\
    & \ge 
            \frac{1}{2} \exp(- c n^{d / ( 2 \beta + d )} \log n ) 
    \notag
  \end{align}
  for \( n \) sufficiently large, where the last inequality follows from the law
  of large numbers.
  We note that 
  \begin{align}
            \sum_{j = J + 1}^{ \infty } \exp( - \tau j^{ 1 / d } ) 
    & \le 
            \int_{J}^{ \infty }  \exp( - \tau x^{ 1 / d } ) \, d x 
    = 
            d 
            \int_{ J^{1/d} }^{ \infty } e^{- \tau y} y^{d - 1} \, d y 
    \\
    & \lesssim 
            \tau^{-1} \exp( - \tau J^{ 1 / d } ) J^{ 1 - 1 / d },
    \notag
  \end{align}
  where we have used
  \begin{align}
    \int_{ J^{1/d} }^{ \infty } e^{- \tau y} y^{d - 1} \, d y 
    & = 
    \tau^{-1} \exp( - \tau J^{ 1/d } ) J^{ 1 - 1/d } 
    + 
    ( d - 1 ) \tau^{-1} 
    \int_{ J^{1/d} }^{ \infty } e^{- \tau y} y^{d - 2} \, d y 
    \\
    & \le 
    \tau^{-1} \exp( - \tau J^{ 1/d } ) J^{ 1 - 1/d } 
    + 
    \frac{ ( d - 1 ) J^{ - 1/d } }{ \tau } 
    \int_{ J^{1/d} }^{ \infty } e^{- \tau y} y^{d - 1} \, d y 
    \notag
    \\
    & \le 
    \tau^{-1} \exp( - \tau J^{ 1/d } ) J^{ 1 - 1/d } 
    + 
    \frac{1}{2}
    \int_{ J^{1/d} }^{ \infty } e^{- \tau y} y^{d - 1} \, d y 
    \notag
  \end{align}
  for \( n \) sufficiently large.
  Via Markov's inequality, we arrive at
  \begin{align}
          \mathbb{P}
          \Big\{ \sum_{ j = J + 1 }^{ \infty } 
                 \exp( - \tau j^{ 1 / d } ) Z_{j}^{2} 
             \le 
                 \frac{ \varepsilon_{n}^{2} }{2}
          \Big\} 
    & \ge 
          1 - \frac{2}{ \varepsilon_{n}^{2} } 
              \sum_{ j = J + 1 }^{ \infty } 
              \exp( - \tau j^{ 1 / d } )
    \\
    & \ge 
          1 - \frac{ c J \exp( - \tau J^{1/d} )           }
                   { J^{ 1 / d } \tau \varepsilon_{n}^{2} } 
      \ge 
          \frac{1}{2}
    \notag
  \end{align}
  for \( C_0 \) large enough in the definition of \( \varepsilon_n \) above.

  Furthermore, with the same notation as in the proof of Corollary
  \ref{cor_ExponentiallyDecayingEigenvalues}, taking 
  \( h = \sum_{j = 1}^J f_{0, j} \phi_j \in \mathbb{H} \) with 
  \( J = ( \tau_n^{-1} \log n )^d = n^{d / (2 \beta + d)} \), we get that
  \begin{align}
    \|h-f_0\|_{2}^2&=\sum_{j=J+1}^{\infty}f_{0,j}^2\leq J^{-2\beta/d}\|f_0\|_{S^{\beta}}^2=n^{-\frac{2\beta}{2\beta+d}}\|f_0\|_{S^{\beta}}^2,\\
    \|h\|_{\mathbb{H}}^2 
    & =
    \sum_{j = 1}^J \frac{ f_{0, j}^2 }{ \lambda_{j} }
    \le
    \max_{j \leq J} \big( j^{-2 \beta / d} \exp( \tau j^{ 1 / d } ) \big) 
    \sum_{j = 1}^J j^{2 \beta / d} f_{0, j}^2
    \le 
    C n^{\frac{d}{2\beta+d}} \|f_0\|_{S^{\beta}}^2,
    \notag
  \end{align}
  where in the last inequality we have used that the function
  \( x \mapsto e^{\tau x} x^{-2 \beta} \) is convex for any \( \beta > 0 \) and
  therefore the maximum is taken at one of the end points, i.e.,
  \begin{align}
             \max_{j \le J} \big( j^{-2 \beta / d} \exp(\tau j^{1/d}) \big)
    \le      e^{\tau} \lor n J^{-2 \beta / d} 
    \lesssim n^{d / (2 \beta + d)}. 
  \end{align}

  Hence the decentralization term for \( f_0 \in S^{\beta}(M) \) can be bounded
  from above for \( \epsilon_{n} \gtrsim n^{-\beta / (2 \beta + d)} \) by
  \begin{align*}
             \inf_{ h \in \mathbb{H}: \| h - f_0 \|_{2} \le \epsilon_{n} }
             \|h\|_{\mathbb{H}}^2 
    \lesssim n^{d/(d+2\beta)} 
    \lesssim n \epsilon_{n}^2.
  \end{align*}
  Combining the upper bounds on the decentralization term and log small ball
  probability, we get that that
  \begin{align*}
              \varphi_{ f_0 }(\varepsilon_n)
    \le       C n^{d / (2 \beta + d)} \log n 
    \lesssim  n\varepsilon_n^2.
  \end{align*}
  For  condition \eqref{eq_M_EigenvalueCondition}, we simply note that for
  \( m = m_{n} \ge n^{d / (2 \beta + d)} \), as before,
  \begin{align}
             \sum_{j = m + 1}^{\infty} \lambda_j
    =        \sum_{j = m + 1}^{\infty} \exp( -\tau j^{1/d} )
    \lesssim \tau^{-1} \exp( -\tau m^{1/d} ) m^{1 - 1/d}
    \lesssim n^{-2 \beta / (2\beta + d)},
  \end{align}
  and with the analogous reasoning as in Lemma 4 of
  \cite{NiemanEtal2022ContractionRates},
  \( \mathbb{E} \widehat{\lambda}_{ m + 1 } \lesssim 1/n \).
  Indeed, for 
  \( \underline{m}: = \lfloor ( \tau^{-1} ( \log n - C ) )^{d} \rfloor < m \)
  with \( C > 0 \) large enough, we have
  \begin{align}
    \sum_{ j = \underline{m} }^{m} 
    \exp( - \tau j^{ 1 / d } )
    & \ge 
    \int_{ ( \underline{m}  - 1 )^{ 1 / d } }^{ m^{ 1/d } } 
      e^{ - \tau y } y^{ d - 1 }
    \, d y  
    \ge 
    \Big[ 
      -\tau^{-1} e^{ - \tau y } y^{ d - 1 } 
    \Big]_{ ( \underline{m} - 1 )^{ 1 / d } }^{ m^{ 1/d } } 
    \\
    & \gtrsim 
    \tau^{-1} \exp \big(- \tau m^{1/d} \big) m^{1/d} 
    \gtrsim 
    n^{ - 2 \beta / ( 2 \beta + d ) }
    \notag
  \end{align}
  via partial integration.
  Consequently, there exists an 
  \( i \in \{ \underline{m}, \underline{m} + 1, \dots, m \} \) with 
  \( \mathbb{E} \widehat{ \lambda }_{i} \le C \lambda_{i} \).
  Otherwise
  \begin{align}
    \mathbb{E} \sum_{ j = \underline{m} }^{m} \widehat{ \lambda }_{j}
    \ge  
    C \sum_{ j = \underline{m} }^{m} \lambda_{j}
    = 
    C \sum_{ j = \underline{m} }^{m} \exp( - \tau j^{ 1 / d } )
    \ge 
    C \sum_{ j = \underline{m} }^{\infty} \lambda_{j}, 
  \end{align}
  which contradicts Proposition \ref{prp_ExpectationOfTheTrace} for \( C > 0 \)
  large enough.
  This implies
  \begin{align}
    \mathbb{E} \widehat{ \lambda }_{m + 1} 
    & \le 
    \mathbb{E} \widehat{ \lambda }_{i} 
    \le 
    C \lambda_{i} 
    \le 
    C \lambda_{ \underline{m} }
    = 
    C n^{-1}. 
  \end{align}
  
  Finally, the conditions \( \beta > d/2 \) and \( p > 4 + 8d / (2 \beta + d) \)
  guarantee that 
  \(
    m_{n} = o(      ( \sqrt{n} / \log n ) \\ 
              \land ( n^{ ( p / 4 - 1 ) / 2 } ( \log n )^{ p / 8 - 1 }  )
             )
  \). 
  This concludes the proof of the corollary.
\end{proof}


\section{Complementary results}
\label{sec_ComplementaryResults}

\begin{proof}[\normalfont \textbf{Proof of Proposition \ref{prp_ContractionOfApproximation}} (Contraction of approximation)]
  \label{prf_ContractionOfApproximation}
  Fix \( M_{n} \to \infty \), set \( M_{n}': = \sqrt{ M_{n} } \) and \( \mathcal{F}_{n}: = \{ d_{H}( \cdot, f_{0} ) \ge M_{n} \varepsilon_{n} \} \).
  Since by Markov's inequality,
  \begin{align}
          \mathbb{P}_{ f_{0} }^{ \otimes n } \{ 
                \Pi_{n} \{ \mathcal{F}_{n} | X, Y \}
            \ge
                M_{n}'
                \mathbb{E}_{ f_{0} } \Pi_{n} \{ \mathcal{F}_{n} | X, Y \}
          \} 
    & \le
          \frac{1}{ M_{n}' } 
    \xrightarrow[ ]{ n \to \infty } 0, 
  \end{align}
  we can restrict the event \( A_{n} \) from Proposition \ref{prp_StandardContractionRate} such that 
  \begin{align}
        \Pi_{n} \{ \mathcal{F}_{n} | X, Y \} \mathbf{1}_{ A_{n} }
    \le
        C_{1} M_{n}' 
        \exp( - C_{2} n M_{n}^{2} \varepsilon_{n}^{2} ).
  \end{align}

  Exactly as in the proof of Theorem 5 in \cite{RaySzabo2019VariationalBayes},
  we can then use the duality formula, see Corollary 4.15 in
  \cite{BoucheronEtal2013Concentration},
  \begin{align}
    \label{eq_B_KLDuality}
        \kl( \mathbb{Q}, \mathbb{P} ) 
    & = 
        \sup_{g} \Big( 
          \int g \, d \mathbb{Q} - \log \int e^{g} \, d \mathbb{P}  
        \Big)
  \end{align}
  for the Kullback-Leibler divergence between two probability measures 
  \( \mathbb{Q} \) and \( \mathbb{P} \) on the same space.
  The supremum above is taken over all measurable functions \( f \) such that 
  \( \int e^{g} \, d \mathbb{P} < \infty \). 
  Applying Equation \eqref{eq_B_KLDuality} with 
  \( \mathbb{Q} = \nu_{n}, \mathbb{P} = \Pi_{n}( \cdot | X, Y ) \)
  and 
  \( g = n M_{n}^{2} \varepsilon_{n}^{2} \mathbf{1}_{ \mathcal{F}_{n} } \)
  yields that on \( A_{n} \cap A_{n}' \),
  \begin{align}
            n M_{n}^{2} \varepsilon_{n}^{2} 
            \nu_{n}( \mathcal{F}_{n} ) \mathbf{1}_{ A_{n} \cap A_{n}' } 
    & \le 
            \kl( \nu_{n}, \Pi_{n}( \cdot | X, Y ) ) \mathbf{1}_{ A_{n}' } 
          + 
            \log 
            \int 
              e^{ n M_{n}^{2} \varepsilon_{n}^{2} 
                  \mathbf{1}_{ \mathcal{F}_{n} }(f) 
                }
            \, \Pi_{n}( d f | X, Y )  
            \mathbf{1}_{ A_{n} }
    \notag 
    \\
    & \le 
            n M_{n}'^{2} \varepsilon_{n}^{2}
          + 
            e^{ n M_{n}^{2} \varepsilon_{n}^{2} / 2 } 
            \Pi_{n}( \mathcal{F}_{n} | X, Y ) \mathbf{1}_{ A_{n} }
    \\
    & \le
            n M_{n}'^{2} \varepsilon_{n}^{2}
          + 
            C_{1} M_{n}'  
            e^{ ( 1 - C_{2} ) n M_{n}^{2} \varepsilon_{n}^{2} }.
    \notag
  \end{align}
  Note that for the second inequality, we have used 
  \( \log( 1 + x ) \le x \), \( x \ge 0 \).
  Rearranging the terms and using that the constant \( C_{2} \) from Proposition
  \ref{prp_ContractionOfApproximation} can be chosen strictly larger than 
  \( 1 \) implies that 
  \( \nu( \mathcal{F}_{n} ) \mathbf{1}_{ A_{n} \cap A_{n}' } \to 0 \).
  Consequently, for any \( \alpha > 0 \),
  \begin{align}
            \mathbb{P}_{ f_{0} }^{ \otimes n } \{ 
              \nu_{n}( \mathcal{F}_{n} ) \ge \alpha
            \} 
    & \le 
            \mathbb{P}_{ f_{0} }^{ \otimes n }( ( A_{n} \cap A_{n}' )^{c} )
          + 
            \mathbb{P}_{ f_{0} }^{ \otimes n } \{ 
              \nu_{n}( \mathcal{F}_{n} ) \mathbf{1}_{ A_{n} \cap A_{n}' }
              \ge \alpha
            \}
    \xrightarrow[ ]{ n \to \infty } 0,
  \end{align}
  which finishes the proof.
\end{proof}

\begin{proof} 
  [\normalfont \textbf{Proof of Theorem \ref{thm_LanczosEigenvalueBound}}
  (Lanczos: Eigenvalue bound)]
  \label{prf_LanczosEigenvalueBound}
  We split the proof in separate steps.

  \textbf{Step 1: Polynomial formulation of orthogonality.}
  Any \( x \in \mathcal{K}_{ \tilde m } \), can be written as 
  \( x = p(A) v_{0} \) where \( p \) is a polynomial with 
  \( \deg p \le \tilde m - 1 \).
  We now prove that for any \( i \le \tilde m \), the statement
  \( x \perp \tilde u_{1}, \dots \tilde u_{ i - 1 } \) is equivalent to 
  \( p( \tilde \lambda_{1} ) = \dots = p( \tilde \lambda_{ i - 1 } ) = 0 \).

  If we set \( \tilde \alpha_{l}: = \langle v_{0}, \tilde u_{l} \rangle \), \( l \le \tilde m \), then for any \( j \le i - 1 \),
  \begin{align}
        \langle x, \tilde u_{j} \rangle
    & = 
        \Big\langle 
          \sum_{ l = 1 }^{ \tilde m }
          \tilde \alpha_{l} p(A) \tilde u_{l}, 
          \tilde u_{j}
        \Big\rangle 
    = 
        \Big\langle 
          \sum_{ l = 1 }^{ \tilde m }
          \tilde \alpha_{l} p( \tilde \lambda_{l}  ) \tilde u_{l}, 
          \tilde u_{j}
        \Big\rangle 
    = 
        \tilde \alpha_{j} p( \tilde \lambda_{j} ),
  \end{align}
  where we have used the fact that 
  \( 
    A \tilde u_{l} - \tilde \lambda_{l} \tilde u_{l} 
    \perp \mathcal{K}_{ \tilde m } 
  \) 
  from Lemma \ref{lem_ElementaryPropertiesOfLanczosEigenquantities} (ii) for the
  second equality.
  The claim now follows from the fact that \( \tilde \alpha_{j} \ne 0 \). 
  Indeed, otherwise, 
  \( 0 = \alpha_{j} = \langle \tilde u_{j}, v_{0} \rangle \) and further
  \begin{align}
        0 
    & = 
        \tilde \lambda_{j} \tilde \alpha_{j} 
    = 
        \langle \tilde \lambda_{j} \tilde u_{j}, v_{0} \rangle 
    = 
        \langle A \tilde u_{j}, v_{0} \rangle
    = 
        \langle \tilde u_{j}, A v_{0} \rangle,
  \end{align}
  where we have again used that 
  \( 
    A \tilde u_{l} - \tilde \lambda_{l} \tilde u_{l} 
    \perp \mathcal{K}_{ \tilde m } 
  \) 
  and 
  \( v_{0} \in \mathcal{K}_{ \tilde m } \).
  Inductively, this yields 
  \( \langle \tilde u_{j}, A^{l} v_{0} \rangle = 0 \) for all 
  \( 0 \le l \le \tilde m \), i.e., \( \tilde u_{j} = 0 \).
  This contradicts \( \dim \mathcal{K}_{ \tilde m } = \tilde m \), which is true
  under Assumption \hyperref[ass_LanczosWelldefined]
                   {\normalfont \textbf{{\color{blue} (LWdf)}}}.

  \textbf{Step 2: Polynomial formulation of the approximation.}
  Let \( U \) denote a linear subspace of \( \mathbb{R}^{n} \) and 
  \( \tilde A: = V V^{ \top } A V V^{ \top } \).
  The definition of the Lanczos algorithm \ref{alg_Lanczos} and the
  Courant-Fisher characterization of eigenvalues implies that
  \begin{align}
        \tilde \lambda_{i} 
    & = 
        \min_{ \dim U^{ \perp } = i - 1 } 
        \max_{ 0 \ne u \in U }
        \frac{ \langle \tilde A u, u \rangle }{ \langle u, u \rangle }
    = 
        \max_{ 0 \ne u \perp \tilde u_{1}, \dots, \tilde u_{ i - 1 } } 
        \frac{ \langle \tilde A u, u \rangle }
             { \langle u, u \rangle }
    \\
    & = 
        \max_{ 0 \ne u \perp \tilde u_{1}, \dots, \tilde u_{ i - 1 } } 
        \frac{ \langle A V V^{ \top } u, V V^{ \top } u \rangle }
             { \langle V V^{ \top } u, V V^{ \top } u \rangle }
    = 
        \max_{ \substack{ 
               0 \ne x \in \mathcal{K}_{ \tilde m }, \\
               x \perp \tilde u_{1}, \dots, \tilde u_{ i - 1 }
             } }
        \frac{ \langle A x, x \rangle }{ \langle x, x \rangle }.
    \notag
  \end{align}
  Note that for the third equality above, we may assume that without loss of
  generality \( u \in K_{ \tilde m } \), since restricting to 
  \( \mathcal{K}_{ \tilde m } \) reduces the denominator without reducing the
  numerator.

  Now, for \( L: = \{ i + 1, \dots, i + \tilde p \} \), set 
  \( x_{L}: = \prod_{ l \in L } ( A - \widehat{ \lambda }_{l}I ) v_{0} \) and
  \( \widehat{ \alpha }_{j}: = \langle x_{L}, \widehat{u}_{j} \rangle \),
  \( j \le n \). 
  Here, \( \widehat{ \alpha }_{j} = 0 \) for \( j \in L \), due to the
  definition of \( x_{L} \) and 
  \( \prod_{ l \in L } ( \tilde \lambda_{j} - \widehat{ \lambda }_{l} ) \ne 0 \)
  for all \( j < i \) due to Lemma
  \ref{lem_ElementaryPropertiesOfLanczosEigenquantities} (iii).
  Writing \( x \in \mathcal{K}_{ \tilde m } \) as \( p(A) v_{0} \) for a
  polynomial \( p \) with \( \deg p \le \tilde m - 1 \), Step 1 implies that
  \begin{align}
    \label{eq_A_PolynomaialApproximationFormulation}
            \widehat{ \lambda }_{i} - \tilde \lambda_{i}
    & = 
            \widehat{ \lambda }_{i} 
            - 
            \max_{ \substack{ 
                   0 \ne x \in \mathcal{K}_{ \tilde m }, \\
                   x \perp \tilde u_{1}, \dots, \tilde u_{ i - 1 }
                 } }
            \frac{ \langle A x, x \rangle }{ \langle x, x \rangle }
    = 
            \min_{ \substack{ 
                   \deg p \le \tilde m - 1, \\ 
                   \forall j < i: p( \tilde \lambda_{j} ) = 0
                 } }
            \frac{ \langle ( \widehat{ \lambda }_{i} I - A )
                           p(A) v_{0}, p(A) v_{0}
                   \rangle 
                 }
                 { \langle p(A) v_{0}, p(A) v_{0} \rangle }
    \\
    & \le 
            \min_{ \substack{ 
                   \deg p \le \tilde m - \tilde p - 1, \\ 
                   \forall j < i: p( \tilde \lambda_{j} ) = 0
                 } }
            \frac{ \langle ( \widehat{ \lambda }_{i} I - A ) 
                   p(A) x_{L}, p(A) x_{L}
                   \rangle 
                 }
                 { \langle p(A) x_{L}, p(A) x_{L} \rangle }
    \notag 
    \\
    & = 
            \min_{ \substack{ 
                   \deg p \le \tilde m - 1 - \tilde p, \\ 
                   \forall j < i: p( \tilde \lambda_{j} ) = 0
                 } }
            \frac{ \sum_{ j = 1 }^{ i - 1 } 
                   ( \widehat{ \lambda }_{i} - \widehat{ \lambda }_{j} ) 
                   | \widehat{ \alpha }_{j} p( \widehat{ \lambda }_{j} ) |^{2}
                 + 
                   \sum_{ j = i + \tilde p + 1 }^{n} 
                   ( \widehat{ \lambda }_{i} - \widehat{ \lambda }_{j} ) 
                   | \widehat{ \alpha }_{j} p( \widehat{ \lambda }_{j} ) |^{2}
                 }
                 { \sum_{ j = 1 }^{n} 
                   | \widehat{ \alpha }_{j} p( \widehat{ \lambda }_{j} ) |^{2}
                 }.
    \notag
  \end{align}
  Due to the minimization over \( p \), without loss of generality, all
  denominators in Equation \eqref{eq_A_PolynomaialApproximationFormulation}
  differ from zero.
  Since \( \widehat{ \lambda }_{1} \ge \dots \ge \widehat{ \lambda }_{n} \), the
  left sum in the numerator is non-positive.
  Additionally, 
  \( 
        \widehat{ \alpha }_{i} 
    =
        \prod_{ l \in L }
        ( \widehat{ \lambda }_{i} - \widehat{ \lambda }_{l} )
        \langle v_{0}, \widehat{u}_{i} \rangle
    \ne
        0
  \),
  since \( \langle \widehat{u}_{i}, v_{0} \rangle \ne 0 \) under Assumption
  \hyperref[ass_LanczosWelldefined]{\normalfont \textbf{{\color{blue} (LWdf)}}}.
  This yields
  \begin{align}
    \label{eq_A_PolynomaialApproximationFormulationFinal}
            \widehat{ \lambda }_{i} - \tilde \lambda_{i} 
    & \le 
            ( \widehat{ \lambda }_{i} - \widehat{ \lambda }_{n} ) 
            \min_{ \substack{ 
                   \deg p \le \tilde m - \tilde p - 1, \\ 
                   \forall j < i: p( \tilde \lambda_{j} ) = 0
                 } }
            \sum_{ j = i + \tilde p + 1 }^{n} 
            \frac{ | \widehat{ \alpha }_{j} p( \widehat{ \lambda }_{j} ) |^{2} }
                 { | \widehat{ \alpha }_{i} p( \widehat{ \lambda }_{i} ) |^{2} }
    \\
    & \le 
            ( \widehat{ \lambda }_{i} - \widehat{ \lambda }_{n} ) 
            \kappa_{ i, \tilde p }^{2}
            \tan^{2}( \widehat{u}_{i}, v_{0} )
            \min_{ \substack{ 
                   \deg p \le \tilde m - \tilde p - 1, \\ 
                   \forall j < i: p( \tilde \lambda_{j} ) = 0
                 } }
            \max_{ j \ge i + \tilde p + 1 } 
            \frac{ p( \widehat{ \lambda }_{j} )^{2} }
                 { p( \widehat{ \lambda }_{i} )^{2} },
    \notag
  \end{align}
  where we have used the fact that due to 
  \( 
       \widehat{ \alpha }_{j} 
    = 
       \prod_{ l \in L }
       ( \widehat{ \lambda }_{j} - \widehat{ \lambda }_{l} ) 
       \langle v_{0}, \widehat{u}_{j} \rangle 
  \), 
  \begin{align}
            \sum_{ j = i + \tilde p + 1 }^{n} 
            \frac{ \widehat{ \alpha }_{j}^{2} }
                 { \widehat{ \alpha }_{i}^{2} }
    & \le 
            \prod_{ l \in L } 
            \frac{ ( \widehat{ \lambda }_{l} - \widehat{ \lambda }_{n} )^{2} }
                 { ( \widehat{ \lambda }_{i} - \widehat{ \lambda }_{l}  )^{2} } 
            \sum_{ j = i + \tilde p + 1 }^{n}
            \frac{ \langle v_{0}, \widehat{u}_{j} \rangle^{2} }
                 { \langle v_{0}, \widehat{u}_{i} \rangle^{2} } 
    \le 
            \kappa_{ i, \tilde p }^{2} \tan^{2}( v_{0}, \widehat{u}_{i} ). 
  \end{align}
   
  \textbf{Step 3: Choice of the polynomial.}
  In Equation \eqref{eq_A_PolynomaialApproximationFormulationFinal}, we can
  restrict the choice of \( p \) to 
  \begin{align}
        p( \lambda ) 
    & = 
        \frac{ ( \tilde \lambda_{1} - \lambda ) \dots
               ( \tilde \lambda_{ i - 1 } - \lambda ) 
             }
             { ( \tilde \lambda_{1} - \widehat{ \lambda }_{i} ) \dots
               ( \tilde \lambda_{ i - 1 } - \widehat{ \lambda }_{i} )
             } 
        q( \lambda ), 
    \qquad \lambda \in \mathbb{R}, 
  \end{align}
  such that q is a polynomial with \( \deg q \le \tilde m - \tilde p - i \).
  We can then estimate
  \begin{align}
            \max_{ j \ge i + \tilde p + 1 }
            \frac{ p( \widehat{ \lambda }_{j} )^{2} }
                   { p( \widehat{ \lambda }_{i} )^{2} }
    & \le 
            \tilde \kappa_{i}^{2}
            \min_{ \deg q \le \tilde m - \tilde p - i } 
            \max_{ j \ge i + \tilde p + 1 }
            \frac{ q( \widehat{ \lambda }_{j} )^{2} }
                 { q( \widehat{ \lambda }_{i} )^{2} }
    \notag 
    \\
    & \le 
            \tilde \kappa_{i}^{2}
            \min_{ \substack{ 
                   \deg q \le \tilde m - \tilde p - i, \\
                   q( \widehat{ \lambda }_{i} ) = 1
                 } } 
            \max_{ \lambda \in
                   [ \widehat{ \lambda }_{n},
                     \widehat{ \lambda }_{ i + \tilde p + 1 } ] 
                 } 
            q( \lambda )^{2}
    = 
            \frac{ \tilde \kappa_{i}^{2} }
                 { T_{ \tilde m - i - \tilde p }( \gamma_{i} )^{2} }, 
    \notag
  \end{align}
  where the last equality follows from Theorem 4.8 in
  \cite{Saad2011LargeEigenvalueProblems}.
  Note that the estimate
  \( 
    0 \le ( \tilde    \lambda_{  i - 1} - \widehat{ \lambda }_{j} )^{2}
      \le ( \tilde    \lambda_{  i - 1} - \widehat{ \lambda }_{n} )^{2}
  \)
  used to arrive at the term \( \tilde \kappa_{i}^{2} \) requires the
  assumption \( \tilde \lambda_{ i - 1 } > \widehat{ \lambda }_{i} \). 
  This finishes the proof.
\end{proof}

\begin{proof} 
  [\normalfont \textbf{Proof of Theorem \ref{thm_LanczosEigenvectorBound}}
   (Lanczos: Eigenvector bound)]
  \label{prf_LanczosEigenvectorBound}
  Let \( ( \tilde \lambda^{*}, \tilde u^{*} ) \) be the approximate eigenpair
  that satisfies 
  \( 
      | \widehat{ \lambda }_{i} - \tilde \lambda^{*} | 
    = \min_{ j \le n } | \widehat{ \lambda }_{i} - \tilde \lambda_{j} | 
  \).
  From Theorem 4.6 in \cite{Saad2011LargeEigenvalueProblems}, we then have
  \begin{align}
            \sin^{2}( \widehat{u}_{i}, \tilde u^{*} ) 
    & \le 
            \Big( 1 + \frac{ \| P_{ \mathcal{K}_{ \tilde m } } A
                                ( I - P_{ \mathcal{K}_{ \tilde m } } )
                             \|_{ \text{op} }
                           }
                           { \delta^{2}_{i} }
            \Big)
            \sin^{2}( \widehat{u}_{i}, \mathcal{K}_{ \tilde m } )
    \\
    & \le 
            \Big( 1 + \frac{ \| K \|_{ \text{op} } } { n \delta^{2}_{i} }
            \Big)
            \sin^{2}( \widehat{u}_{i}, \mathcal{K}_{ \tilde m } ),
    \notag
  \end{align}
  where \( P_{ \mathcal{K}_{ \tilde m } } \) denotes that orthogonal projection
  onto the Krylov space \( \mathcal{K}_{ \tilde m } \).
  By Lemma 6.1 in \cite{Saad2011LargeEigenvalueProblems},
  \begin{align}
            \sin^{2}( \widehat{u}_{i}, \mathcal{K}_{ \tilde m } )
    & \le 
            \min_{ \substack{ \deg p \le \tilde m - 1, \\
                   p( \widehat{ \lambda }_{i} ) = 1
                 } }
            \| p(A) y_{i} \|^{2} 
            \tan^{2}( \widehat{u}_{i}, v_{0} ),
  \end{align}
  where the minimum is taken over polynomials \( p \) and 
  \(  
    y_{i}: =    ( I - \widehat{u}_{i} \widehat{u}_{i}^{ \top } ) v_{0} /
             \| ( I - \widehat{u}_{i} \widehat{u}_{i}^{ \top } ) v_{0} \| 
  \).
  Note that the denominator is not zero according to Assumption
  \hyperref[ass_LanczosWelldefined]{\normalfont \textbf{{\color{blue} (LWdf)}}}.

  For \( L: = i + 1, \dots, i + \tilde p \), set 
  \( 
    x_{L}: = ( \prod_{ l \in L }
               ( A - \widehat{ \lambda }_{l} ) / 
               ( \widehat{ \lambda }_{i} - \widehat{ \lambda }_{l} ) 
             ) y_{i} 
  \) and \( \widehat{ \alpha }_{j}: = \langle x_{L}, \widehat{u}_{j} \rangle \),
  \( j \le n \).
  Since \( \widehat{ \alpha }_{l} = 0 \) for all \( l \in L \cup \{ i \} \), we
  can now estimate
  \begin{align}
            \min_{ \substack{ 
                   \deg p \le \tilde m - 1, \\
                   p( \widehat{ \lambda }_{i} ) = 1
                 } }
            \| p(A) y_{i} \|^{2}
    & \le 
            \min_{ \substack{ 
                   \deg p \le \tilde m - \tilde p - 1, \\
                   p( \widehat{ \lambda }_{i} ) = 1
                 } }
            \| p(A) x_{L} \|^{2}
    = 
            \min_{ \substack{ 
                   \deg p \le \tilde m - \tilde p - 1, \\
                   p( \widehat{ \lambda }_{i} ) = 1
                 } }
            \sum_{ j = 1 }^{n}
            | p( \widehat{ \lambda }_{j} ) \widehat{ \alpha }_{j} |^{2} 
    \\
    & \le 
            \prod_{ j = 1 }^{ i - 1 } 
            \Big( \frac{ \widehat{ \lambda }_{j} - \widehat{ \lambda }_{n} }
                       { \widehat{ \lambda }_{j} - \widehat{ \lambda }_{i} }
            \Big)^{2}
            \min_{ \substack{ 
                   \deg q \le \tilde m - \tilde p - i, \\
                   q( \widehat{ \lambda }_{i} ) = 1
                 } }
            \sum_{ j = i + \tilde p + 1 }^{n}
            | q( \widehat{ \lambda }_{j} ) \widehat{ \alpha }_{j} |^{2} 
    \notag
    \\
    & \le 
            \prod_{ j = 1 }^{ i - 1 } 
            \Big( \frac{ \widehat{ \lambda }_{j} - \widehat{ \lambda }_{n} }
                       { \widehat{ \lambda }_{j} - \widehat{ \lambda }_{i} }
            \Big)^{2}
            \sum_{ j = i + \tilde p + 1 }^{n}
            | \widehat{ \alpha }_{j} |^{2} 
            \min_{ \substack{ 
                   \deg q \le \tilde m - \tilde p - i, \\
                   q( \widehat{ \lambda }_{i} ) = 1
                 } }
            \max_{ \lambda \in
                   [ \widehat{ \lambda }_{n},
                     \widehat{ \lambda }_{ i + \tilde p + 1 } ]
                 } 
            | q( \lambda ) |^{2},
    \notag
  \end{align}
  where for the second inequality, we have restricted the choice of \( p \) to
  polynomials
  \begin{align}
            p( \lambda )
    & = 
            \prod_{ j = 1 }^{ i - 1 } 
            \frac{ \widehat{ \lambda }_{j} - \lambda }
                 { \widehat{ \lambda }_{j} - \widehat{ \lambda }_{i} } 
            q( \lambda ), 
    \qquad \lambda \in \mathbb{R}
  \end{align}
  with polynomials \( q \) such that \( \deg q \le \tilde m - \tilde p - i \).
  Finally,
  \begin{align}
        \sum_{ j = i + \tilde p + 1 }^{n} 
        \widehat{ \alpha }_{j}^{2} 
    & = 
        \sum_{ j = i + \tilde p + 1 }^{n} 
        \prod_{ l \in L } 
        \Big( \frac{ \widehat{ \lambda }_{j} - \widehat{ \lambda }_{l} }
                   { \widehat{ \lambda }_{i} - \widehat{ \lambda }_{l} }
        \Big)^{2}
        \langle y_{i}, \widehat{u}_{j} \rangle^{2}
    \le 
        \kappa_{ i, \tilde p }^{2} \| y_{i} \|^{2} 
    = 
        \kappa_{ i, \tilde p }^{2}
  \end{align}
  and 
  \begin{align}
        \min_{ \substack{
               \deg q \le \tilde m - \tilde p - i, \\
               q( \widehat{ \lambda }_{i} ) = 1
             } }
        \max_{ \lambda \in
               [ \widehat{ \lambda }_{n},
                 \widehat{ \lambda }_{ i + \tilde p + 1 } ]
             } 
        | q( \lambda ) |^{2}
    = 
        \frac{1}{ T_{ \tilde m - \tilde p - i }( \gamma_{i} )^{2} }
  \end{align}
  by Theorem 4.8 in \cite{Saad2011LargeEigenvalueProblems}.
  This finished the proof.
\end{proof}

\begin{proof} 
  [\normalfont \textbf{Proof of Proposition \ref{prp_RelativePerturbatonBounds}}
  (Relative perturbation bounds)]
  \label{prf_RelativePerturbationBounds}
  We consider the kernel operator 
  \begin{align}
    T_{k}: L^{2}(G) \to  L^{2}(G), \qquad 
           f        \mapsto \int f(x) k( \cdot, x) \, G( d x )
                          = 
                            \sum_{ j \ge 1 } 
                            \lambda_{j} 
                            \langle f, \phi_{j} \rangle_{ L^{2} } \phi_{j}, 
  \end{align}
  from Equation \eqref{eq_APC_KernelOperator} with
  summable eigenvalues \( ( \lambda_{j} )_{ j \ge 1 } \) and orthonormal basis
  \( ( \phi_{j} )_{ j \ge 1 } \) of \( L^{2}(G) \).
  Under Assumption \hyperref[ass_KarhunenLoeveMoments]
                   {\normalfont \textbf{{\color{blue} (KLMom)}}}, checking the
  second moments yields that the kernel \( k \) is given by
  \begin{align}
    k(x, x) = \sum_{ j = 1 }^{ \infty }
              \lambda_{j} \phi_{j}(x)^{2}
    \quad \text{and} \quad 
    k(x, x') = \sum_{ j = 1 }^{ \infty }
               \lambda_{j} \phi_{j}(x) \phi_{j}( x'), 
    \qquad x, x' \in \mathfrak{X}
  \end{align}
  \( G \)- and \( G^{ \otimes 2 } \)-almost surely respectively.
  By redefining the process \( k \) and the \( ( \phi_{j} )_{ j \ge 1 } \) on
  a nullset, without loss of generality, the equality is true everywhere.
  Following the reasoning in Corollary 12.16 of \cite{Wainwright2019HDStats},
  the RKHS \( \mathbb{H} \) induced by \( k \) is equal to 
  \( \ran T_{k}^{ 1 / 2 } \) and we may consider the restriction
  \begin{align}
    \Sigma: \mathbb{H} \to     \mathbb{H},
            \qquad 
            h  \mapsto \mathbb{E} \langle h, k( \cdot, X_{1} ) 
                                  \rangle_{ \mathbb{H} } k(\cdot, X_{1})
                     = \sum_{ j \ge 1 }
                       \lambda_{j} 
                       \langle h, \varphi_{j} \rangle_{ \mathbb{H} }
                       \varphi_{j},
  \end{align}
  of \( T_{k} \) to \( \mathbb{H} \), which is the covariance operator of 
  \( k( \cdot, X_{1} ) \).
  Note that in the eigensystem of \( \Sigma \), the functions 
  \( \varphi_{j} = \sqrt{ \lambda_{j} } \phi_{j} \), \( j \ge 1 \) form an
  orthonormal basis of \( \mathbb{H} \).
  Its empirical version is given by
  \begin{align}
    \widehat{ \Sigma }: \mathbb{H} \to \mathbb{H}, 
    \qquad 
            h 
    \mapsto 
            \frac{1}{n} \sum_{ i = 1 }^{n} 
            \langle h, k( \cdot, X_{i} ) \rangle_{ \mathbb{H} } k( \cdot, X_{i} ) 
          = 
            \sum_{ j = 1 }^{n} 
            \widehat{ \lambda }_{j} 
            \langle h, \widehat{ \varphi }_{j} \rangle_{ \mathbb{H} } 
            \widehat{ \varphi }_{j},
  \end{align}
  where \( ( \widehat{ \lambda_{j} }, \widehat{ \varphi }_{j} )_{ j \le n } \)
  are empirical versions of the \( ( \lambda_{j}, \varphi_{j} )_{ j \ge 1 } \). 

  Note that the sampling operator and its adjoint
  \begin{align}
    S_{n}: & \mathbb{H} \to \mathbb{R}^{n},
    \qquad 
            h 
    \mapsto 
            ( h( X_{1} ), \dots, h( X_{n} ) )^{ \top },
    \\
    S_{n}^{*}: & \mathbb{R}^{n} \to \mathbb{H}, 
    \qquad 
            u
    \mapsto 
            \frac{1}{n} \sum_{ i = 1 }^{n} u_{i} k( \cdot, X_{i} ) 
    \notag
  \end{align}
  with respect to the empirical dot-product \( \langle \cdot, \cdot \rangle_{n} \)
  satisfy \( S_{n} S_{n}^{*} = n^{-1} K \) and 
  \( S_{n}^{*} S_{n} = \widehat{ \Sigma } \).
  Therefore, \( n^{-1} K \) has the same eigenvalues as 
  \( \widehat{ \Sigma } \).

  In the following, we adopt the notation from
  \cite{JirakWahl2023RelativePerturbationBounds} and set
  \begin{align}
    P_{j}:
    & =
    \varphi_{j} \otimes \varphi_{j}: \mathbb{H} \to \mathbb{H},
    \qquad 
    h \mapsto \langle h, \varphi_{j} \rangle_{ \mathbb{H} } \varphi_{j},
    \qquad j \ge 1,
    \\
    P_{ \ge s }:
    & 
    = \sum_{ j \ge s } P_{j}, \qquad s \ge 1,
    \qquad 
    E: = \widehat{ \Sigma } - \Sigma.
    \notag
  \end{align}

  \textbf{Step 1: Deterministic analysis.}
  Corollary 3 in \cite{JirakWahl2023RelativePerturbationBounds} states that for
  any \( r \ge 1 \) and \( r_{0} \ge r \) such that 
  \( \lambda_{ r_{0} } \le \lambda_{r} / 2 \)
  \begin{align}
    | \widehat{ \lambda }_{r} - \lambda_{r} | 
    & \le 
    \lambda_{r} x,
  \end{align}
  whenever
  \begin{align}
    \label{eq_ConcentrationCondition}
    \frac{ \| P_{s} E P_{t} \|_{ \text{HS} } }
         { \sqrt{ \lambda_{s} \lambda_{t} } }, 
    \frac{ \| P_{s} E P_{ \ge r_{0} } \|_{ \text{HS} } }
         { \sqrt{ \sum_{ j \ge r_{0} } \lambda_{s} \lambda_{j} } }, 
    \frac{ \| P_{ \ge r_{0} } E P_{ \ge r_{0} } \|_{ \text{HS} } }
         { \sqrt{ \sum_{ j, j' \ge r_{0} } \lambda_{j} \lambda_{ j' } } }
    & \le x,
    \qquad \text{ for all } s, t \le r_{0}
    \\
    \label{eq_RelativeRankConditionII}
    \qquad \text{ and } \qquad \qquad \qquad \qquad
    \mathbf{r}_{r}( \Sigma ) & \le \frac{1}{ 6 x }.
  \end{align}
  Note that in our setting, 
  \( \lambda_{ m_{0} } \le \lambda_{m} / 2 \le \lambda_{i} / 2 \)
  for all \( i \le m \).
  Additionally, in our assumptions, the bound on the relative rank is also
  uniform in \( i \le m \).
  It remains to control the event in Equation \eqref{eq_ConcentrationCondition}
  with high probability for \( x = C \sqrt{ n / \log n } \) and 
  \( r_{0} = m_{0} \).
  Our result follows from there.

  \textbf{Step 2: Adapted concentration result.} 
  We prove that for any \( I, J \subset \mathbb{N} \) 
  \begin{align}
            \mathbb{P}
            \Big\{ \frac{ \| P_{I} E P_{J} \|_{ \text{HS} } }
                        { \sqrt{ \sum_{ i \in I, j \in J } 
                                 \lambda_{i} \lambda_{j}
                               }
                        } 
              \ge \frac{ C x }{ \sqrt{n} }
            \Big\} 
    & \le 
            \frac{n}{ ( \sqrt{n} x )^{ p / 2 } } + e^{ - x^{2} },
    \qquad \qquad \text{ for all } x \ge 1,
  \end{align}
  where \( P_{I} \) denotes \( \sum_{ i \in I } P_{i} \).
  Indeed, with \( \delta_{ i, j } \) denoting the Kronecker delta,
  \begin{align}
        n P_{I} E P_{J} 
    & = 
        \sum_{ l = 1 }^{n}
        \sum_{ i \in I, j \in J } 
        ( \langle k( X_{l}, \cdot ), \varphi_{i} \rangle_{ \mathbb{H} } 
          \langle k( X_{l}, \cdot ), \varphi_{j} \rangle_{ \mathbb{H} }
        - 
          \sqrt{ \lambda_{i} \lambda_{j} } \delta_{ i, j }
        ) 
        \varphi_{i} \otimes \varphi_{j}
    = 
        \sum_{ l = 1 }^{n} Z_{l} 
  \end{align}
  with 
  \( 
    Z_{l}: = \sum_{ i \in I, j \in J }
             ( \langle k( X_{l}, \cdot ), \varphi_{i} \rangle_{ \mathbb{H} }
               \langle k( X_{l}, \cdot ), \varphi_{j} \rangle_{ \mathbb{H} } 
             - 
                \sqrt{ \lambda_{i} \lambda_{j} } \delta_{ i, j }
             )
             \varphi_{i} \otimes \varphi_{j} \), \( l \le n 
  \).
  Note that the \( ( Z_{l} )_{ l = 1 }^{n} \) are independent, identically
  distributed and centered, since 
  \( 
    \mathbb{E} ( \langle k( X_{l}, \cdot ), \varphi_{i} \rangle_{ \mathbb{H} }
    \langle k( X_{l}, \cdot ), \\ \varphi_{j} \rangle_{ \mathbb{H} } ) 
  = 
    \langle \Sigma^{ 1 / 2 } \varphi_{i}, \Sigma^{ 1 / 2 } \varphi_{j}
    \rangle_{ \mathbb{H} } = \sqrt{ \lambda_{i} \lambda_{j} } \delta_{i, j}
  \). 
  Using Jensen's inequality, we have
  \begin{align}
            \Big( \mathbb{E} \| \sum_{ l = 1 }^{n} Z_{l} \|_{ \text{HS} } 
            \Big)^{2} 
    & \le 
            \mathbb{E} \| \sum_{ l = 1 }^{n} Z_{l} \|_{ \text{HS} }^{2} 
    = 
            \sum_{ l, l' = 1 }^{n} 
            \mathbb{E} \langle Z_{l}, Z_{ l' } \rangle_{ \text{HS} }
    = 
            n \mathbb{E} \langle Z_{1}, Z_{1} \rangle_{ \text{HS} } 
    \\
    & = 
            n \sum_{ i, i' \in I } \sum_{ j, j' \in J } 
            \mathbb{E} ( \tilde \eta_{ i, j } \tilde \eta_{ i', j' } ) 
            \tr ( 
              ( \varphi_{i} \otimes \varphi_{j} )^{*}
              \varphi_{ i' } \otimes \varphi_{ j' }
            ),
    \notag
  \end{align}
  where 
  \( 
    \tilde \eta_{ i, j } 
  =
    \langle k( X_{1}, \cdot ), \varphi_{i} \rangle_{ \mathbb{H} } \langle k(
    X_{1}, \cdot ), \varphi_{j} \rangle_{ \mathbb{H} } - \sqrt{ \lambda_{i}
    \lambda_{j} } \delta_{ i, j }
  \),
  \( i \in I \), \(  j \in J \).
  With 
  \( 
      \tr ( ( \varphi_{i}  \otimes \varphi_{j} )^{*}
              \varphi_{i'} \otimes \varphi_{j'}
          ) 
    = \delta_{ i, i' } \delta_{ j, j' } 
  \)
  and Assumption \hyperref[ass_KarhunenLoeveMoments]
                 {\normalfont \textbf{{\color{blue} (KLMom)}}},
  we obtain that
  \begin{align}
    \label{eq_FukNagaevConditionI}
            \Big( \mathbb{E} \| \sum_{ l = 1 }^{n} Z_{l} \|_{ \text{HS} } 
            \Big)^{2} 
    & \le 
            n \sum_{ i \in I, j \in J } \mathbb{E} \tilde \eta_{ i j }^{2}  
    \le 
            C n \sum_{ i \in I, j \in J } \lambda_{i} \lambda_{j}.
  \end{align}

  By analogous arguments,
  \begin{align}
    \label{eq_FukNagaevConditionII}
          \mathbb{E} \| Z_{1} \|_{ \text{HS} }^{ p / 2 }
    & = 
          \mathbb{E} \langle Z_{1}, Z_{1} \rangle_{ \text{HS} }^{ p / 4 }
    = 
          \mathbb{E} \Big( 
            \sum_{ i \in I, j \in J } \tilde \eta_{ i, j }^{2} 
          \Big)^{ p / 4 }
    \\
    & \le 
          C \Big[ 
            \Big( \sum_{ i \in I, j \in J } \lambda_{i} \lambda_{j} \Big)^{p/4}
          + 
            \mathbb{E} 
            \Big(
              \sum_{ i \in I, j \in J } 
              \langle k( X_{1}, \cdot ), \varphi_{i} \rangle_{ \mathbb{H} }^{2}
              \langle k( X_{1}, \cdot ), \varphi_{j} \rangle_{ \mathbb{H} }^{2}
            \Big)^{ p / 4 }
          \Big] 
    \notag
    \\
    & = 
          C \Big[ 
            \Big( \sum_{ i \in I, j \in J } \lambda_{i} \lambda_{j} \Big)^{p/4}
          + 
            \mathbb{E} \Big( 
              \Big( \sum_{ i \in I }
                    \langle k( X_{1}, \cdot ), \varphi_{i} 
                    \rangle_{ \mathbb{H} }^{2} 
              \Big)^{p/4}
              \Big( \sum_{ j \in J }
                    \langle k( X_{1}, \cdot ), \varphi_{j} 
                    \rangle_{ \mathbb{H} }^{2} 
              \Big)^{p/4}
            \Big)
          \Big] 
    \notag
    \\
    & \le
          C \Big[
            \Big( \sum_{ i \in I, j \in J } \lambda_{i} \lambda_{j} \Big)^{p/4}
          + 
            \| \sum_{ i \in I } 
               \langle k( X_{1}, \cdot ), \varphi_{i} \rangle_{ \mathbb{H} }^{2} 
            \|_{p/2}^{p/4} 
            \| \sum_{ i \in I }
               \langle k( X_{1}, \cdot ), \varphi_{j} \rangle_{ \mathbb{H} }^{2} 
            \|_{ p / 2 }^{ p / 4 } 
          \Big] 
    \notag
    \\
    & \le 
          C \Big( \sum_{ i \in I, j \in J } \lambda_{i} \lambda_{j} \Big)^{p/4}.
    \notag
  \end{align}
  
  Finally, for any Hilbert-Schmidt operator \( f \) with 
  \( \| f \|_{ \text{HS} } \le 1 \), 
  \begin{align}
    \label{eq_FukNagaevConditionIII}
            \mathbb{E} 
            \sum_{ l = 1 }^{n}
            \langle f, Z_{l} \rangle_{ \text{HS} }^{2} 
    & \le 
            n \mathbb{E} \| Z_{1} \|_{ \text{HS} }^{2} 
    \le 
            C n \sum_{ i \in I, j \in J } \lambda_{i} \lambda_{j}.
  \end{align}

  Due to the conditions in Equation \eqref{eq_FukNagaevConditionI},
  \eqref{eq_FukNagaevConditionII}, \eqref{eq_FukNagaevConditionIII}, we can
  apply Theorem 3.1 in \cite{EinmahlLi2007LILBehaviourInBanachSpaces}, which is
  a version of Fuk-Nagaev inequality for Hilbert space valued random variables
  and obtain that for any \( x' > 0 \), 
  \begin{align}
            \| \sum_{ l = 1 }^{n} Z_{l} \|_{ \text{HS} } 
    & \le 
            C \sqrt{ n \sum_{ i \in I, j \in J } \lambda_{i} \lambda_{j} } + x'
  \end{align}
  with probability at least
  \begin{align}
    1 - \exp \Big( 
          \frac{ - c x'^{2} }
               { n \sum_{ i \in I, j \in J } \lambda_{i} \lambda_{j} }
        \Big) 
      - 
        C n \Big( 
          \sum_{ i \in I, j \in J } \lambda_{i} \lambda_{j} 
        \Big)^{ p / 4 } 
        x'^{ - p / 2 }.
  \end{align}
  Setting 
  \( x = c x' / \sqrt{ n \sum_{ i \in I, j \in J } \lambda_{i} \lambda_{j} } \),
  this yields
  \begin{align}
            \| \sum_{ l = 1 }^{n} Z_{l} \|_{ \text{HS} } 
    & \le 
            C \Big( 
              \sqrt{ n \sum_{ i \in I, j \in J } \lambda_{i} \lambda_{j} }
            + 
              x \sqrt{ n \sum_{ i \in I, j \in J } \lambda_{i} \lambda_{j} }
            \Big)
  \end{align}
  with probability at least 
  \( 1 - e^{ - x^{2} } - C n ( \sqrt{n} x )^{ - p / 2 } \).
  The result now follows by deviding the above by \( n \) and only considering
  \( x \ge 1 \).

  The eigenvector result now follows from Step 2 using a union bound over
  \( s, t \le m_{0} \) and setting \( x = C \sqrt{ n / \log n } \).
\end{proof}

\end{appendices}

\textbf{Funding.}
Co-funded by the European Union (ERC, BigBayesUQ, project number: 101041064).
Views and opinions expressed are however those of the author(s) only and do not
necessarily reflect those of the European Union or the European Research
Council.
Neither the European Union nor the granting authority can be held responsible
for them.

\printbibliography

\end{document}